\documentclass{article}

 \usepackage[main, final]{neurips_2026}

\usepackage{wrapfig}
\usepackage{needspace}

\usepackage{graphicx}

\usepackage[utf8]{inputenc} 
\usepackage[T1]{fontenc}    
\usepackage{hyperref}       
\usepackage{url}            
\usepackage{booktabs}       
\usepackage{enumitem}       
\usepackage{amsfonts}       
\usepackage{nicefrac}       
\usepackage{microtype}      
\usepackage[dvipsnames,table]{xcolor} 
\usepackage{graphicx}       
\usepackage{subcaption}     
\usepackage{multirow}       
\usepackage{threeparttable} 
\usepackage{algorithm}      
\usepackage{algorithmic}    
\usepackage{amsmath}        
\usepackage{amssymb}        
\usepackage{amsthm}         
\usepackage{dsfont}         
\usepackage{fontawesome5}
\definecolor{burgundy}{rgb}{0.5, 0.0, 0.13}
\definecolor{antiquewhite}{rgb}{0.98, 0.92, 0.84}
\definecolor{blizzardblue}{rgb}{0.67, 0.9, 0.93}
\colorlet{shadecolor}{pink}
\definecolor{lightpurple}{RGB}{225, 215, 245}
\usepackage{hyperref}
\usepackage{url}
\usepackage{cleveref}

\hypersetup{
    colorlinks=true,
    citecolor=burgundy,
    linkcolor=blue,
    filecolor=magenta,
    urlcolor=cyan
}
\newcommand{\mathbbm}{\mathds}

\newtheorem{theorem}{Theorem}
\newtheorem{proposition}{Proposition}
\newtheorem{lemma}{Lemma}

\newtheorem{definition}{Definition}
\newtheorem{assumption}{Assumption}
\newtheorem{hypothesis}{Hypothesis}
\newtheorem{remark}{Remark}

\title{Subspace-Aware Sparse Autoencoders for Effective Mechanistic Interpretability}

\author{%
  Seyed Arshan Dalili \\
  The Pennsylvania State University \\
  \texttt{sbd5760@psu.edu} \\
  \And
  Mehrdad Mahdavi \\
  The Pennsylvania State University \\
  \texttt{mzm616@psu.edu} \\
}

\begin{document}
\raggedbottom

\maketitle

\begin{abstract}
  Sparse Autoencoders (SAEs) are widely used for mechanistic interpretability in large language models, yet their standard formulation assigns each latent feature a single decoder direction, implicitly assuming features to be one-dimensional. We show that this assumption can induce feature splitting when model features are multi-dimensional. \emph{Geometrically}, accurately reconstructing a feature under a fixed co-activation budget can require exponentially many decoder directions as its dimension grows. From an \emph{optimization} perspective, adding decoder directions can strictly lower the $\ell_1$-regularized population risk, even when a basis can already reconstruct the feature exactly. Motivated by these results, we introduce \emph{Subspace-Aware Sparse Autoencoders} (SASA), which replace single-vector decoders with learned subspaces, enforce sparsity through Top-$s$ group gating, and adapt group rank with nuclear-norm regularization. Under an idealized single-feature model, we prove that the SASA objective is minimized by a single group whose reconstruction map recovers the feature subspace. We also derive polynomial sample complexity bounds for subspace estimation within a group and a lower bound on the samples needed to observe every atom in a fixed fragmented dictionary. Experiments on GPT-2 and Mistral-7B show reduced absorption compared with SAEs. Matched-budget GPT-2 comparisons show lower absorption for concept families and more accurate prediction of their members using a small number of latent coordinates, at a reconstruction cost. The code is available at \ \href{https://github.com/arshandalili/sasa}{\texttt{\faGithub\ arshandalili/sasa}}.

\end{abstract}


\begin{figure}[t]
    \centering
    \includegraphics[width=\linewidth]{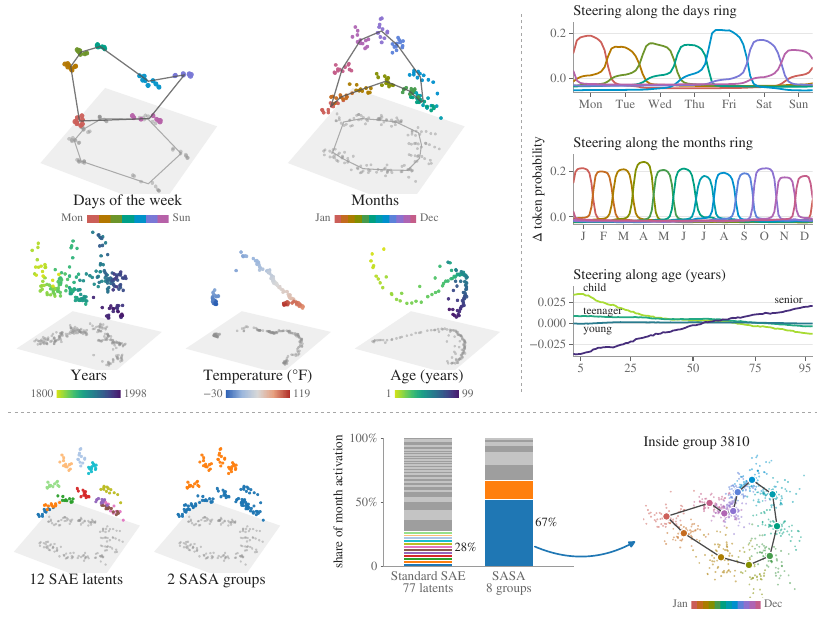}\vspace{-0.25cm}
    \caption{\textbf{Concept geometry and feature splitting in Mistral-7B.} At the top, PCA projections of layer-8 activations show rings and curves, with days and months centered within each carrier sentence. Interventions along the circular representations of days and months change calendar predictions, while movement along the age curve favors ``senior'' over ``child''. At the bottom left, colors identify the $12$ SAE latents~\citep{Engels_Michaud_Liao_Gurnee_Tegmark_2025} and $2$ SASA groups with the dominant contributions to the representation of months. The bars show $77$ SAE latents and $8$ SASA groups selected for their responses to month-name tokens (Appendix~\ref{apx:group-selection}). Within these sets, the colored units contribute $28\%$ of SAE activations and $67\%$ of SASA group $\ell_2$ norms over month-name tokens. The monthly centroids within the dominant SASA group form a complete ring in calendar order (right).} \vspace{-0.5cm}
    \label{fig:real-manifolds}
\end{figure}

\section{Introduction}
Mechanistic interpretability of Large Language Models (LLMs) aims to reveal how these models internally encode and manipulate knowledge, providing insights crucial for diagnosing behavior~\citep{Wang_Variengien_Conmy_Shlegeris_Steinhardt_2022, Conmy_Mavor-Parker_Lynch_Heimersheim_Garriga-Alonso_2023}, enhancing robustness~\citep{García-Carrasco_Maté_Trujillo_2024, Winninger_Addad_Kapusta_2025}, and ensuring alignment~\citep{Arditi_Obeso_Syed_Paleka_Panickssery_Gurnee_Nanda_2024, Lee_Bai_Pres_Wattenberg_Kummerfeld_Mihalcea_2024}. Recently, Sparse Autoencoders (SAEs) have become a prominent tool in the mechanistic interpretability of LLMs. Rooted in the Superposition Hypothesis \citep{Elhage_Hume_Olsson_Schiefer_Henighan_Kravec_Hatfield-Dodds_Lasenby_Drain_Chen_etal._2022} and Linear Representation Hypothesis (LRH) \citep{Alain_Bengio_2018, Park_Choe_Veitch_2024}, SAEs aim to make LLM representations interpretable by disentangling them into a \emph{sparse} set of directions in the ambient space. SAEs achieve this objective by learning a compressed, sparse representation of the model's high-dimensional activations, effectively isolating a small set of active directions that capture the most salient features of the data. Each of the directions in the decoder often corresponds to a distinct, one-dimensional semantic feature, allowing for a more fine-grained, interpretable mapping between the model's internal representations and specific concepts or patterns.

Nevertheless, interpreting LLM representations with SAEs remains challenging for both training and inference. 
From a training perspective, SAEs require large corpora of LLM activations (e.g., hidden states from specific layers and token positions). In practice, this often means billions of hidden states, which limits scaling to large models. 
From an inference perspective, existing SAEs can learn \emph{redundant} latents-- multiple vectors corresponding to the same underlying feature--a phenomenon known as \emph{feature splitting} \citep{Bricken_Templeton_Batson_Chen_Jermyn_Conerly_Turner_Anil_Denison_Askell_etal._2023}. This fragmentation complicates interpretation because an individual latent captures only part of a feature, even when the SAE reconstructs the activations accurately.

Intuitively, an SAE exhibits feature splitting when a single semantic feature is not captured by one decoder vector, but is instead distributed across many decoder directions. Each direction explains only a local ``slice'' of the feature, so interpretation requires aggregating a cluster of near-duplicate latents rather than inspecting a single unit (as illustrated in Figure~\ref{fig:real-manifolds}, with synthetic examples in Appendix~\ref{apx:synthetic-illustrations}). We rigorously formalize this intuition via a covering number and show that a solely sparsity-regularized decoder risk, i.e., $\ell_1$, favors adding non-basis directions to the feature basis.

In parallel, recent empirical observations in LLM representations suggest an extension of the long-held LRH, known as the Multi-dimensional LRH \citep{Engels_Michaud_Liao_Gurnee_Tegmark_2025,Modell_Rubin-Delanchy_Whiteley_2025}. Unlike LRH, Multi-dimensional LRH suggests that features are not represented by a \emph{vector} in the ambient space, but rather by \emph{subspaces}. For example, features such as days, months, and colors exhibit circular representations. In Mistral-7B, moving the activation at a day or month token around the corresponding ring changes the model's probabilities for day or month names (Figure~\ref{fig:real-manifolds}).
Building on these empirical findings, a natural question arises: if LLM features are structured as low-dimensional \emph{subspaces} rather than single directions, what does this imply for SAEs, whose objective is to recover a \emph{sparse} set of \emph{vectors}? 

Collectively, these observations suggest a fundamental \textit{geometric incompatibility} between the structure of features in modern LLM representations and the vector-level sparsity assumptions underlying standard SAEs, one that may manifest as the redundant and fragmented learned vectors commonly observed as feature splitting. This motivates applying sparsity to the selection of features while allowing several coordinates to describe variation within each active feature. To accommodate multi-dimensional features, we rigorusly examine feature splitting in standards SAEs and introduce \emph{Subspace-Aware Sparse Autoencoders (SASA)}, which replaces single-vector decoders with learned subspaces, enforces sparsity through Top-$s$ group gating, and adapts group rank with nuclear-norm regularization. In Figure~\ref{fig:real-manifolds}, a single SASA group retains the full cyclic structure of the month feature.\vspace{1mm}
\noindent\textbf{Contributions.}~We make the following contributions:\vspace{-2mm}
\begin{itemize}
    \item  We identify a geometric mismatch between the multi-dimensional feature structure suggested by the Multi-dimensional LRH and the vector-feature assumptions implicit in standard SAEs, and show how this mismatch can causally induce feature splitting.\vspace{0mm}
\item We formalize the notion of feature splitting for SAEs and give a principled criterion for when a single multi-dimensional feature is fragmented across many learned vectors. We further show that, when features are inherently multi-dimensional and low reconstruction error is required, splitting is unavoidable under a fixed co-activation budget smaller than the feature-subspace dimension. We also show that the $\ell_1$-regularized risk with optimized codes favors additional non-basis directions even when a feature basis is available.\vspace{0mm}
\item We propose \emph{SASA} algorithm, which incorporate multi-dimensional structure directly into the SAE objective through group sparsity with a rank-adaptive regularizer. By moving sparsity to the group level, SASA removes the coordinatewise sparsity incentive to split multi-dimensional features, mitigating splitting and improving interpretability.\vspace{-0mm}
\item We compare the sample complexity of estimating a feature subspace with that of observing every atom in a fixed fragmented dictionary. This comparison identifies a potential data-acquisition cost of feature splitting.\vspace{0mm}
\item Finally, we validate these gains empirically on synthetic and real-world data and models. \vspace{-2mm}
\end{itemize}
\noindent\textbf{Roadmap.} Section~\ref{sec:problem} formalizes the multi-dimensional superposition hypothesis and reviews standard SAEs.
In Section~\ref{sec:necessity}, we make feature splitting precise via a $k$-sparse covering criterion that requires SAEs to allocate many decoder directions when the co-activation budget is small and low reconstruction error is required. We then show that the optimized-code $\ell_1$ risk favors additional non-basis directions.
Motivated by the negative results of SAEs, Section~\ref{sec:sasa} introduces SASA, which represents features with low-rank subspaces and enforces structured sparsity over these subspaces. Section~\ref{sec:sample-compelxity} analyzes sample complexity followed by experiments in Section~\ref{sec:exp}. The appendices discuss additional related work and provide proofs, algorithms, and additional analyses.\vspace{-1mm}
\section{Problem Formulation and Preliminaries}
\label{sec:problem}
\noindent\textbf{Notation.} We use bold-faced font to denote vectors and matrices, e.g., $\boldsymbol{x} \in \mathbb{R}^d$ and $\boldsymbol{X} \in \mathbb{R}^{d\times d}$. We use calligraphic letters to denote sets, e.g., $\mathcal{X}$. We let $[n] := \{1, 2, \cdots, n\}$. For a vector $\boldsymbol{x}\in\mathbb{R}^d$, $[\boldsymbol{x}]_i$ denotes its $i$-th element and $\|\boldsymbol{x}\|_2$ its Euclidean norm. For a matrix $\boldsymbol{X} \in \mathbb{R}^{d\times d}$, $\|\boldsymbol{X}\|_2$ denotes its spectral norm and $\mathrm{tr}(\boldsymbol{X})$ its trace. Its nuclear norm $\|\boldsymbol{X}\|_*$ is the sum of its singular values.
For a linear subspace $\mathcal{V} \subset \mathbb{R}^d$ with $\operatorname{dim}(\mathcal{V})= d^\prime \le d$, let $\boldsymbol{V}\in \mathbb{R}^{d \times d^\prime}$ be an orthonormal basis and $\boldsymbol{P}_{\mathcal{V}} = \boldsymbol{V}\boldsymbol{V}^\top$ be the orthogonal projector onto $\mathcal{V}$, and define the distance of a point $\boldsymbol{x}$ from subspace $\mathcal{V}$  as
$
\operatorname{dist}(\boldsymbol{x},\mathcal{V}) := \inf_{\boldsymbol{v} \in \mathcal{V}} \|\boldsymbol{x}-\boldsymbol{v}\|_2
$. The unit sphere in $\mathbb{R}^d$ is denoted by $\mathbb{S}^{d-1}$.

\noindent\textbf{Multi-dimensional superposition hypothesis.}~With the notation equipped, we now define our data generation process based on the multi-dimensional superposition hypothesis on LLM representations~\citep{Modell_Rubin-Delanchy_Whiteley_2025, Engels_Michaud_Liao_Gurnee_Tegmark_2025}. We model an activation $\boldsymbol{h}\in\mathbb{R}^d$ of an LLM at a fixed layer and token position (e.g., a residual stream vector) as a superposition of a small number of low-dimensional features.
Let $\mathcal{I}:=[F]$ index the features. Each feature $i\in\mathcal{I}$ has a corresponding subspace $\mathcal{V}_i \subset \mathbb{R}^d$ with feature-subspace dimension $1\le d_i \le d$ and an orthonormal basis
$\boldsymbol{V}_i\in\mathbb{R}^{d\times d_i}$. Here $d_i$ denotes the dimension of the linear subspace containing the feature and may exceed the dimension of its manifold. The following multi-dimensional superposition hypothesis models our generation process for $\boldsymbol{h}$.

\label{sec:mdsh}

\begin{hypothesis}
\label{ass:mdsuperposition}
There exist feature-subspace dimensions $\{d_i\ge 1\}_{i=1}^F$ and matrices
$\{\boldsymbol{V}_i\in\mathbb{R}^{d\times d_i}\}_{i=1}^F$ such that activations drawn from the modeled distribution admit the representation\vspace{-2mm}
\begin{equation}
\boldsymbol{h} \;=\; \sum\nolimits_{i=1}^{F} \boldsymbol{V}_i \boldsymbol{z}_i \;+\; \boldsymbol{\xi}\vspace{-0.2cm}
\label{eq:mdsh-representation}
\end{equation}
where $\boldsymbol{z}_i\in\mathbb{R}^{d_i}$ are feature coordinates and $\boldsymbol{\xi}\in\mathbb{R}^d$ is noise, subject to the following conditions:
\begin{enumerate}[noitemsep,topsep=0pt,parsep=0pt,partopsep=0pt]
\item \textbf{(Low-rank support)} $\operatorname{rank}(\boldsymbol{V}_i)=d_i$.
\item \textbf{(Sparsity)} There exists $1\le s \ll F$ such that $|\{i:\ \boldsymbol{z}_i\neq \boldsymbol{0}\}|\le s$.
\item \textbf{(Coherence)} There exists $0\le \mu \ll 1$ such that for all $i\neq j$, $\boldsymbol{V}_i$ and $\boldsymbol{V}_j$ are $\mu$-coherent, meaning $|\boldsymbol{u}^\top \boldsymbol{v}| \le \mu$ for all unit vectors $\boldsymbol{u} \in \mathrm{col}(\boldsymbol{V}_i)$ and $\boldsymbol{v} \in \mathrm{col}(\boldsymbol{V}_j)$. For block recovery, we assume $s<\frac{1}{2}(1+\frac{1}{\mu})$ \citep{eldar2010block}.
\end{enumerate}
\end{hypothesis}
We denote the resulting distribution over hidden states by $\mathcal{D}$. For the analysis, we assume access to i.i.d. samples $\mathcal{S}:=\{\boldsymbol{h}_i\}_{i=1}^n$ with $\boldsymbol{h}_i \sim \mathcal{D}$. Hypothesis~\ref{ass:mdsuperposition} raises two non-convex inference problems:
(i) identifying the active set for an $\boldsymbol{h}$, i.e., $\{i:\boldsymbol{z}_i \neq \boldsymbol{0}\}$, and (ii) identifying the feature subspaces $\mathcal{V}_i$ for the active features. A sparse autoencoder aims to address these two inference problems through its \emph{encoder} and \emph{decoder}.
To analyze a single feature in isolation, we write $\mathcal{M}_i(t)\subseteq\{\boldsymbol{h}\in\mathcal{V}_i:\ \|\boldsymbol{h}\|_2=t\}$ for the \emph{idealized noiseless single-feature slice} of feature $i$, consisting of its possible realizations at radius $t>0$. We assume this slice is nonempty. Appendix~\ref{apx:low-dim-validation} demonstrates the low-dimensional structure in raw LLM activations.


\noindent\textbf{Standard sparse autoencoders.}~A \emph{standard sparse autoencoder (SAE)} with $m$ latents is parameterized by an encoder $\boldsymbol{E}\in\mathbb{R}^{d\times m}$ and a decoder $\boldsymbol{D}\in\mathbb{R}^{d\times m}$.
\label{sec:standard-sae}
Define the sparse code (the vector of latent activations) as $\boldsymbol{a}(\boldsymbol{h}) := \sigma(\boldsymbol{E}^\top\boldsymbol{h})$ (for simplicity, we drop biases). The SAE decoder associates each latent with a single direction $\boldsymbol{d}_j\in\mathbb{R}^d$, so that $\hat{\boldsymbol{h}}=\sum_j [\boldsymbol{a}(\boldsymbol{h})]_j \boldsymbol{d}_j$.
For analysis, we take the general objective form to be the standard reconstruction and $\ell_1$ sparsity objective:
\begin{equation}
\min_{\boldsymbol{E},\boldsymbol{D}}
\ 
\sum_{i=1}^n\left[\tfrac{1}{2}\big\|\boldsymbol{h}_i-\boldsymbol{D}\,\boldsymbol{a}(\boldsymbol{h}_i)\big\|_2^2
\;+\; \lambda \,\|\boldsymbol{a}(\boldsymbol{h}_i)\|_1\right]\qquad \text{s.t.} \quad\|\boldsymbol{d}_j\|_2=1\quad \forall \boldsymbol{d}_j  \in \boldsymbol{D}.
\label{eq:standard-objective}
\end{equation}
These models aim to infer feature activity through $\boldsymbol{a}(\boldsymbol{h})$ and feature geometry through $\boldsymbol{D}$. We therefore examine how many decoder directions are required for sparse reconstruction of a multi-dimensional feature and whether the $\ell_1$-regularized objective favors non-basis solutions.
\section{On the Necessity of Subspace Learning}
\label{sec:necessity}

In this section, we formalize the geometric inductive bias imposed by standard SAEs and analyze its consequences for representing multidimensional features. All proofs are in Appendix~\ref{apx:theory-necessity}.

{\noindent\textbf{Standard SAE inductive bias on $\boldsymbol{V}_i$s.}~In a standard SAE, given an activation $\boldsymbol{h}$, a feature $i\in\mathcal{I}$ is deemed \emph{active} whenever $[\boldsymbol{a}(\boldsymbol{h})]_i \neq 0$. This gating rule implicitly associates each latent coordinate with a single decoder direction, i.e., it posits $\boldsymbol{V}_i \in \mathbb{R}^{d\times 1}$ and treats $\boldsymbol{V}_i$ as a \emph{vector}.
This vector-based view is motivated by two common hypotheses.
First, the \emph{One-dimensional Linear Representation Hypothesis} \cite{Alain_Bengio_2018, Park_Choe_Veitch_2024} suggests that many high-level features are approximately linear directions in activation space, so they can be captured by a single vector and read out via projection.
Second, the \emph{superposition} hypothesis proposes that many such features are simultaneously encoded in a layer by an overcomplete collection of nearly orthogonal directions, and that SAEs can recover these directions \cite{Elhage_Hume_Olsson_Schiefer_Henighan_Kravec_Hatfield-Dodds_Lasenby_Drain_Chen_etal._2022,Bricken_Templeton_Batson_Chen_Jermyn_Conerly_Turner_Anil_Denison_Askell_etal._2023}.
In the ideal regime where this one-dimensional LRH holds, each semantic feature corresponds to a single decoder vector, and sparsity encourages each example to activate only the direction for the features it contains. With one latent per feature, restricting to a single-feature slice therefore gives a one-vector reconstruction \cite{Gao_Tour_Tillman_Goh_Troll_Radford_Sutskever_Leike_Wu_2024, Lee_Davies_Canby_Hockenmaier_2025}.}

When a feature is inherently multi-dimensional and irreducible, this assumption carries two consequences that we make precise in this section. The first consequence is purely geometric and states that, under local richness assumptions, any decoder restricted to co-activate at most $k < d_i$ directions per input feature must allocate a number of atoms that grows without bound as the demanded reconstruction error shrinks, causing splitting, while $d_i$ basis directions suffice once $k$ reaches $d_i$. We then extend these bounds to the full superposition setting of Hypothesis~\ref{ass:mdsuperposition} with noise and cross-feature interference (Section~\ref{sec:covering}). The second consequence--and perhaps the more important one--is that the standard SAE objective prefers additional in-subspace non-basis directions over the basis alone (Section~\ref{sec:instability}). With latent activations $\boldsymbol{a}(\boldsymbol{h})$ optimized separately for each input $\boldsymbol{h}$, adding such a direction strictly lowers the $\ell_1$-regularized population risk.  Appendix~\ref{apx:clusters} shows this redundancy within clusters of SAE decoder while Appendix~\ref{apx:instability} further shows that for a decoder of basis padded with duplicate columns, there exists a continuous path from it to a strictly lower population risk. Together, these results motivate the subspace-aware design of the SASA algorithm in Section~\ref{sec:sasa}. 

We develop these two consequences in turn, starting with the geometric one and asking how many decoder directions a standard SAE must allocate to reconstruct every realization of a single multi-dimensional feature. 
Then, we analyze how $\ell_1$ regularization of SAEs leads to fragmented decoders.

\subsection{The subspace-covering phase transition}
\label{sec:covering}

Fix feature $i \in \mathcal{I}$ with dimension $d_i \ge 2$. On the single-feature slice $\mathcal{M}_i(t)$, where $\boldsymbol{h} = \boldsymbol{V}_i\boldsymbol{z}_i \in \mathcal{V}_i$ with $\|\boldsymbol{h}\|_2 = t$, a decoder with co-activation budget $k$ represents each input via the span of at most $k$ of its $N$ columns. We formalize the minimum $N$ required for uniform $\varepsilon$-approximation.

\begin{definition}[$k$-sparse covering number]
\label{def:k-cover}
For $1 \le k \le d_i$ and $0 \le \varepsilon < t$, define\vspace{-1mm}
\[
L_i^{(k)}(\varepsilon) := \min\left\{N : \exists\,\boldsymbol{d}_1,\dots,\boldsymbol{d}_N \in \mathbb{S}^{d-1}\text{s.t.}\sup_{\boldsymbol{h} \in \mathcal{M}_i(t)}\min_{\substack{J \subseteq [N], |J| \le k}} \operatorname{dist}\bigl(\boldsymbol{h},\operatorname{span}\{\boldsymbol{d}_j\}_{j \in J}\bigr) \le \varepsilon\right\}.\vspace{-4mm}
\]
\end{definition}
We set $L_i^{(k)}(\varepsilon)=+\infty$ if no finite dictionary exists. The feature $i$ \emph{splits at budget $k$} whenever $L_i^{(k)}(\varepsilon) > k$.
When $L_i^{(k)}(\varepsilon) > k$, the dictionary devotes more than $k$ columns to feature $i$ and distinct inputs route to distinct subsets, fragmenting the feature across decoder atoms. To ensure that feature $i$ is \emph{irreducible} to any dimension smaller than $d_i$, we assume a mild richness condition.

\begin{assumption}[Local richness]
\label{ass:local-richness}
Let $t>0$, $\dim(\mathcal{V}_i)=d_i\ge 2$. There exists $\boldsymbol{u}_0\in \mathcal{V}_i\cap\mathbb{S}^{d-1}$ and $\rho\in(0,\pi/2)$ such that $
t\cdot \big(B_{\gamma}(\boldsymbol{u}_0,\rho)\cap \mathcal{V}_i\big)\subseteq \mathcal{M}_i(t)$.
\end{assumption}
\vspace{-3mm}
Here the angular ball is $B_\gamma(\boldsymbol{u}_0,\rho):=\{\boldsymbol{u}\in\mathbb{S}^{d-1}:\gamma(\boldsymbol{u},\boldsymbol{u}_0)\le\rho\}$, where $\gamma(\boldsymbol{u},\boldsymbol{v}):=\arccos|\langle\boldsymbol{u},\boldsymbol{v}\rangle|$. To lower bound the $k$-dimensional spans needed to cover the feature slice, we  apply this local cap.
\begin{theorem}[Covering number of feature subspace]
\label{thm:line-covering}
Let $d_i \ge 2$ and Assumption~\ref{ass:local-richness} hold. For every $1 \le k < d_i$, there exists a constant $C := C(d_i,k,\rho) > 0$ such that, for all $\varepsilon \in (0, t\sin\rho)$,
\begin{equation}\label{eq:cover-lb}
L_i^{(k)}(\varepsilon) \;\ge\; C\!\left(\frac{t}{\varepsilon}\right)^{(d_i - k)/k}.
\end{equation}\vspace{-2mm}
When $k \ge d_i$, any orthonormal basis of $\mathcal{V}_i$ yields $L_i^{(d_i)}(0) = d_i$ with zero error and no splitting.
\end{theorem}
The exponent $(d_i-k)/k$ below decreases as $k$ increases. Once $k\ge d_i$, a basis suffices to reconstruct the slice without splitting.

\noindent\textbf{Extension to superposition.}~The preceding analysis considers a noiseless single-feature slice (see Appendix~\ref{apx:real-line-cover} for an empirical illustration of the $k=1$ case). The next proposition extends the same lower-bound phenomenon to the full superposition setting of Hypothesis~\ref{ass:mdsuperposition}, where activations may include both noise and interference from other active features.

\begin{proposition}[Superposition covering number extension]
\label{prop:superposition}
Fix a feature $i$ satisfying Assumption~\ref{ass:local-richness}. Suppose every point in $\mathcal{M}_i(t)$ occurs as the feature-$i$ component of an admissible input, and $N$ decoder columns reconstruct every such input to error $\varepsilon>0$ using at most $k$ columns in total, where $1\le k<d_i$. Under Hypothesis~\ref{ass:mdsuperposition} with $\mu(2s-1) < 1$, $\|\boldsymbol{\xi}\|_2 \le \eta$, and $\|\boldsymbol{z}_j\|_2 \le t$, the bound~\eqref{eq:cover-lb} holds with $\varepsilon$ replaced by $\varepsilon + \eta + \mu st$, provided the effective error remains below $t\sin\rho$.
\end{proposition}
\subsection{Feature splitting under $\ell_1$ regularization}
\label{sec:instability}
The covering bound of Section~\ref{sec:covering} is purely geometric. It counts \emph{how many} atoms are needed to cover the feature slice, yet it reveals nothing about which dictionary optimization actually selects. To study objective preference, we turn from counting to optimization, studying the landscape that the standard SAE objective of Eq.~\ref{eq:standard-objective} induces over the decoder alone.

Throughout this subsection, we take $\mathcal{M}_i(t)$ to be the full radius-$t$ sphere in $\mathcal{V}_i$. Recall that this objective penalizes reconstruction error together with an $\ell_1$ term on the codes while constraining the decoder columns to unit norm. For a fixed decoder $\boldsymbol{D}$, the smallest loss attainable on an input $\boldsymbol{h}$ is obtained by optimizing the codes. Averaging this loss over inputs drawn uniformly from $\mathcal{M}_i(t)$ defines the population risk
\begin{equation}\label{eq:risk}
R(\boldsymbol{D}) \;:=\; \underset{\boldsymbol{h}\sim \mathrm{Uniform}(\mathcal{M}_i(t))}{\mathbb{E}}\left[\min_{\boldsymbol{a}(\boldsymbol{h}) \in \mathbb{R}^m}\;\tfrac{1}{2}\|\boldsymbol{h} - \boldsymbol{D}\boldsymbol{a}(\boldsymbol{h})\|_2^2 + \lambda\|\boldsymbol{a}(\boldsymbol{h})\|_1\right].
\end{equation}
For width $m$, the feasible decoders form the unit-norm column manifold $\mathcal{D}_m = \{\boldsymbol{D} \in \mathbb{R}^{d \times m} : \|\boldsymbol{d}_j\|_2 = 1\;\forall j\}$. Any encoder, whether amortized as in SAEs or solved exactly, must produce codes whose loss is at least this inner minimum, and so $R(\boldsymbol{D})$ lower bounds the population loss of every SAE that uses the same decoder and input distribution. The results below concern this $\ell_1$-regularized risk and do not establish preferences under other sparsity objectives and mechanisms (e.g., TopK).


We examine whether adding a decoder direction can lower the risk of $\boldsymbol{D}_{\mathrm{basis}} = [\boldsymbol{v}_1,\dots,\boldsymbol{v}_{d_i}]$, even though its columns already span $\mathcal{V}_i$. A loss decrease on a set of positive probability suffices to lower this risk, because the added coefficient can be set to zero elsewhere. To identify such a direction, we compute the residual left by the optimal basis code. For $\boldsymbol{h}=t\boldsymbol{u}$ with $\boldsymbol{u}\in\mathbb{S}^{d-1}\cap\mathcal{V}_i$, write $u_k=\boldsymbol{v}_k^\top\boldsymbol{u}$. Because the basis columns are orthonormal, the code coefficients can be optimized separately, giving $a_k^\star = \operatorname{sign}(u_k)(t|u_k| - \lambda)_+$ (Lemma~\ref{lem:residual}). This shrinkage leaves the reconstruction residual $\boldsymbol{r}(\boldsymbol{h})=\boldsymbol{h}-\boldsymbol{D}_{\mathrm{basis}}\boldsymbol{a}^\star(\boldsymbol{h})$.

Consider inputs for which every coefficient of the optimal basis code is nonzero. Their unit directions form the set $\mathcal{A} = \{\boldsymbol{u} \in \mathbb{S}^{d-1}\cap\mathcal{V}_i : t|u_k| > \lambda\;\forall\,k\}$, which has positive measure when $0<\lambda<t/\sqrt{d_i}$. On this set, soft thresholding reduces $t|u_k|$ by $\lambda$, leaving the coordinate residual $tu_k-a_k^\star=\lambda\operatorname{sign}(u_k)$. Thus the residual is $\boldsymbol{r}(\boldsymbol{h}) = \lambda\sum\nolimits_{k=1}^{d_i}\operatorname{sign}(u_k)\,\boldsymbol{v}_k$. For an added unit direction $\widetilde{\boldsymbol{d}}$, the optimality condition for keeping its coefficient at zero while retaining the basis code is $|\widetilde{\boldsymbol{d}}^\top\boldsymbol{r}(\boldsymbol{h})|\le\lambda$. If this inactivity certificate fails, a sufficiently small coefficient with the sign of $\widetilde{\boldsymbol{d}}^\top\boldsymbol{r}(\boldsymbol{h})$ reduces the reconstruction term by more than its $\ell_1$ cost. For inputs in $\mathcal{A}$ with $u_k>0$ for every $k$, take the fixed unit direction $\widetilde{\boldsymbol{d}}=\frac{1}{\sqrt{d_i}}\sum_{k=1}^{d_i}\boldsymbol{v}_k$.  Then it follows that  $\widetilde{\boldsymbol{d}}^\top\boldsymbol{r}(\boldsymbol{h})=\lambda\sqrt{d_i}>\lambda$ for $d_i\ge2$. Thus the same added direction lowers the loss throughout this region of positive probability. Theorem~\ref{thm:violation} extends this argument to any unit non-basis direction in $\mathcal{V}_i$ by considering inputs with matching coordinate signs.
\begin{theorem}[Dual certificate violation]
\label{thm:violation}
Let $d_i \ge 2$ and $\lambda \in (0,t/\sqrt{d_i})$. For any $\widetilde{\boldsymbol{d}} \in \mathrm{col}(\boldsymbol{D}_{\mathrm{basis}})$ with basis coordinates $\boldsymbol{c} = (c_1,\dots,c_{d_i})^\top$, $\|\boldsymbol{c}\|_2 = 1$, $\widetilde{\boldsymbol{d}} \neq \pm\boldsymbol{v}_k\;\forall k$, for $\boldsymbol{h}= t\boldsymbol{u}$,  the inactivity certificate for $\widetilde{\boldsymbol{d}}$ is violated on the positive-measure region $\mathcal{U}_{\boldsymbol{c}} = \{\boldsymbol{u} \in \mathcal{A} : \operatorname{sign}(u_k) = \operatorname{sign}(c_k)\ \text{for all }k\text{ with }c_k\neq0\}$:
\begin{equation}\label{eq:violation}
\widetilde{\boldsymbol{d}}^\top \boldsymbol{r}(\boldsymbol{h}) = \lambda\|\boldsymbol{c}\|_1 > \lambda.
\end{equation}
\end{theorem}
This certificate violation means that adding a non-basis decoder direction $\widetilde{\boldsymbol{d}}$ can lower the loss on a positive-measure set of activations. We next turn this local violation into a strict risk decrease and then into instability of the basis dictionary.
\begin{proposition}
\label{prop:reduction}
Under the assumptions of Theorem~\ref{thm:violation}, let $\boldsymbol{D}=[\boldsymbol{D}_{\mathrm{basis}},\,\widetilde{\boldsymbol{d}}]$. Then
\begin{equation}\label{eq:strict-risk-gap}
R(\boldsymbol{D}_{\mathrm{basis}})-R(\boldsymbol{D})
\;\ge\;
\frac{\lambda^2}{2}\bigl(\|\boldsymbol{c}\|_1-1\bigr)^2\,\Pr(\boldsymbol{u}\in\mathcal{U}_{\boldsymbol{c}})
\;>\;0.
\end{equation}
\end{proposition}
\begin{remark}
The risk decrease in Proposition~\ref{prop:reduction} occurs without enlarging the decoder span, showing that the $\ell_1$-regularized risk favors additional non-basis directions over the basis alone. At fixed dictionary width, Theorem~\ref{thm:instability} in Appendix~\ref{apx:instability} shows that a basis padded with duplicate columns is not a local minimizer. An arbitrarily small rotation of a duplicate lowers the risk while preserving the original basis. This separates representational capacity from objective preference and complements Bhalla et al.~\cite{bhalla2026sparse}, who give conditions for compact manifold capture by an idealized sparse decoder with a suitable fixed dictionary.
\end{remark}

\section{Subspace-Aware Sparse Autoencoders}
\label{sec:sasa}
Section~\ref{sec:necessity} established a covering cost when fewer than $d_i$ atoms co-activate and showed that coordinatewise $\ell_1$ regularization favors additional non-basis directions. To alleviate this issue, we propose \emph{SASA}, which considers a subspace as the unit of representation through block decoders $\boldsymbol{D}_k$ and moves sparsity to the level of blocks, so that one block carries an entire feature at the price of a single active group. We rigorously show that a single SASA block can represent the feature slice and every population minimizer recovers its subspace. All proofs are deferred to Appendix~\ref{apx:theory-sasa}. 

\noindent\textbf{Model.}~Let $\boldsymbol{h}\in\mathbb{R}^d$ denote the input and let $m=Kr$ be the latent dimension. We parameterize an encoder $\boldsymbol{p}(\boldsymbol{h})=\boldsymbol{E}\boldsymbol{h}\in\mathbb{R}^{m}$ with $\boldsymbol{E}=[\boldsymbol{E}_1^\top\ \cdots\ \boldsymbol{E}_K^\top]^\top\in\mathbb{R}^{m\times d}$ where $\boldsymbol{E}_k\in\mathbb{R}^{r\times d}, k \in [K]$, so that the pre-activations partition as $\boldsymbol{p}_k(\boldsymbol{h})=\boldsymbol{E}_k\boldsymbol{h}\in\mathbb{R}^r$. The Top-$s$ gate $\mathcal{T}_s(\boldsymbol{h})\in\arg\max_{T\subset[K],\,|T|=s}\sum_{k\in T}\|\boldsymbol{p}_k(\boldsymbol{h})\|_2$ keeps the $s$ groups of largest norm and the block-sparse code is defined as 
\[
\boldsymbol{a}_k(\boldsymbol{h})=
\begin{cases}
\boldsymbol{p}_k(\boldsymbol{h}), & k\in\mathcal{T}_s(\boldsymbol{h}),\\
\boldsymbol{0}, & \text{otherwise}.
\end{cases}
\]
The decoder is block-columned $\boldsymbol{D}=[\boldsymbol{D}_1\ \ldots\ \boldsymbol{D}_K]\in\mathbb{R}^{d\times m}$ with $\boldsymbol{D}_k\in\mathbb{R}^{d\times r}$, and the reconstruction is $\boldsymbol{D}\boldsymbol{a}(\boldsymbol{h})=\sum_{k=1}^K \boldsymbol{D}_k \boldsymbol{a}_k(\boldsymbol{h})$. Because $\mathcal{T}_s$ ranks groups solely by the scalar $\|\boldsymbol{p}_k(\boldsymbol{h})\|_2$ and the construction charges no per-coordinate penalty inside a group, it removes the coordinatewise sparsity incentive analyzed in Section~\ref{sec:instability}.

\noindent\textbf{Objective.}~We analyze SASA on activations $\mathcal{S}=\{\boldsymbol{h}_i\}_{i=1}^n$ by minimizing the reconstruction error under the Top-$s$ block-sparse gate together with a spectral regularizer acting on each block
\begin{equation}
\label{eq:sasa-objective}
\min_{\boldsymbol{E},\boldsymbol{D}}\
\frac{1}{n}\sum\nolimits_{i=1}^n \big\|\boldsymbol{h}_i-\boldsymbol{D}\boldsymbol{a}(\boldsymbol{h}_i)\big\|_2^2
+\lambda_{\mathrm{dim}}\sum\nolimits_{k=1}^{K}\big\|\boldsymbol{D}_k\boldsymbol{E}_k\big\|_*.
\end{equation}
The theoretical results below do not impose unit-norm constraints on encoder rows or decoder columns. The regularizer in Eq.~\ref{eq:sasa-objective} acts on the spectrum of the reconstruction map $\boldsymbol{W}_k=\boldsymbol{D}_k\boldsymbol{E}_k$. The nuclear norm admits the variational identity $\|\boldsymbol{W}_k\|_*=\min_{\boldsymbol{D}_k\boldsymbol{E}_k=\boldsymbol{W}_k}\tfrac12(\|\boldsymbol{D}_k\|_{\mathrm{F}}^2+\|\boldsymbol{E}_k\|_{\mathrm{F}}^2)$, which equals the minimum symmetric weight-decay penalty among encoder and decoder matrices whose product is $\boldsymbol{W}_k$ \citep{Kunin_Bloom_Goeva_Seed_2019}. This penalty shrinks the singular values of each block map, allowing its effective rank to adapt to the feature dimension.

The penalty can be evaluated in batch through the Gram trace identity $\sum_{j=1}^{r}\sigma_j(\boldsymbol{W}_k)=\mathrm{tr}\big((\boldsymbol{W}_k^\top\boldsymbol{W}_k)^{1/2}\big)$. An auxiliary loss \citep{Gao_Tour_Tillman_Goh_Troll_Radford_Sutskever_Leike_Wu_2024} trains inactive coordinates to reconstruct the residual. We use it as a stabilization heuristic, with details provided in Appendix~\ref{apx:alg-aux} and full training and inference procedures in Appendix~\ref{apx:alg-pseuodo}.

\noindent\textbf{Single-block capacity.}~We first show that a single block of sufficient width can contain the feature.
\begin{proposition}[SASA on a single-feature slice]
\label{prop:sasa-no-splitting}
Fix $i\in\mathcal{I}$ and suppose $\dim(\mathcal{V}_i)=d_i\ge 1$. If SASA uses block size $r\ge d_i$ and Top-$s$ gating with $s=1$, then there exist decoder blocks $\{\boldsymbol{D}_k\}_{k=1}^K$ and some $k^\star\in[K]$ such that $\sup_{\boldsymbol{h}\in\mathcal{M}_i(t)} \operatorname{dist}\left(\boldsymbol{h},\operatorname{col}(\boldsymbol{D}_{k^\star})\right)=0$. In particular, $L_i^{(d_i)}(0)=d_i$.
\end{proposition}
\noindent\textbf{Single-block optimality.}~Proposition~\ref{prop:sasa-no-splitting} shows that a block of width $r\ge d_i$ can carry the feature, yet capacity does not determine which representation the objective favors. Section~\ref{sec:instability} established the analogous capacity for the standard SAE while showing that its $\ell_1$ risk decreases when a non-basis direction is added. Under the same uniform sampling model on the full radius-$t$ sphere $\mathcal{M}_i(t)$ in $\mathcal{V}_i$, the SASA population objective is $R(\boldsymbol{E},\boldsymbol{D}) =\underset{\boldsymbol{h}\sim\mathrm{Uniform}(\mathcal{M}_i(t))}{\mathbb{E}} \|\boldsymbol{h}-\boldsymbol{D}\boldsymbol{a}(\boldsymbol{h})\|_2^2 +\lambda_{\mathrm{dim}}\sum_{k}\|\boldsymbol{W}_k\|_*$. The Top-$1$ gate routes each activation to one block, so splitting means assigning different parts of the same feature slice to different blocks. To compare these assignments, let $\mathcal{C}_k$ be the inputs routed to block $k$ and define its contribution to the population second moment
$\boldsymbol{M}_k=\mathbb{E}\big[\boldsymbol{h}\boldsymbol{h}^{\top}\,\mathbbm{1}\{\boldsymbol{h}\in\mathcal{C}_k\}\big]$.
Since the routing blocks partition the slice, these moments satisfy
$\sum_k\boldsymbol{M}_k=\boldsymbol{\Sigma}_i=\tfrac{t^2}{d_i}\boldsymbol{P}_{\mathcal{V}_i}$.
Changing the routing therefore redistributes the same second moment across blocks.

We use the block moments $\boldsymbol{M}_k$ to derive a lower bound on the population objective $R$ that holds for every encoder and decoder. Every block of positive probability contains variation along all $d_i$ directions of $\mathcal{V}_i$, even if it receives only part of the slice. When $0<\lambda_{\mathrm{dim}}<t^2/d_i$, the proof uses this property to show that attaining the bound requires routing almost every input to the same block. An explicit single-block construction attains the bound, giving the following recovery guarantee.



\begin{theorem}[Block recovery on the slice]
\label{thm:block-recovery}
Fix $i\in\mathcal{I}$ with $\dim(\mathcal{V}_i)=d_i\ge 2$, let SASA use block size $r\ge d_i$ and
Top-$s$ gating with $s=1$, and fix $0<\lambda_{\mathrm{dim}}<t^2/d_i$. Let $R$ denote the population
objective under uniform sampling from the full radius-$t$ sphere $\mathcal{M}_i(t)$ in $\mathcal{V}_i$. Then every global minimizer of $R$ activates a single block
$k^\star$ almost surely, with $\operatorname{col}(\boldsymbol{W}_{k^\star})=\mathcal{V}_i$. The active reconstruction map is unique up to the index $k^\star$. If $r=d_i$, its $d_i$ decoder columns span this eigenspace as well. \vspace{-1mm}
\end{theorem}
In the single-feature setting of Theorem~\ref{thm:block-recovery}, every global minimizer of the SASA population objective routes almost every realization of the feature to the same block, whose reconstruction map spans $\mathcal{V}_i$. This contrasts with the $\ell_1$-regularized SAE analysis in Section~\ref{sec:instability}, where adding a non-basis direction to the basis decoder strictly lowers the risk. \vspace{-1mm}

\noindent\textbf{From the slice to superposition.}~For superposed activations under Hypothesis~\ref{ass:mdsuperposition}, several feature blocks must be identified from the same input. For a known dictionary whose blocks have orthonormal columns, the condition $s<\frac{1}{2}(1+\frac{1}{\mu})$ is sufficient for exact recovery of the active blocks from noiseless inputs by block orthogonal matching pursuit \citep{eldar2010block}. SASA amortizes this inference with a learned encoder and Top-$s$ norm gate. Extending the single-feature routing guarantee to superposed inputs therefore requires establishing correct selection of the active blocks by this amortized encoder, which is outside the scope of the present work.

\section{Sample Complexity Analysis}
\label{sec:sample-compelxity}
A major practical limitation of SAEs is their high training cost. This bottleneck stems from the nature of the training data, since \textit{each} input activation $\boldsymbol{h}$ is not drawn from a static dataset but produced by a forward pass through an LLM that extracts an intermediate representation at a specified layer. Generating every training example therefore incurs the cost of an LLM inference, and data acquisition itself becomes the dominant cost. 

\begin{wrapfigure}{r}{0.5\textwidth}
    \centering\vspace{-0.75cm}
    \captionsetup{font=small,skip=4pt,justification=raggedright,singlelinecheck=false}
    \includegraphics[width=\linewidth]{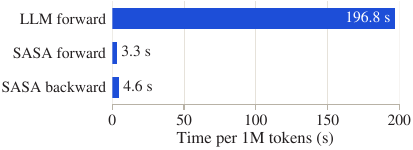}
    \caption{\textbf{Data acquisition dominates training time.} A full Mistral-7B forward pass that caches layer-8 activations, versus the SASA forward and backward passes on them, averaged over $20$ runs.}\vspace{-0.25cm}
    \label{fig:mistral-wall-clock}
\end{wrapfigure}

On Mistral-7B, for example, producing the activations for $1\mathrm{M}$ tokens takes about $25$ times as long as the SASA forward and backward passes on them (Figure~\ref{fig:mistral-wall-clock}). Splitting a feature across many atoms can increase this acquisition cost because collecting examples for every atom requires sampling all of their routing regions. We therefore compare the sample complexity of estimating a single feature subspace with that of observing every atom in a fixed splitting SAE. We first derive a lower bound on the samples needed to route at least one input to every atom. All proofs are provided in Appendix~\ref{apx:theory-sample-complexity}.


\begin{proposition}
\label{prop:split-sample-complexity}
Consider a standard SAE with fixed decoders $\boldsymbol{d}_1,\ldots,\boldsymbol{d}_N\in\mathbb{S}^{d-1}$, $N\ge2$, co-activation budget $k=1$, and routing rule $J(\boldsymbol{h})\in\arg\min_{j\in[N]}\operatorname{dist}\big(\boldsymbol{h},\operatorname{span}(\boldsymbol{d}_j)\big)$ with fixed tie-breaking. Let $C_j$ denote the number of i.i.d. training samples $\boldsymbol{h}_1,\dots,\boldsymbol{h}_n$ with $n\ge1$ routed to index $j$. Fix $\delta\in(0,1)$. If $\mathbb{P}(C_j\ge1\ \text{for all }j\in[N])\ge1-\delta$, then $n\ge(N-1)\log\big(\frac{N(1-\delta)}{\delta}\big)$.
\end{proposition}

\noindent Theorem~\ref{thm:block-recovery} establishes subspace recovery at the population optimum of the SASA objective. We now bound the number of samples  needed to estimate its subspace using the empirical second moment.
\begin{theorem}
\label{thm:sasa-sample-complexity}
Fix feature $i\in\mathcal{I}$ and its subspace $\mathcal{V}_i\subset \mathbb{R}^d$ with subspace dimension $d_i=\dim(\mathcal{V}_i)\ge 2$ and radius $t>0$. Let $\boldsymbol{h}_1,\dots,\boldsymbol{h}_n$ be i.i.d. uniform samples from the full radius-$t$ sphere $\mathcal{M}_i(t)$ in $\mathcal{V}_i$, let $\widehat{\boldsymbol{\Sigma}}_i:=\frac{1}{n}\sum_{j=1}^{n}\boldsymbol{h}_j\boldsymbol{h}_j^{\top}$, and let $\widehat{\mathcal{V}}_i$ be the span of the top $d_i$ eigenvectors of $\widehat{\boldsymbol{\Sigma}}_i$. Then for any $\varepsilon\in(0,t)$ and $\delta\in(0,1)$, if $n \ge 128\,d_i^2 \log\big(\frac{2d}{\delta}\big) \frac{t^2}{\varepsilon^2}$, then with probability at least $1-\delta$, $\sup_{\boldsymbol{h}\in\mathcal{M}_i(t),\ \|\boldsymbol{h}\|_2=t}\operatorname{dist}\big(\boldsymbol{h},\widehat{\mathcal{V}}_i\big)\le \varepsilon$.
\end{theorem}

\noindent These results highlight a potential increase in sample complexity due to splitting. Under the stated single-feature model, a sufficient condition for estimating the subspace $\mathcal{V}_i$ is $n \ge 128\,d_i^2\log(\frac{2d}{\delta})\frac{t^2}{\varepsilon^2}$, which depends polynomially on $d_i$ (Theorem~\ref{thm:sasa-sample-complexity}). In contrast, for a line-covering dictionary that reconstructs the feature, Theorem~\ref{thm:line-covering} gives $N\ge C(\frac{t}{\varepsilon})^{d_i-1}$, the $k=1$ instance of the general $\Omega((\frac{t}{\varepsilon})^{(d_i-k)/k})$ bound. Proposition~\ref{prop:split-sample-complexity} then requires $n=\Omega(N\log N)$ samples at fixed confidence for every direction to receive a routed sample. Since the cost of data acquisition scales linearly with $n$ and each sample requires a forward pass of the LLM, under the stated single-feature assumptions, subspace estimation can require fewer activations and thus fewer LLM forward passes.

\section{Experiments}
\label{sec:exp}
We conduct experiments on real LLMs and synthetic data to evaluate SASA's reconstruction, feature absorption, interpretability, and causal interventions. Additional results are reported in Appendix~\ref{apx:exp}.  

\noindent\textbf{Experimental setup.}~We train SASA on residual-stream activations from two pretrained LLMs, GPT-2 Small ($d=768$, layer~7) and Mistral-7B-v0.1 ($d=4096$, layer~8). Following standard SAE practice, we sample text from a large web corpus, run the frozen LLM to cache residual-stream activations, and optimize an autoencoder on the resulting activation stream. We use OpenWebText for GPT-2 and Pile for Mistral. Each training run uses one A100 80GB GPU.

We train models on the same data, with the same token budget, optimizer, learning rate, and decay schedule. The comparison for GPT-2 uses $24{,}576$ decoder columns, $300\mathrm{M}$ OpenWebText tokens, and target scalar sparsity $\ell_0=120$. SASA uses $K=4096$ groups of width $r=6$, with $s=20$ active groups per token. Dictionary width counts decoder columns for every model, while a feature unit is one scalar SAE latent or one SASA group. Full training details are provided in Appendix~\ref{apx:training}.

We compare SASA against  TopK~\citep{Gao_Tour_Tillman_Goh_Troll_Radford_Sutskever_Leike_Wu_2024}, BatchTopK~\citep{Bussmann_Leask_Nanda_2024}, and Matryoshka~\citep{Bussmann_Nabeshima_Karvonen_Nanda_2025} SAEs under this matched protocol. We also compare SASA with the ReLU SAEs for GPT-2~\citep{bloom2024saetrainingcodebase} and Mistral-7B~\citep{Engels_Michaud_Liao_Gurnee_Tegmark_2025} (Table~\ref{tab:reference-configs}).

\noindent\textbf{Core metrics.}~In the GPT-2 comparison, SASA achieves lower absorption and more accurate prediction of family members at a reconstruction cost (Table~\ref{tab:matched-budget}). We measure reconstruction quality using the CE score~\citep{Rajamanoharan_Conmy_Smith_Lieberum_Varma_Kramár_Shah_Nanda_2024}, computed from the LLM's next-token cross-entropy after replacing activations with their reconstructions. The score is normalized to 0\% for zero ablation and 100\% for the original activations. SASA achieves 98.7\%, compared with 99.6\% for TopK and 99.7\% for BatchTopK.

Beyond reconstruction, we evaluate \emph{feature absorption}, which occurs when the main latents associated with a feature are weak or inactive while other latents represent it. We use the first-letter benchmark~\citep{Chanin_Wilken-Smith_Dulka_Bhatnagar_Golechha_Bloom_2025}, evaluating scalar coordinates for every model, including SASA. Mean fractional absorption quantifies the contribution of other latents when a letter's main latents do not fully account for its representation. Full absorption measures how often all main latents fall below the benchmark's activation threshold while another latent explains a substantial part of the letter's representation. SASA has the lowest mean and full absorption in the GPT-2 comparison. Full absorption is 8.7\%, compared with 21.4\% for BatchTopK, the best vector baseline. Against the  Mistral-7B SAE, SASA reduces mean and full absorption from 66.1\% to 25.9\% and 17.0\% to 4.4\%, respectively (Table~\ref{tab:absorption_summary}).

We extend the absorption evaluation to ten concept families (weekdays, months, countries, US cities, numbers, colors, directions, seasons, languages, and animals), using natural occurrences of their members in the Pile dataset~\citep{Gao_Biderman_Black_Golding_Hoppe_Foster_Phang_He_Thite_Nabeshima_etal._2020}. Here, we score the decoded contribution of an entire SASA group, while retaining individual latents as the units for vector SAEs. This makes the SASA score independent of the choice of coordinates within a group. SASA achieves mean family absorption of 23.9\%, compared with 46.4\% the best vector baseline on this task.

We next test whether units selected to detect a concept family also encode which member is present in the input. For example, we use the code of a group selected to detect weekday names to predict which weekday is mentioned. For each family, we select one width-6 SASA group or up to six SAE latents by family-detection F1, giving every model the same maximum budget of six decoder columns. We fit a linear classifier on the selected code coordinates and evaluate member prediction on held-out examples. SASA achieves 86.2\% accuracy, compared with 68.9\% for BatchTopK, the best vector baseline. Additional results are provided in Appendix~\ref{apx:value-readout}.

\begin{table}[tbp]
\centering
\caption{GPT-2 results. Values are mean $\pm$ standard deviation over three seeds. Member prediction accuracy is averaged over ten concept families.}
\label{tab:matched-budget}
\footnotesize
\setlength{\tabcolsep}{3.5pt}
\resizebox{\linewidth}{!}{%
\begin{tabular}{@{}lccccccc@{}}
\toprule
\multirow{2}{*}{Model} & \multirow{2}{*}{CE score $\uparrow$} & \multicolumn{2}{c}{First-letter absorption $\downarrow$} & \multicolumn{2}{c}{Family absorption $\downarrow$} & \multirow{2}{*}{\shortstack{Member\\prediction $\uparrow$}} & \multirow{2}{*}{\shortstack{Auto-\\Interp $\uparrow$}} \\
\cmidrule(lr){3-4}\cmidrule(lr){5-6}
 & & Mean & Full & Mean & Full & & \\
\midrule
TopK & $99.6\%$ & $21.7\%{\pm}3.0\%$ & $22.8\%{\pm}2.1\%$ & $52.7\%{\pm}18.0\%$ & $25.0\%{\pm}14.4\%$ & $63.8\%{\pm}0.9\%$ & $86.3\%{\pm}0.2\%$ \\
BatchTopK & $\mathbf{99.7}\%$ & $25.3\%{\pm}2.3\%$ & $21.4\%{\pm}2.1\%$ & $55.2\%{\pm}1.4\%$ & $29.6\%{\pm}8.7\%$ & $68.9\%{\pm}1.8\%$ & $87.7\%{\pm}0.4\%$ \\
Matryoshka & $99.6\%$ & $37.5\%{\pm}3.5\%$ & $35.0\%{\pm}5.0\%$ & $46.4\%{\pm}3.6\%$ & $25.2\%{\pm}3.1\%$ & $62.8\%{\pm}1.7\%$ & $\mathbf{88.4\%{\pm}0.2\%}$ \\
\midrule
SASA & $98.7\%$ & $\mathbf{11.3\%{\pm}0.9\%}$ & $\mathbf{8.7\%{\pm}0.6\%}$ & $\mathbf{23.9\%{\pm}4.4\%}$ & $\mathbf{10.9\%{\pm}4.0\%}$ & $\mathbf{86.2\%{\pm}0.5\%}$ & $83.2\%{\pm}0.5\%$ \\
\bottomrule
\end{tabular}}\vspace{-5mm}
\end{table}

\begin{wrapfigure}{r}{0.375\textwidth}
\centering\vspace{-7mm}
\includegraphics[width=\linewidth]{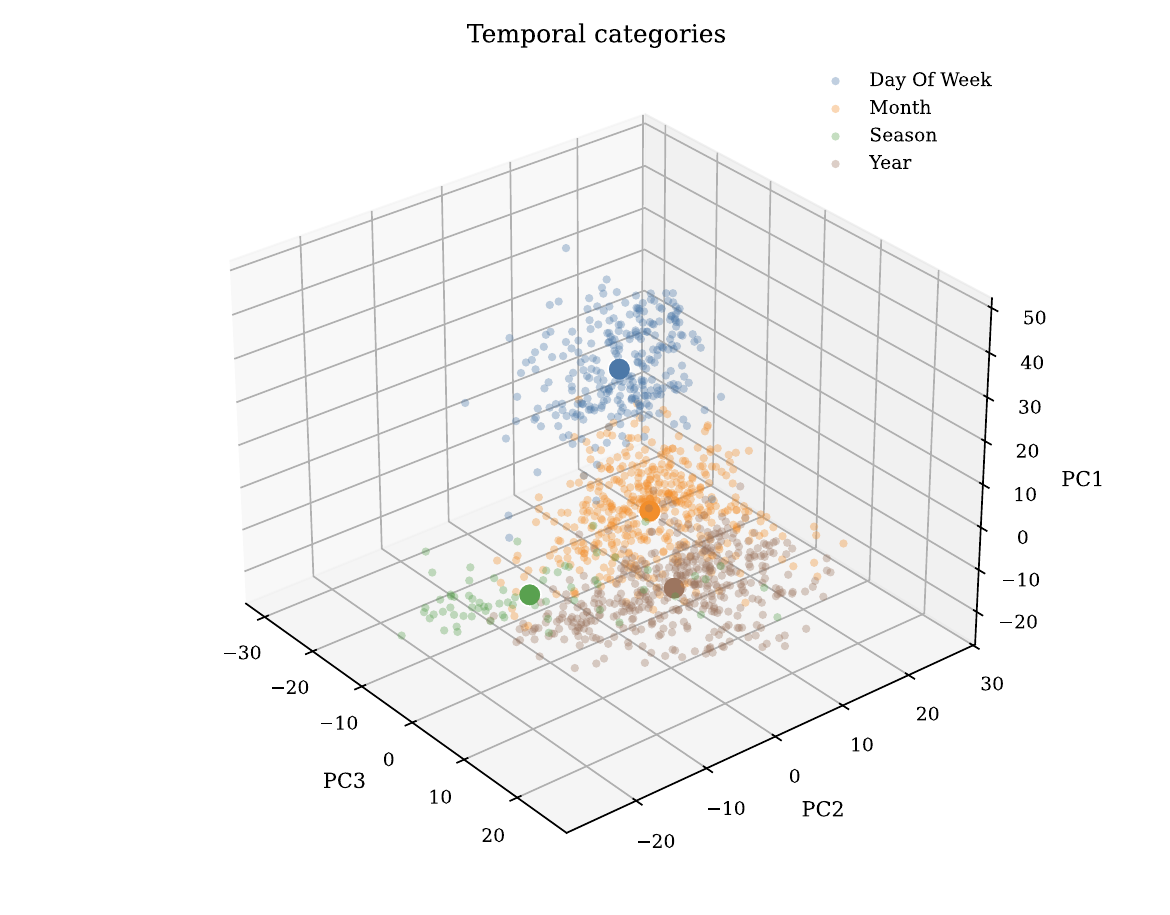}\vspace{-3mm}
\captionsetup{font=small,skip=2pt,justification=raggedright,singlelinecheck=false}
\caption{PCA of temporal categories in GPT-2 SASA Group 1473.}\vspace{-3mm}
\label{fig:temporal_scatter}
\end{wrapfigure}

\noindent\textbf{Interpretability.}~
We use AutoInterp~\citep{karvonen2025saebench} to assess how well a natural-language explanation predicts unit activity. An LLM generates an explanation from activating examples and uses it to predict which held-out contexts activate the unit. SASA scores 83.2\%, compared with 86.3--88.4\% for the vector baselines (Table~\ref{tab:matched-budget}). Each SASA group is evaluated as one unit through the norm of its decoded contribution. Combining several related feature values in one group may make its activation pattern harder to capture with a single description, so the absorption and member prediction gains need not translate into a higher AutoInterp score.

To examine how related feature values are organized within a group, we visualize its code coordinates. In GPT-2 SASA, Group 1473 responds to occurrences of days, months, years, and seasons in natural text. We collect its six-dimensional codes at these tokens in OpenWebText and project them with PCA. The temporal categories occupy distinct regions within this single group (Figure~\ref{fig:temporal_scatter}). Activation profiles and text examples are provided in Appendix~\ref{apx:temporal}. In Mistral SASA, a single group also preserves the circular arrangement of months (Figure~\ref{fig:real-manifolds}). Additional geographical and sports examples appear in Appendices~\ref{apx:geo} and~\ref{apx:additional-interp}.

\noindent\textbf{Causal interventions.}~We test whether steering Mistral-7B using SASA groups changes the probability of generating a month or weekday name. For each concept, we select the group $k$ whose pre-activation norm has the highest AUROC for separating its names from background tokens. We then modify the prompt's hidden state within group $k$'s decoder subspace, either by adding a steering vector $\boldsymbol{v}_k$ scaled by $\alpha$ to promote the concept or by removing the state's component in that subspace. We then complete the forward pass with this modified hidden state and measure $\mathbb{P}(\mathcal{C})$, the total next-token probability assigned to the relevant month or weekday names at the completion position (see Appendix~\ref{apx:causal}). Removing the component in the month or weekday decoder subspace from the hidden state lowers $\mathbb{P}(\mathcal{C})$ from 86.1\% to 40.9\% for months and from 90.0\% to 45.2\% for weekdays. Adding $\boldsymbol{v}_k$ raises these probabilities from 5.6\% and 5.0\% to 34.6\% and 52.0\% at $\alpha=1$, and to 78.7\% and 80.1\% at $\alpha=2$. By comparison, steering with vectors from the other group or with random vectors of the same norm gives at most 8.3\% probability to month and 12.0\% to weekday  (Table~\ref{tab:causal}).






\begin{table}[t]
\centering
\caption{Mean next-token probability  of month or weekday names after steering Mistral-7B at layer~8. }
\label{tab:causal}
\small
\setlength{\tabcolsep}{4pt}
\begin{tabular}{@{}llcccccc@{}}
\toprule
\multirow{2}{*}{Concept} & \multirow{2}{*}{Intervention} & \multirow{2}{*}{$\alpha$} & \multirow{2}{*}{Unsteered} & \multirow{2}{*}{\shortstack{Selected\\group}} & \multicolumn{3}{c}{Controls} \\
\cmidrule(lr){6-8}
 & & & & & Same concept & Different concept & Random \\
\midrule
\multirow{3}{*}{Months} & Removal & -- & $86.1\%$ & $\mathbf{40.9}\%$ & $85.9\%$ & $84.7\%$ & $86.1\%$ \\
 & \multirow{2}{*}{Addition} & $1.0$ & $5.6\%$ & $\mathbf{34.6}\%$ & $6.1\%$ & $8.3\%$ & $5.6\%$ \\
 & & $2.0$ & $5.6\%$ & $\mathbf{78.7}\%$ & $7.2\%$ & $7.2\%$ & $5.9\%$ \\
\midrule
\multirow{3}{*}{Weekdays} & Removal & -- & $90.0\%$ & $\mathbf{45.2}\%$ & $88.1\%$ & $88.9\%$ & $90.0\%$ \\
 & \multirow{2}{*}{Addition} & $1.0$ & $5.0\%$ & $\mathbf{52.0}\%$ & $5.8\%$ & $12.0\%$ & $5.4\%$ \\
 & & $2.0$ & $5.0\%$ & $\mathbf{80.1}\%$ & $11.3\%$ & $8.7\%$ & $5.9\%$ \\
\bottomrule
\end{tabular}\vspace{-5mm}
\end{table}


\noindent\textbf{Ablations.}~We vary feature dimension and sparsity on synthetic activations with known subspaces. We project each atom's decoded contribution onto a feature's subspace and count the TopK atoms accounting for 90\% of the sum of squared norms. For $d_i=16$, reducing $\ell_0$ from 64 to 32 increases the count from 31.8 to 646.8 atoms, while the count for $d_i=4$ remains near 8. Thus projected contributions spread over more separately gated decoder atoms as the scalar budget tightens, starting with higher-dimensional features and accompanied by poorer reconstruction (Appendix~\ref{apx:synthetic-dimension}). With known feature subspaces, the nuclear-norm penalty on $\boldsymbol{D}_k\boldsymbol{E}_k$ lowers the effective rank of groups matched to one-dimensional features from 7.1 to 4.2, while groups matched to eight-dimensional features retain an effective rank of 7.9 (Appendix~\ref{apx:rank-adaptivity}). In the GPT-2 setup, reducing $r$ from 6 to 2 raises mean first-letter absorption from 11.3\% to 13.9\% and mean family absorption from 23.9\% to 43.4\% (Appendix~\ref{apx:rank-grid}).
 We also perform the GPT-2 comparison across dictionary widths and $\ell_0$ values. SASA has lower mean family absorption than all baselines at $24{,}576$ and $49{,}152$ columns across the tested sparsities (Appendix~\ref{apx:width-grid}).

\section{Conclusion \& Discussion}
\label{sec:discussion}
We introduced Subspace-Aware Sparse Autoencoders (SASA), which replace single-decoder directions with learned decoder subspaces and enforce structured sparsity at the group level for mechanistic interpretability of LLMs. We provided theoretical motivation for why vector dictionaries can split multi-dimensional concepts, characterized when splitting is geometrically necessary, showed that additional non-basis directions can lower the $\ell_1$-regularized population risk, and explained why subspace groups can represent these concepts more faithfully. We also analyzed the sample complexity of estimating a feature subspace and compared it with the samples needed to observe every atom in a splitting SAE. We note that our covering bound requires a per-feature co-activation budget below the feature dimension, and the $\ell_1$ risk analysis with optimal codes does not establish that TopK training must converge to a fragmented dictionary. Our SASA recovery theorem concerns population optima under uniform spherical sampling of a single feature, and the sample bound concerns estimating its subspace. We plan to extend the recovery analysis to learned routing of noisy superpositions and practical training objectives. Furthermore, learning a feature \emph{subspace} does not automatically yield an interpretable coordinate system inside that subspace, and our current interpretation often relies on manual inspection. We therefore aim to recover meaningful internal coordinates within each learned subspace and evaluate their use for inference-time steering and intervention across contexts.



\clearpage
\begin{ack}
This material is based upon work supported in part by the National Science Foundation (NSF) under Grant Number 2318101. Any opinions, findings, conclusions, or recommendations expressed in this material are those of the author(s) and do not necessarily reflect the views of the NSF.
The authors declare no competing interests.
\end{ack}

\bibliographystyle{plain}
\bibliography{ref}

\clearpage
\appendix

\section*{Appendix}

The appendix is organized as follows:
\begin{itemize}
    \item Appendix~\ref{sec:related_work} reviews work on mechanistic interpretability, representation geometry, superposition, and sparse autoencoders.
    \item Appendix~\ref{apx:alg} describes SASA training, inference, and the auxiliary loss.
    \item Appendix~\ref{apx:feature-geometry} presents synthetic illustrations of feature splitting, geometry in raw LLM activations, and clusters of SAE decoder atoms.
    \item Appendix~\ref{apx:exp} reports evaluation protocols and additional comparisons, including absorption, member prediction, feature dimension and group rank, causal interventions, empirical covering, and concept analyses.
    \item Appendix~\ref{apx:theory} provides proofs for the covering bounds, $\ell_1$ risk analysis, SASA recovery guarantees, and sample complexity analysis.
\end{itemize}

\section{Additional Related Work}
\label{sec:related_work}

\noindent\textbf{Mechanistic Interpretability.}~Mechanistic interpretability seeks mechanistic accounts of neural network behavior, typically validated via targeted interventions. Wang et al.~\cite{Wang_Variengien_Conmy_Shlegeris_Steinhardt_2022} demonstrated that complex behaviors can be decomposed into compact circuits that remain functional under controlled ablations and activation patching. Subsequent work has begun to automate this workflow, for example, by automatically identifying sparse, behavior-preserving subgraphs in the computational graph \citep{Conmy_Mavor-Parker_Lynch_Heimersheim_Garriga-Alonso_2023}. Despite the progress made, Mechanistic Interpretability faces key challenges and questions, namely how to decompose networks into meaningful, causally-relevant components and how to rigorously validate mechanistic hypotheses to distinguish faithful explanations from misleading interpretability illusions \citep{Sharkey_Chughtai_Batson_Lindsey_Wu_Bushnaq_Goldowsky-Dill_Heimersheim_Ortega_Bloom_etal._2025}.

\noindent\textbf{LLM Representations.}~A central line of work in Deep Learning studies \emph{representation hypotheses} that posit geometric structures for semantic variables in activation space. The LRH models concepts as directions in representation space \citep{Park_Choe_Veitch_2024}. Linear probing provides an operational test of where information is linearly encoded \citep{Alain_Bengio_2018}. More recently, this perspective has been broadened into multi-dimensional hypotheses, where concepts may correspond to low-dimensional subspaces or more structured geometric objects (e.g., simplices) under an appropriate notion of contrast and inner product \citep{Engels_Michaud_Liao_Gurnee_Tegmark_2025,Park_Choe_Jiang_Veitch_2025}. The manifold view of Modell et al.~\cite{Modell_Rubin-Delanchy_Whiteley_2025} allows curved representations whose linear-span dimension exceeds their manifold dimension and whose geometric distortion can be hidden by low-dimensional projections. These hypotheses matter for mechanistic interpretability because they define concrete geometric targets for measurement and causal testing: once a concept's structure is specified, probes, interventions, and steering can be applied with respect to that structure rather than individual neurons. To date, much of the mechanistic interpretability literature has implicitly aligned with LRH, but the emerging multi-dimensional evidence motivates revisiting our tools and analyses in the setting where features occupy subspaces instead of single directions.

Modell et al.~\cite{Modell_Rubin-Delanchy_Whiteley_2025} conjecture that SAE sparsity penalties encourage nearly collinear vectors to cover representation manifolds, contributing to feature splitting. Bhalla et al.~\cite{bhalla2026sparse} establish conditions for compact manifold capture by an idealized decoder, yet observe fragmentation across redundant atoms, which they termed \emph{dilution}, in trained LLM SAEs. Our analysis quantifies the covering cost under a fixed per-feature co-activation budget and shows that adding non-basis directions can lower the $\ell_1$ population risk with optimized codes even when a feature basis is available (Section~\ref{sec:instability}).

\noindent\textbf{Superposition.}~The superposition hypothesis explains why neuron-level interpretations often fail: models can represent many sparse features in a limited activation space, producing polysemantic units and interference \citep{Elhage_Hume_Olsson_Schiefer_Henighan_Kravec_Hatfield-Dodds_Lasenby_Drain_Chen_etal._2022}. This motivates dictionary learning methods that aim to recover a more feature-aligned basis than neurons. Sparse autoencoders (SAEs) are a widely used unsupervised approach in this line, learning sparse latents whose decoded directions reconstruct activations and often admit semantic interpretations \citep{Bricken_Templeton_Batson_Chen_Jermyn_Conerly_Turner_Anil_Denison_Askell_etal._2023}. Yet their training is challenging, as they usually require many hidden states from the LLM to train, motivating the search for more efficient, but less interpretable, methods \citep{Leask_Nanda_Moubayed_2025}.

\noindent\textbf{Standard SAEs.}~Most SAE variants implement a vector dictionary: activations are reconstructed as sparse combinations of decoder directions, with sparsity enforced by mechanisms such as ReLU \citep{Bricken_Templeton_Batson_Chen_Jermyn_Conerly_Turner_Anil_Denison_Askell_etal._2023}, JumpReLU \citep{Rajamanoharan_Lieberum_Sonnerat_Conmy_Varma_Kramár_Nanda_2024}, gating \citep{Rajamanoharan_Conmy_Smith_Lieberum_Varma_Kramár_Shah_Nanda_2024}, Top-$k$ \citep{Gao_Tour_Tillman_Goh_Troll_Radford_Sutskever_Leike_Wu_2024}, BatchTopK \citep{Bussmann_Leask_Nanda_2024}, and Matryoshka \citep{Bussmann_Nabeshima_Karvonen_Nanda_2025}. Recent scaling analysis shows that when features have manifold structure, SAEs may allocate many vectors to cover a single manifold, resulting in far fewer distinct learned features than the model width would suggest \citep{Michaud_Gorton_McGrath_2025}. These limitations motivate moving beyond vector dictionaries toward subspace-level features, using structured sparsity and low-rank control to represent a single factor as a compact learned subspace rather than a cluster of directions \citep{Yuan_Lin_2006, Candès_Recht_2009, Theodosis_Tolooshams_Tankala_Tasissa_Ba_2022}.

\section{SASA Implementation and Training}
\label{apx:alg}
In this section, we explain how SASA is trained on LLM activations and used for reconstruction. We specify the data and optimization settings, describe how the auxiliary loss trains inactive coordinates, and present the detailed training and inference algorithms.

\subsection{Training Setup}
\label{apx:training}
We use residual-stream activations at \texttt{blocks.$\ell$.hook\_resid\_pre}, with $\ell=7$ for GPT-2 and $\ell=8$ for Mistral-7B. The matched GPT-2 runs use OpenWebText~\citep{Gokaslan_Cohen_Pavlick_Tellex_2019}, context length $128$, and $300\mathrm{M}$ training tokens. Mistral SASA uses Pile~\citep{Gao_Biderman_Black_Golding_Hoppe_Foster_Phang_He_Thite_Nabeshima_etal._2020}, context length $512$, and $500\mathrm{M}$ training tokens.

The main GPT-2 and Mistral SASA configurations are $(K,r,s)=(4096,6,20)$ and $(4096,8,10)$, respectively. We use $\lambda_{\mathrm{aux}}=1$, with $s_{\mathrm{aux}}=384$ for GPT-2 and $256$ for Mistral. Inputs are layer-normalized and decoder columns have unit norm. Mistral uses batch-level group selection during training and a group-norm threshold calibrated on Pile at inference.

We use Adam with $(\beta_1,\beta_2)=(0.9,0.999)$ and no weight decay. The learning rate is $4\times10^{-4}$ for GPT-2 and $2\times10^{-4}$ for Mistral, with a token batch size of $4{,}096$. We use a $1{,}000$-step linear warm-up and linearly decay the learning rate over the final fifth of training. The matched GPT-2 vector baselines omit the warm-up and use unnormalized inputs. As stated before, experiments were conducted on an A100 80GB GPU.

\subsection{Auxiliary Loss}
\label{apx:alg-aux}
Top-$s$ gating can leave some coordinates rarely active \citep{Gao_Tour_Tillman_Goh_Troll_Radford_Sutskever_Leike_Wu_2024}. We apply the auxiliary loss to coordinates that have been inactive for every token during the previous 100 training steps in the matched GPT-2 runs, or $1{,}000$ steps in Mistral. Let $p_{k,j}(\boldsymbol{h})$ denote the $j$th coordinate of $\boldsymbol{p}_k(\boldsymbol{h})=\boldsymbol{E}_k\boldsymbol{h}$, and let
\[
\mathcal{C}_{\mathrm{dead}}\subset[K]\times[r]
\]
denote the set of dead coordinates. For each $\boldsymbol{h}$, we select the dead coordinates with the largest pre-activations,
\[
\widetilde{\mathcal{T}}_{s_{\mathrm{aux}}}(\boldsymbol{h})
\in
\arg\max_{T\subset\mathcal{C}_{\mathrm{dead}},\,|T|=\min\{s_{\mathrm{aux}},|\mathcal{C}_{\mathrm{dead}}|\}}
\sum_{(k,j)\in T}p_{k,j}(\boldsymbol{h}),
\]
and define auxiliary activations
\[
\widetilde{a}_{k,j}(\boldsymbol{h})=
\begin{cases}
p_{k,j}(\boldsymbol{h}), & (k,j)\in\widetilde{\mathcal{T}}_{s_{\mathrm{aux}}}(\boldsymbol{h}),\\
0, & \text{o.w.}
\end{cases}
\]
Define the residual
\[
 \boldsymbol{r}(\boldsymbol{h})=\boldsymbol{h}-\boldsymbol{D}\boldsymbol{a}(\boldsymbol{h}),
\]
and treat $ \boldsymbol{r}(\boldsymbol{h})$ as a frozen target (i.e., gradients do not flow through $\boldsymbol{r}$).
\begin{equation}
\label{eq:sasa-aux}
\mathcal{L}_{\mathrm{aux}}(\boldsymbol{h})
=
\| \boldsymbol{r}(\boldsymbol{h})-\boldsymbol{D}\widetilde{\boldsymbol{a}}(\boldsymbol{h})\|_2^2.
\end{equation}

\subsection{Training and Inference Algorithms}
\label{apx:alg-pseuodo}
Algorithm~\ref{alg:sasa-train} gives minibatch training for the nuclear-norm objective in Eq.~\ref{eq:sasa-objective}, together with the auxiliary loss for reconstructing the residual. Algorithm~\ref{alg:sasa-infer} applies the trained encoder, gate, and decoder to reconstruct a new activation.
\begin{algorithm}[!h]
  \caption{SASA Training with Composition Regularization}
  \label{alg:sasa-train}
  \begin{algorithmic}
    \STATE {\bfseries Input:} minibatch $\{\boldsymbol{h}^{(j)}\}_{j=1}^{B}$, parameters $\{\boldsymbol{E}_k,\boldsymbol{D}_k\}_{k=1}^{K}$
    \STATE {\bfseries Hyperparameters:} $s$, $s_{\mathrm{aux}}$, $\lambda_{\mathrm{aux}}$, $\lambda_{\mathrm{dim}}$
    \STATE Rescale every column of $\boldsymbol{D}$ to unit norm.
    \STATE $J \gets 0$.
    \FOR{$j=1$ {\bfseries to} $B$}
      \STATE $\boldsymbol{h}\gets \boldsymbol{h}^{(j)}$, $\boldsymbol{p}(\boldsymbol{h}) \gets \boldsymbol{E}\boldsymbol{h}$ with blocks $\boldsymbol{p}_k(\boldsymbol{h})=\boldsymbol{E}_k\boldsymbol{h}$
      \STATE $\mathcal{T}_s(\boldsymbol{h})\in\arg\max_{T\subset[K],|T|=s}\ \sum_{k\in T}\|\boldsymbol{p}_k(\boldsymbol{h})\|_2$
      \STATE $\boldsymbol{a}_k(\boldsymbol{h}) \gets \boldsymbol{p}_k(\boldsymbol{h})$ if $k\in\mathcal{T}_s(\boldsymbol{h})$, else $\boldsymbol{a}_k(\boldsymbol{h})\gets \boldsymbol{0}$
      \STATE $\mathcal{L}_{\mathrm{aux}}(\boldsymbol{h}) \gets 0$
      \STATE $\mathcal{C}_{\mathrm{dead}} \gets$ set of dead coordinates
      \IF{$\mathcal{C}_{\mathrm{dead}}\neq\emptyset$}
        \STATE $\boldsymbol{r}(\boldsymbol{h}) \gets \boldsymbol{h}-\boldsymbol{D}\boldsymbol{a}(\boldsymbol{h})$ (frozen target)
        \STATE $\widetilde{\mathcal{T}}_{s_{\mathrm{aux}}}(\boldsymbol{h})\in\arg\max_{T\subset\mathcal{C}_{\mathrm{dead}},|T|=\min\{s_{\mathrm{aux}},|\mathcal{C}_{\mathrm{dead}}|\}}\ \sum_{(k,j)\in T}p_{k,j}(\boldsymbol{h})$
        \STATE $\widetilde{a}_{k,j}(\boldsymbol{h}) \gets p_{k,j}(\boldsymbol{h})$ if $(k,j)\in\widetilde{\mathcal{T}}_{s_{\mathrm{aux}}}(\boldsymbol{h})$, else $\widetilde{a}_{k,j}(\boldsymbol{h})\gets 0$
        \STATE $\mathcal{L}_{\mathrm{aux}}(\boldsymbol{h}) \gets \|\boldsymbol{r}(\boldsymbol{h})-\boldsymbol{D}\widetilde{\boldsymbol{a}}(\boldsymbol{h})\|_2^2$
      \ENDIF
      \STATE $J \gets J + \|\boldsymbol{h}-\boldsymbol{D}\boldsymbol{a}(\boldsymbol{h})\|_2^2 + \lambda_{\mathrm{aux}}\mathcal{L}_{\mathrm{aux}}(\boldsymbol{h})$
    \ENDFOR
    \STATE $J \gets \frac{1}{B}J + \lambda_{\mathrm{dim}}\sum_{k=1}^{K}\|\boldsymbol{D}_k\boldsymbol{E}_k\|_*$
    \STATE Update $(\boldsymbol{E},\boldsymbol{D})$ to minimize $J$.
  \end{algorithmic}
\end{algorithm}

\begin{algorithm}[t!]
  \caption{SASA Reconstruction Inference}
  \label{alg:sasa-infer}
  \begin{algorithmic}
    \STATE {\bfseries Input:} LLM activation $\boldsymbol{h}$, trained parameters $\{\boldsymbol{E}_k,\boldsymbol{D}_k\}_{k=1}^{K}$, sparsity level $s$
    \FOR{$k=1$ {\bfseries to} $K$}
      \STATE $\boldsymbol{p}_k(\boldsymbol{h}) \gets \boldsymbol{E}_k\boldsymbol{h}$
    \ENDFOR
    \STATE $\mathcal{T}_s(\boldsymbol{h})\in\arg\max_{T\subset[K],|T|=s}\ \sum_{k\in T}\|\boldsymbol{p}_k(\boldsymbol{h})\|_2$
    \FOR{$k=1$ {\bfseries to} $K$}
      \STATE $\boldsymbol{a}_k(\boldsymbol{h}) \gets \boldsymbol{p}_k(\boldsymbol{h})$ if $k\in\mathcal{T}_s(\boldsymbol{h})$, else $\boldsymbol{a}_k(\boldsymbol{h})\gets \boldsymbol{0}$
    \ENDFOR
    \STATE {\bfseries Output:} Reconstruction $\boldsymbol{D}\boldsymbol{a}(\boldsymbol{h}) = \sum_{k=1}^{K}\boldsymbol{D}_k\boldsymbol{a}_k(\boldsymbol{h})$, sparse codes $\{\boldsymbol{a}_k(\boldsymbol{h})\}_{k=1}^{K}$, active set $\mathcal{T}_s(\boldsymbol{h})$
    \STATE {\bfseries Feature intensity:} For feature $k$, intensity is $\|\boldsymbol{a}_k(\boldsymbol{h})\|_2$.
  \end{algorithmic}
\end{algorithm}

\section{Feature Geometry and Splitting}
\label{apx:feature-geometry}
In this section, we investigate whether feature representations exhibit low-dimensional structure and how vector SAEs encode that structure. To do so, we illustrate splitting on synthetic manifolds, analyze the dimensionality of raw LLM activations, and measure redundancy among learned SAE decoder directions.

\subsection{Synthetic Illustrations of Feature Splitting}
\label{apx:synthetic-illustrations}
In Figure~\ref{fig:real-manifolds}, we illustrated feature splitting in Mistral-7B. Here we examine this behavior using synthetic manifolds whose geometry and subspace dimension are known, comparing how vector SAEs and SASA represent each feature (Figures~\ref{fig:toy-manifold-sae-sasa} and~\ref{fig:tuning-covering}). Each dictionary has $256$ decoder columns. TopK and BatchTopK keep four latents active per input, as many decoder columns as the single width-4 group that SASA activates.

\begin{figure}[t!]
    \centering
    \includegraphics[width=\linewidth]{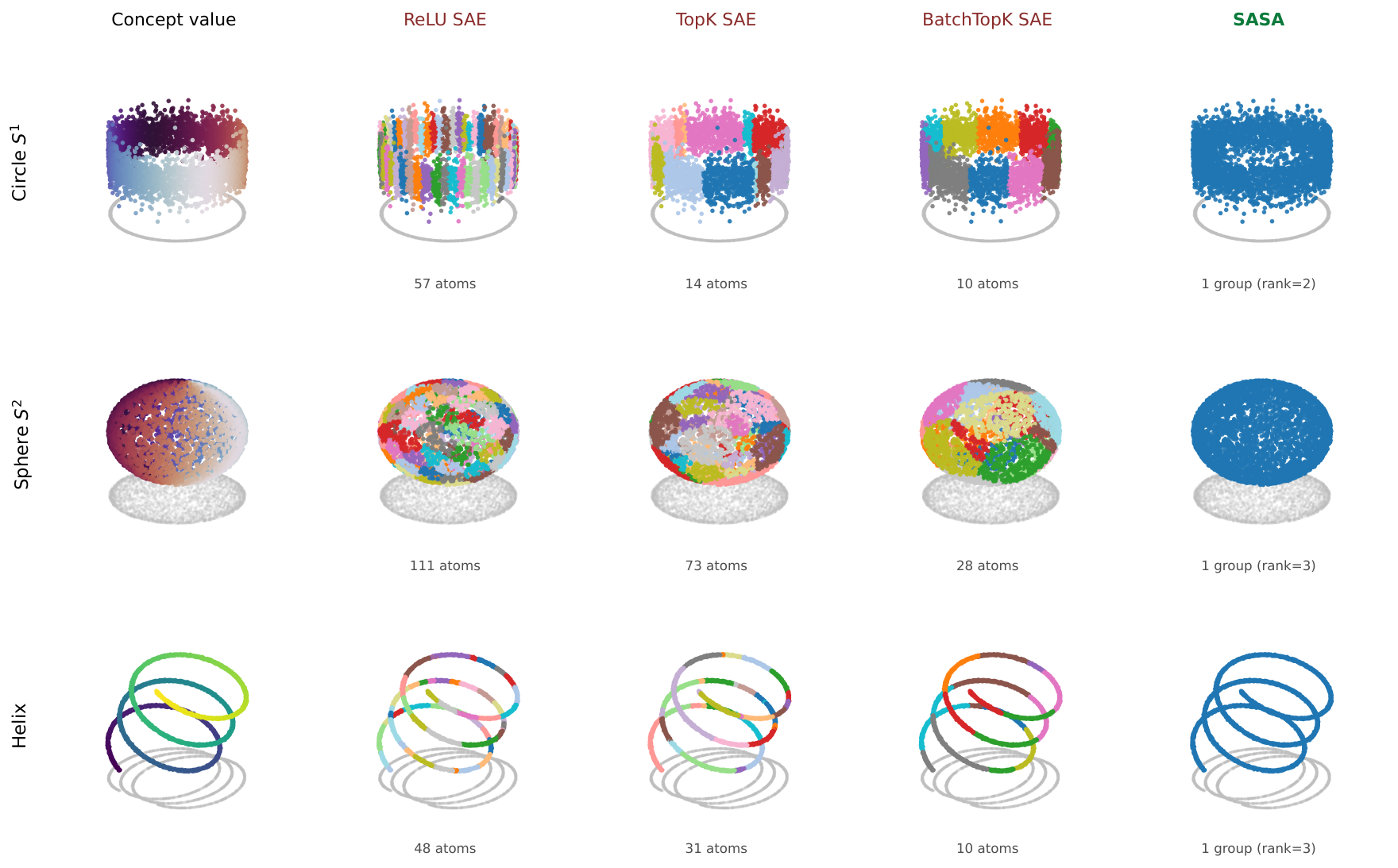}
    \caption{\textbf{Feature splitting on synthetic manifolds.} A circle ($d_i=2$), sphere ($d_i=3$), and helix ($d_i=3$) are embedded in $\mathbb{R}^{64}$ with $5\%$ noise. Each dictionary has width $256$. The first column shows concept values. The next three show ReLU, TopK, and BatchTopK SAEs, with points colored by the most aligned decoder atom. The last column shows SASA. Standard SAEs allocate atoms to local regions, while one SASA group of effective rank $d_i$ responds across each feature.}
    \label{fig:toy-manifold-sae-sasa}
\end{figure}

\begin{figure*}[!h]
  \centering
  \begin{subfigure}[t]{\linewidth}
    \includegraphics[width=\linewidth]{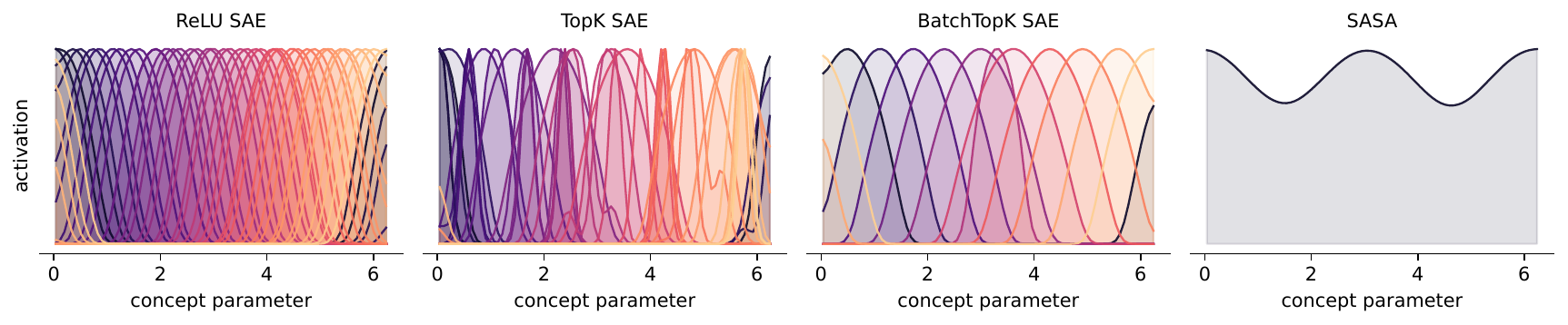}
    \caption{Latent activation profiles for circle.}\label{fig:tune-circle}
  \end{subfigure}\\[2pt]
  \begin{subfigure}[t]{\linewidth}
    \includegraphics[width=\linewidth]{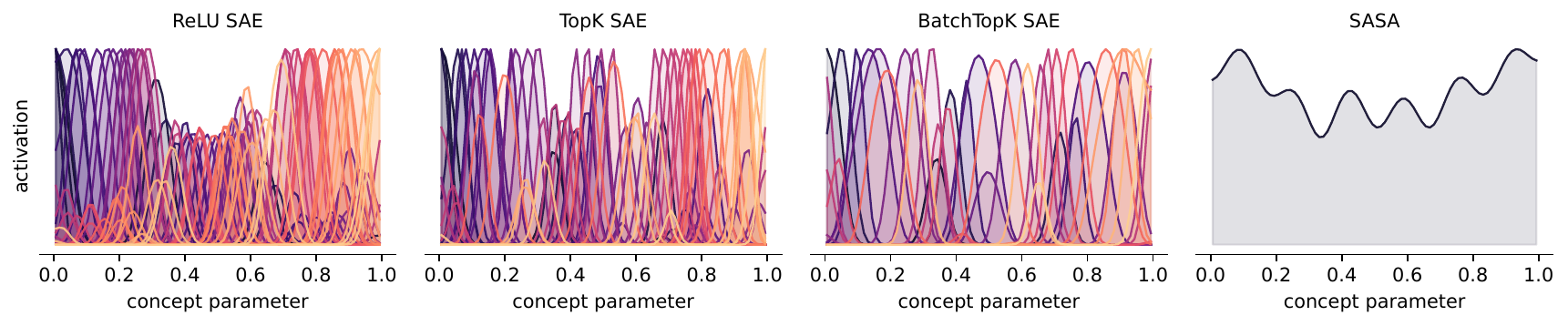}
    \caption{Latent activation profiles for helix.}\label{fig:tune-helix}
  \end{subfigure}\\[4pt]
  \begin{subfigure}[t]{0.5\linewidth}
    \includegraphics[width=\linewidth]{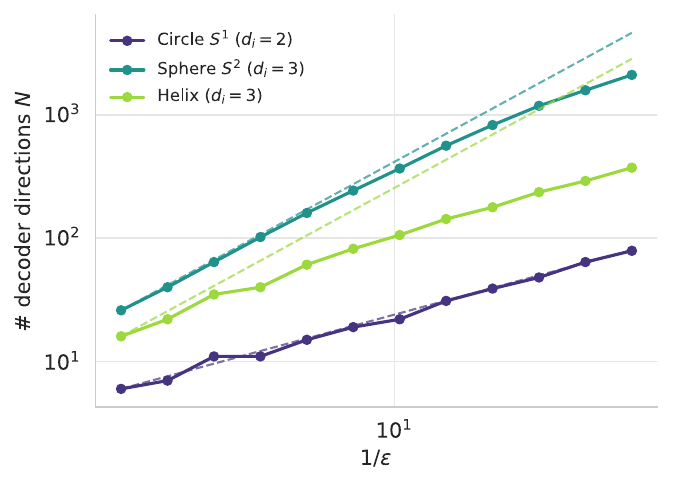}
    \caption{Line-covering number.}\label{fig:covering}
  \end{subfigure}
\caption{\textbf{Tuning profiles and empirical covering counts.} In \textbf{(a,b)}, ReLU, TopK, and BatchTopK SAEs have localized activation profiles on the circle and helix, while a SASA group responds across the concept range. Panel \textbf{(c)} shows the number of decoder directions $N$ used to reach covering error $\varepsilon$. Solid curves give empirical counts, and dashed curves give reference power-law slopes of Theorem~\ref{thm:line-covering}.}
  \label{fig:tuning-covering}
\end{figure*}

\subsection{Geometry in Raw LLM Activations}
\label{apx:low-dim-validation}

To assess the dimensionality of raw LLM activations, we apply PCA directly to GPT-2 layer-7 activations on controlled prompts containing temporal or geographical terms. The number of principal components explaining 90.0\% of the variance is 14 for the temporal prompts and 33 for the geographical prompts, out of $768$ activation dimensions (Figure~\ref{fig:apx-intrinsic-dim}). Figure~\ref{fig:real-line-cover} uses a separate sample of SAE-reconstructed temporal components from OpenWebText, whose $\mathrm{dim90}$ is $10$.

\begin{figure}[htbp]
    \centering
    \begin{subfigure}[t]{0.48\linewidth}
        \centering
        \includegraphics[width=\linewidth]{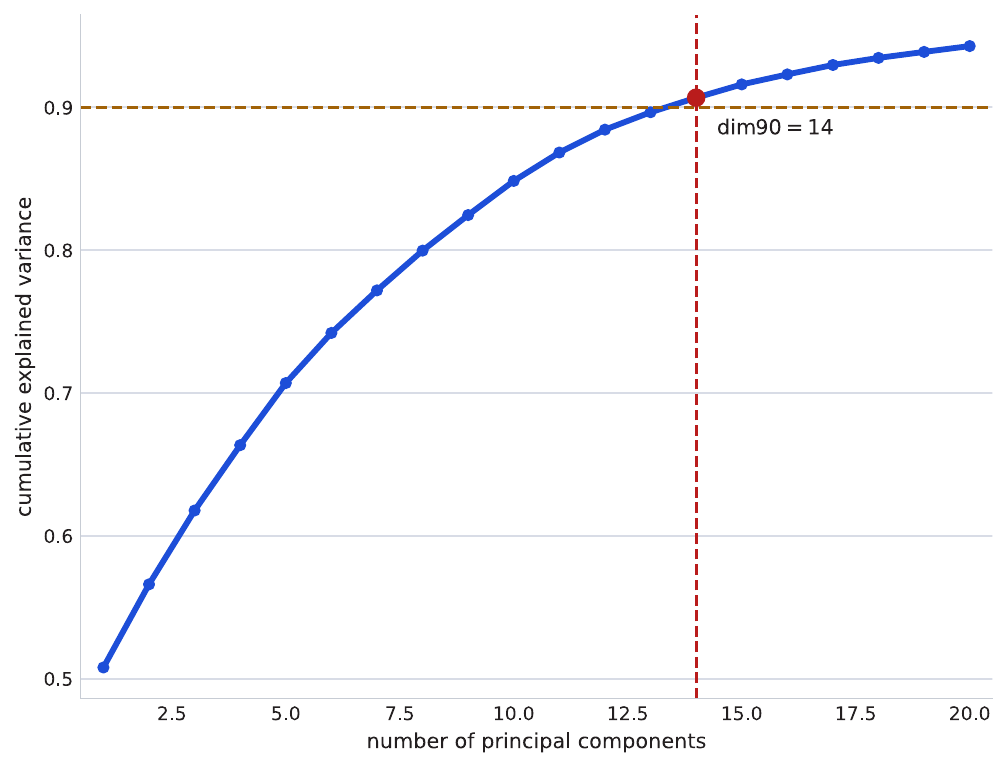}
        \caption{Temporal slice with $\mathrm{dim90}=14$.}
        \label{fig:apx-intrinsic-temporal}
    \end{subfigure}\hfill
    \begin{subfigure}[t]{0.48\linewidth}
        \centering
        \includegraphics[width=\linewidth]{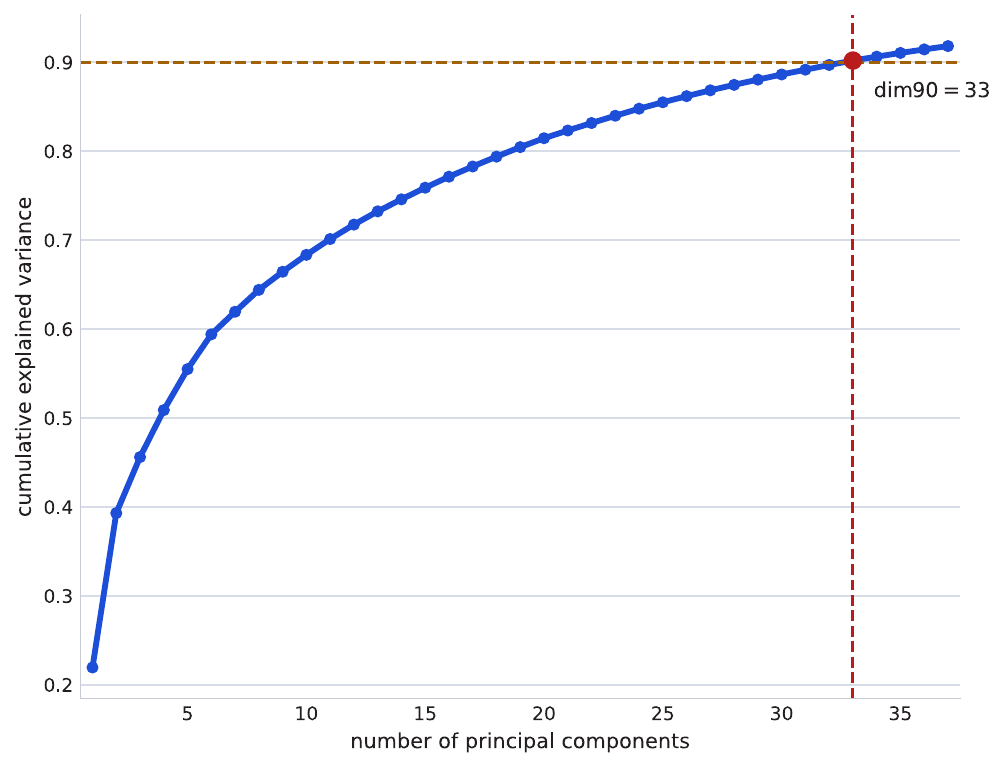}
        \caption{Geography slice with $\mathrm{dim90}=33$.}
        \label{fig:apx-intrinsic-geography}
    \end{subfigure}
    \caption{\textbf{Variance concentration in raw GPT-2 activations.} Cumulative variance explained by PCA on temporal and geographical prompts. Here, $\mathrm{dim90}$ is the number of components needed to explain 90.0\% of the variance.}
    \label{fig:apx-intrinsic-dim}
\end{figure}

To test whether the low-dimensional structure extends beyond individual concepts, we fit a mixture-of-subspaces model on raw OpenWebText activations. We sample $120\mathrm{k}$/$30\mathrm{k}$ train/test activations from layer~7, run $K$-means, and fit local rank-$16$ PCA within each cluster. Each test activation is reconstructed by projection onto the span of up to $16$ selected local subspaces. At $K=256$ clusters, this model captures $84.0\%$ of held-out variance versus $74.5\%$ for global PCA at rank~$256$ (Figure~\ref{fig:apx-mixture-subspaces}), showing improved reconstruction with local subspaces in this comparison.

\begin{figure}[tbp]
    \centering
    \includegraphics[width=0.70\linewidth]{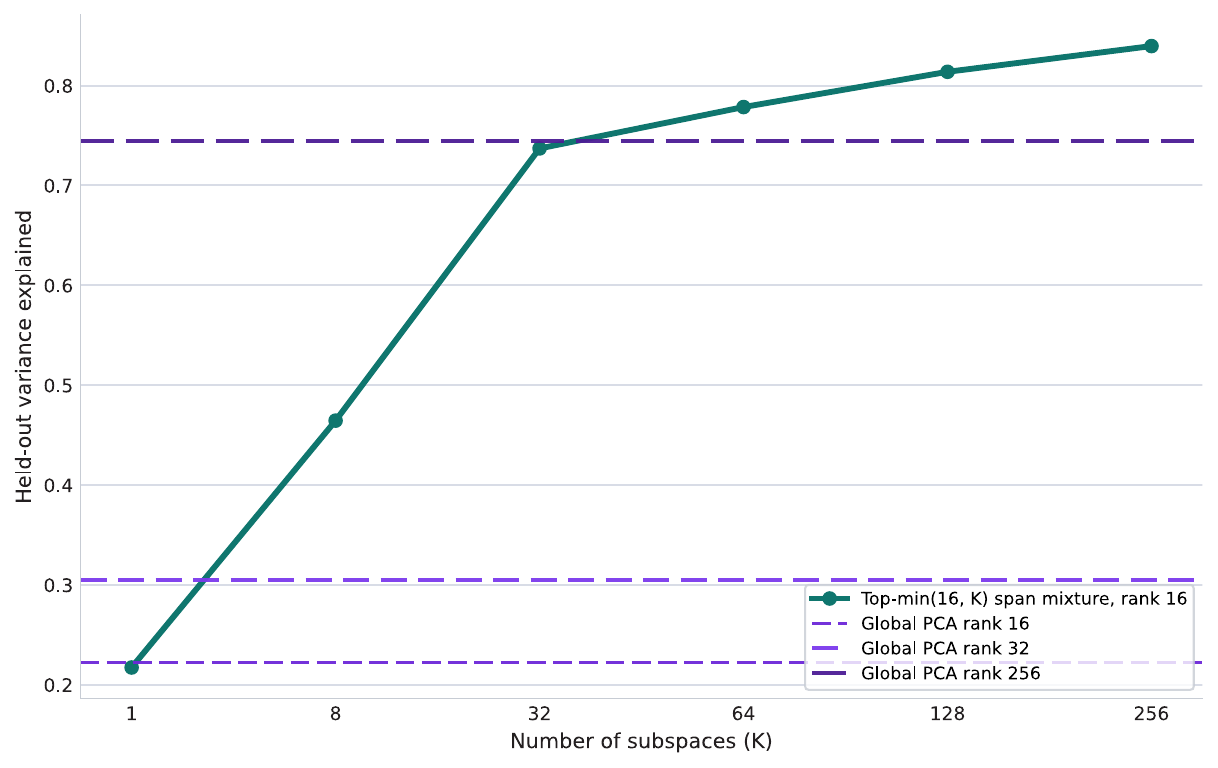}
    \caption{\textbf{Mixture-of-subspaces reconstruction in raw GPT-2 activations.} Rank-$16$ local PCA at $K=256$ clusters captures $84.0\%$ of held-out variance, exceeding global PCA at rank~$256$ ($74.5\%$).}
    \label{fig:apx-mixture-subspaces}
\end{figure}

\subsection{SAE Decoder Clusters}
\label{apx:clusters}
To assess redundancy among decoder directions, we cluster the decoder vectors of the published GPT-2~\citep{bloom2024saetrainingcodebase} and Mistral-7B~\citep{Engels_Michaud_Liao_Gurnee_Tegmark_2025} SAEs. For GPT-2, we perform spectral clustering. For Mistral-7B, we connect vectors whose cosine similarity is at least $0.5$ and take the connected components as clusters.

We then select clusters with size greater than $3$, compute the PCA dimensionality that captures at least $80\%$ of the cluster's variance, and compare it with the cluster size. Figure~\ref{fig:mistral_clusters} and Figure~\ref{fig:gpt_cluster} show the result. The median ratio of cluster size to PCA dimension is $1.67$ for Mistral-7B and $2.18$ for GPT-2. These clusters suggest that a shared low-dimensional structure is distributed across multiple separately gated latents, motivating SASA's use of a jointly activated group to represent and interpret related directions.

\begin{figure}[htbp]
    \centering
    \includegraphics[width=\linewidth]{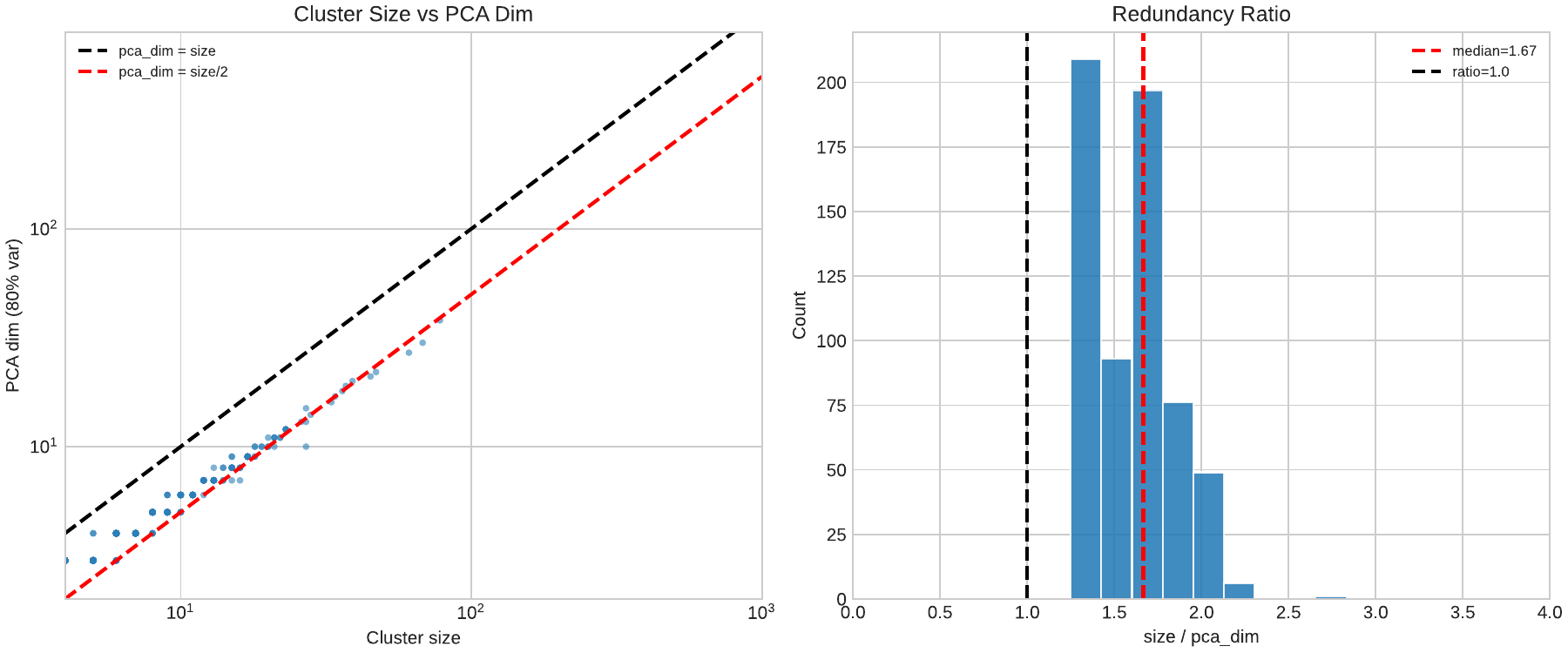}
    \caption{\textbf{Mistral-7B SAE decoder clusters.} Cluster size and the PCA dimension explaining $80\%$ of within-cluster variance (left), and their ratio (right). The median ratio is $1.67$.}
    \label{fig:mistral_clusters}
\end{figure}

\begin{figure}[htbp]
    \centering
    \includegraphics[width=\linewidth]{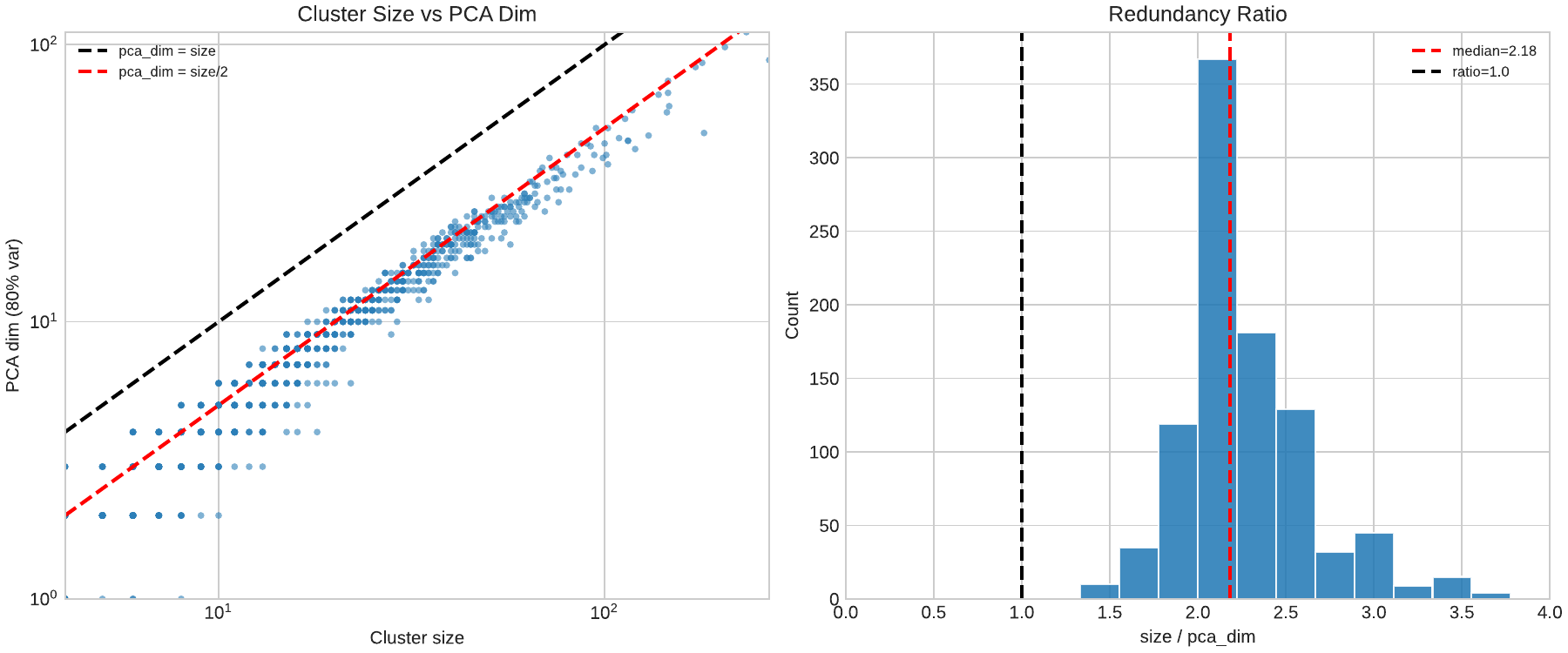}
    \caption{\textbf{GPT-2 SAE decoder clusters.} Cluster size and the PCA dimension explaining $80\%$ of within-cluster variance (left), and their ratio (right). The median ratio is $2.18$.}
    \label{fig:gpt_cluster}
\end{figure}

\section{Evaluation Details}

\label{apx:exp}
This section provides evaluation protocols and additional results for SASA. Appendix~\ref{apx:matched-protocol} describes the comparisons and metrics, followed by absorption in Appendix~\ref{apx:exp-absorption} and prediction of family members in Appendix~\ref{apx:value-readout}. Appendix~\ref{apx:feature-dimension-rank} examines the effects of feature dimension, group width, and nuclear-norm regularization. Finally, Appendix~\ref{apx:real-line-cover} compares empirical covering error with the theoretical scaling, and Appendix~\ref{apx:exp-interpretabiity} examines the activity and internal geometry of selected groups.

\subsection{Comparison setup}
\label{apx:matched-protocol}
The GPT-2 absorption comparisons match architectures at $24{,}576$ decoder columns, $300\mathrm{M}$ training tokens, and target sparsity $\ell_0\in\{60,120,240\}$. SASA uses $K=4096$ groups of width $r=6$, with $s=\ell_0/6$ and $\lambda_{\mathrm{dim}}=0.5$.

CE and KL scores measure how well reconstruction preserves model predictions, normalized to 0\% for zero ablation and 100\% for the original activations~\citep{Braun_Taylor_Goldowsky-Dill_Sharkey_2024,Rajamanoharan_Conmy_Smith_Lieberum_Varma_Kramár_Shah_Nanda_2024}. Explained variance measures reconstruction quality in activation space. Figure~\ref{fig:matched-tradeoff} compares reconstruction with mean and full absorption on both benchmarks across the sparsity sweep.

We use SAEBench AutoInterp~\citep{karvonen2025saebench} with its default protocol and \texttt{gpt-4o-mini} judge, sampling $1{,}000$ eligible units from $2\mathrm{M}$ Pile tokens. A unit is one scalar latent for a vector SAE or one SASA group, whose activity is the norm of its decoded contribution in residual-stream units. Prediction accuracy is averaged over units and reported as mean $\pm$ standard deviation across three evaluation seeds.

\begin{table}[tbp]
\centering
\caption{Configurations for the comparisons with published SAEs.}
\label{tab:reference-configs}
\small
\setlength{\tabcolsep}{5pt}
\begin{tabular}{@{}llrrr@{}}
\toprule
Model & Autoencoder & Dictionary width & Training tokens & $\ell_0$ \\
\midrule
GPT-2 & ReLU SAE & $24{,}576$ & $300\mathrm{M}$ & 60 \\
GPT-2 & SASA & $12{,}288$ & $150\mathrm{M}$ & 60 \\
Mistral & ReLU SAE & $65{,}536$ & $1\mathrm{B}$ & 285 \\
Mistral & SASA & $32{,}768$ & $500\mathrm{M}$ & 80 \\
\bottomrule
\end{tabular}
\end{table}

\begin{table}[tbp]
\centering
\caption{Reconstruction scores for the published ReLU/$\ell_1$ SAEs and the SASA configurations in Table~\ref{tab:reference-configs}.}
\label{tab:sae_sasa_summary}
\small
\setlength{\tabcolsep}{4pt}
\begin{tabular}{@{}lccc@{}}
\toprule
Model & KL score (\%) $\uparrow$ & CE score (\%) $\uparrow$ & Explained variance (\%) $\uparrow$ \\
\midrule
GPT-2 SAE & 98.0 & 98.0 & 99.1 \\
GPT-2 SASA & 97.0 & 97.1 & 98.8 \\
Mistral SAE & 93.2 & 94.1 & 96.7 \\
Mistral SASA & 95.3 & 95.1 & 98.6 \\
\bottomrule
\end{tabular}
\end{table}

\begin{figure}[tbp]
\centering
\includegraphics[width=\linewidth]{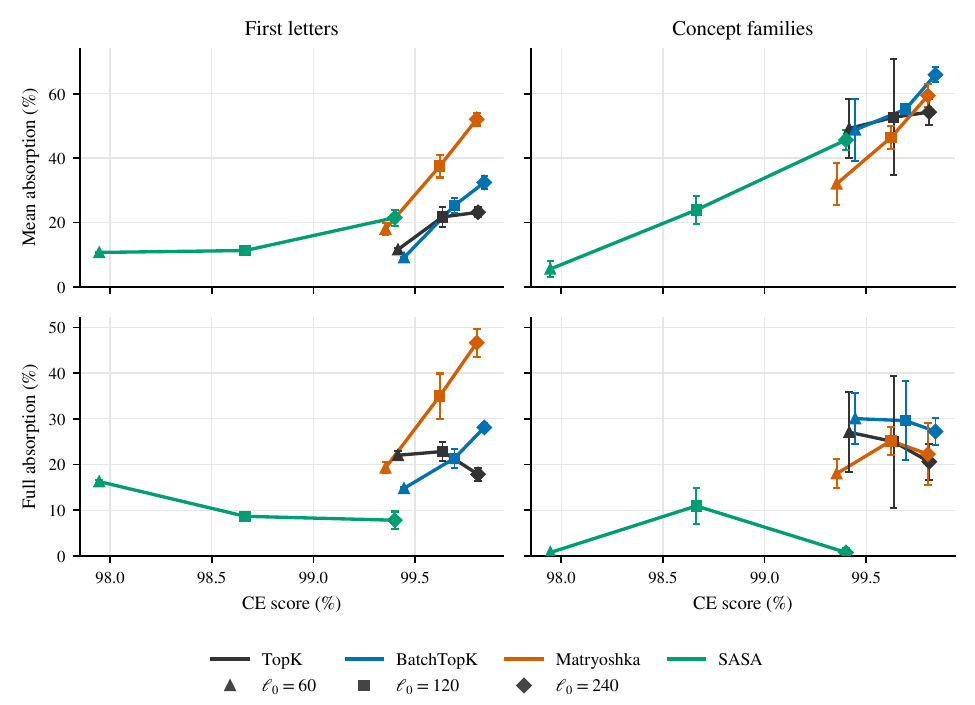}
\caption{Reconstruction and absorption across the matched GPT-2 sparsity sweep at $24{,}576$ decoder columns. Columns show first letters (left) and concept families (right), and rows show mean absorption (top) and full absorption (bottom). Error bars indicate standard deviations over three seeds.}
\label{fig:matched-tradeoff}
\end{figure}

\subsection{Absorption}
\label{apx:exp-absorption}
In this section, we assess whether a feature is consistently represented by the same units across inputs. To do so, we compare feature absorption in SASA and vector SAEs using the first-letter benchmark and its extension to concept families.

\subsubsection{First letters}
We evaluate feature absorption using the first-letter absorption benchmark~\citep{Chanin_Wilken-Smith_Dulka_Bhatnagar_Golechha_Bloom_2025}. Prompts follow the in-context spelling template \texttt{\{word\} has the first letter:}, with up to ten examples drawn from alphabetic tokens. We record activations at the word token and evaluate individual latent coordinates for both SAEs and SASA.

For each letter $\ell$, we train $k$-sparse logistic probes on the latent activations to predict whether a word starts with $\ell$, for $k=1,\dots,10$. Coordinates are ranked by the weights of an $\ell_1$-regularized probe, and each additional coordinate is retained only if it improves F1 by more than $0.03$. The selected coordinates form the letter's \emph{main latents}.

Latent contributions are measured by projecting their decoded outputs onto the unit-norm residual-stream probe direction $\boldsymbol{p}_\ell$ for letter $\ell$. Mean fractional absorption measures the contribution of other latents when the main latents account for less than the projection $\boldsymbol{h}^{\top}\boldsymbol{p}_\ell$. We follow the benchmark's scoring rule, using up to three absorbing latents with decoder--probe cosine similarity of at least $0.1$ and a combined contribution of at least $40\%$ of this projection. Full absorption requires all main activations to be below $10^{-8}$ and the latent with the largest contribution to explain at least $40\%$ of the projection, with cosine similarity of at least $0.025$. Both scores are averaged over the residual-stream probe's true positives for each letter, then over letters with probe F1 above $0.6$.

\begin{table}[tbp]
\centering
\caption{Feature absorption on the first-letter benchmark~\citep{Chanin_Wilken-Smith_Dulka_Bhatnagar_Golechha_Bloom_2025}. GPT-2 uses matched $24{,}576$-column dictionaries, and the Mistral-7B comparison uses the published SAE. Values are mean $\pm$ standard deviation over three seeds.}
\label{tab:absorption_summary}
\small
\setlength{\tabcolsep}{4pt}
\begin{tabular}{@{}lccc@{}}
\toprule
Model & $\ell_0=60$ & $\ell_0=120$ & $\ell_0=240$ \\
\midrule
\multicolumn{4}{@{}l}{\emph{GPT-2, mean absorption (\%)} $\downarrow$} \\
TopK & $11.6{\pm}0.6$ & $21.7{\pm}3.0$ & $23.1{\pm}1.5$ \\
BatchTopK & $\mathbf{9.1{\pm}0.2}$ & $25.3{\pm}2.3$ & $32.4{\pm}2.1$ \\
Matryoshka & $17.9{\pm}1.9$ & $37.5{\pm}3.5$ & $52.1{\pm}1.9$ \\
SASA & $10.7{\pm}0.2$ & $\mathbf{11.3{\pm}0.9}$ & $\mathbf{21.5{\pm}2.5}$ \\
\midrule
\multicolumn{4}{@{}l}{\emph{GPT-2, full absorption (\%)} $\downarrow$} \\
TopK & $22.0{\pm}0.9$ & $22.8{\pm}2.1$ & $17.8{\pm}1.4$ \\
BatchTopK & $\mathbf{14.8{\pm}0.2}$ & $21.4{\pm}2.1$ & $28.1{\pm}0.9$ \\
Matryoshka & $19.3{\pm}1.2$ & $35.0{\pm}5.0$ & $46.6{\pm}3.0$ \\
SASA & $16.3{\pm}0.3$ & $\mathbf{8.7{\pm}0.6}$ & $\mathbf{7.8{\pm}1.9}$ \\
\bottomrule
\end{tabular}
\par\medskip
\begin{tabular}{@{}lcc@{}}
\toprule
\multicolumn{3}{@{}l}{\emph{Mistral-7B, published SAE and SASA}} \\
Model & Mean absorption (\%) $\downarrow$ & Full absorption (\%) $\downarrow$ \\
\midrule
ReLU SAE & $66.1{\pm}1.8$ & $17.0{\pm}0.8$ \\
SASA & $\mathbf{25.9{\pm}1.6}$ & $\mathbf{4.4{\pm}0.3}$ \\
\bottomrule
\end{tabular}
\end{table}

\subsubsection{Concept families}
\label{apx:family-absorption}
We extend the first-letter benchmark to the ten concept families in the main text, each with $4$--$16$ members. For each family, a linear probe predicts family membership from GPT-2 activations collected from Pile-10k. Examples for each member and background tokens are divided equally between training and testing.

We evaluate one scalar SAE latent or one SASA group as a unit. Let $\boldsymbol{y}_u(\boldsymbol{h})$ denote the unit's decoded contribution to the residual stream and $\boldsymbol{p}$ the unit-norm probe direction for the family. Its contribution is $\boldsymbol{y}_u(\boldsymbol{h})^\top\boldsymbol{p}$, and its alignment is the cosine similarity between $\boldsymbol{y}_u(\boldsymbol{h})$ and $\boldsymbol{p}$. A unit is inactive when $\|\boldsymbol{y}_u(\boldsymbol{h})\|_2<10^{-8}$. Main units are selected by sparse probes on their decoded outputs, with a group-sparse penalty for SASA. Logistic probes use all coordinates of the selected units. We retain the first-letter absorption thresholds and F1 stopping rule and average scores over families.

\begin{table}[tbp]
\centering
\caption{Mean and full absorption across ten concept families at $24{,}576$ decoder columns. Values are mean $\pm$ standard deviation over three seeds.}
\label{tab:family-absorption-full}
\small
\setlength{\tabcolsep}{4pt}
\begin{tabular}{@{}lccc@{}}
\toprule
Model & $\ell_0=60$ & $\ell_0=120$ & $\ell_0=240$ \\
\midrule
\multicolumn{4}{@{}l}{\emph{Mean concept-family absorption (\%)} $\downarrow$} \\
TopK & $49.2{\pm}9.2$ & $52.7{\pm}18.0$ & $54.3{\pm}4.0$ \\
BatchTopK & $48.8{\pm}9.6$ & $55.2{\pm}1.4$ & $65.9{\pm}2.3$ \\
Matryoshka & $32.0{\pm}6.5$ & $46.4{\pm}3.6$ & $59.4{\pm}3.7$ \\
SASA & $\mathbf{5.5{\pm}2.6}$ & $\mathbf{23.9{\pm}4.4}$ & $\mathbf{45.7{\pm}3.1}$ \\
\midrule
\multicolumn{4}{@{}l}{\emph{Full concept-family absorption (\%)} $\downarrow$} \\
TopK & $27.1{\pm}8.8$ & $25.0{\pm}14.4$ & $20.6{\pm}4.0$ \\
BatchTopK & $30.0{\pm}5.6$ & $29.6{\pm}8.7$ & $27.2{\pm}2.9$ \\
Matryoshka & $18.0{\pm}3.2$ & $25.2{\pm}3.1$ & $22.3{\pm}6.8$ \\
SASA & $\mathbf{0.8{\pm}0.0}$ & $\mathbf{10.9{\pm}4.0}$ & $\mathbf{0.8{\pm}1.0}$ \\
\bottomrule
\end{tabular}
\end{table}

At $\ell_0=120$, SASA's mean and full family absorption are $23.9\%$ and $10.9\%$, compared with $46.4\%$ and $25.0\%$ for the best vector baselines. The lower full absorption indicates fewer gaps in the activation of the groups associated with each family.

\noindent\textbf{Dictionary width.}~
\label{apx:width-grid}
Table~\ref{tab:width-grid} compares three dictionary widths. At $24{,}576$ and $49{,}152$ columns, SASA has the lowest mean family absorption at every tested sparsity. Mean first-letter absorption depends on the setting, with SASA lowest at $\ell_0=120$ and $240$ for $24{,}576$ columns and at $\ell_0=240$ for $49{,}152$ columns.

\begin{table}[tbp]
\centering
\caption{First-letter and concept-family absorption (\%) on GPT-2 layer~7 at three dictionary widths. SASA uses width-6 groups with $s=\ell_0/6$. Values are mean $\pm$ standard deviation over seeds. Lower is better.}
\label{tab:width-grid}
\small
\setlength{\tabcolsep}{4pt}
\begin{tabular}{@{}llccc@{}}
\toprule
Columns & Model & $\ell_0=60$ & $\ell_0=120$ & $\ell_0=240$ \\
\midrule
\multicolumn{5}{@{}l}{\emph{Mean first-letter absorption (\%)}} \\
$12{,}288$ & TopK & $\mathbf{11.0{\pm}0.5}$ & $\mathbf{13.8{\pm}1.6}$ & $15.4{\pm}1.5$ \\
 & BatchTopK & $18.5{\pm}1.5$ & $31.7{\pm}2.9$ & $35.9{\pm}1.5$ \\
 & Matryoshka & $15.3{\pm}2.1$ & $38.4{\pm}1.6$ & $35.8{\pm}3.8$ \\
 & SASA & $15.9{\pm}1.4$ & $15.1{\pm}1.3$ & $\mathbf{14.8{\pm}2.4}$ \\
\addlinespace
$24{,}576$ & TopK & $11.6{\pm}0.6$ & $21.7{\pm}3.0$ & $23.1{\pm}1.5$ \\
 & BatchTopK & $\mathbf{9.1{\pm}0.2}$ & $25.3{\pm}2.3$ & $32.4{\pm}2.1$ \\
 & Matryoshka & $17.9{\pm}1.9$ & $37.5{\pm}3.5$ & $52.1{\pm}1.9$ \\
 & SASA & $10.7{\pm}0.2$ & $\mathbf{11.3{\pm}0.9}$ & $\mathbf{21.5{\pm}2.5}$ \\
\addlinespace
$49{,}152$ & TopK & $12.0{\pm}0.3$ & $\mathbf{13.2{\pm}0.4}$ & $19.3{\pm}0.8$ \\
 & BatchTopK & $\mathbf{10.6{\pm}0.1}$ & $14.5{\pm}2.0$ & $43.5{\pm}2.8$ \\
 & Matryoshka & $15.4{\pm}1.2$ & $33.6{\pm}3.6$ & $41.6{\pm}3.1$ \\
 & SASA & $15.6{\pm}1.1$ & $16.6{\pm}0.2$ & $\mathbf{14.0{\pm}3.1}$ \\
\midrule
\multicolumn{5}{@{}l}{\emph{Full first-letter absorption (\%)}} \\
$12{,}288$ & TopK & $17.0{\pm}0.5$ & $14.5{\pm}2.2$ & $11.1{\pm}1.5$ \\
 & BatchTopK & $20.7{\pm}2.3$ & $28.8{\pm}5.4$ & $32.6{\pm}1.4$ \\
 & Matryoshka & $16.0{\pm}2.0$ & $33.4{\pm}1.9$ & $28.3{\pm}5.8$ \\
 & SASA & $\mathbf{14.5{\pm}0.3}$ & $\mathbf{9.7{\pm}0.4}$ & $\mathbf{7.2{\pm}0.7}$ \\
\addlinespace
$24{,}576$ & TopK & $22.0{\pm}0.9$ & $22.8{\pm}2.1$ & $17.8{\pm}1.4$ \\
 & BatchTopK & $\mathbf{14.8{\pm}0.2}$ & $21.4{\pm}2.1$ & $28.1{\pm}0.9$ \\
 & Matryoshka & $19.3{\pm}1.2$ & $35.0{\pm}5.0$ & $46.6{\pm}3.0$ \\
 & SASA & $16.3{\pm}0.3$ & $\mathbf{8.7{\pm}0.6}$ & $\mathbf{7.8{\pm}1.9}$ \\
\addlinespace
$49{,}152$ & TopK & $25.5{\pm}0.0$ & $17.9{\pm}0.5$ & $16.0{\pm}1.8$ \\
 & BatchTopK & $19.0{\pm}0.3$ & $16.7{\pm}3.5$ & $33.6{\pm}3.9$ \\
 & Matryoshka & $17.9{\pm}1.0$ & $27.8{\pm}3.7$ & $36.8{\pm}3.1$ \\
 & SASA & $\mathbf{17.3{\pm}0.7}$ & $\mathbf{10.9{\pm}0.2}$ & $\mathbf{6.4{\pm}1.0}$ \\
\midrule
\multicolumn{5}{@{}l}{\emph{Mean concept-family absorption (\%)}} \\
$12{,}288$ & TopK & $43.4{\pm}10.1$ & $49.7{\pm}4.1$ & $59.1{\pm}9.4$ \\
 & BatchTopK & $48.1{\pm}6.0$ & $52.2{\pm}3.6$ & $61.9{\pm}5.6$ \\
 & Matryoshka & $45.9{\pm}8.3$ & $51.0{\pm}10.7$ & $73.2{\pm}7.3$ \\
 & SASA & $\mathbf{5.5{\pm}2.2}$ & $\mathbf{25.5{\pm}5.0}$ & $\mathbf{36.6{\pm}8.3}$ \\
\addlinespace
$24{,}576$ & TopK & $49.2{\pm}9.2$ & $52.7{\pm}18.0$ & $54.3{\pm}4.0$ \\
 & BatchTopK & $48.8{\pm}9.6$ & $55.2{\pm}1.4$ & $65.9{\pm}2.3$ \\
 & Matryoshka & $32.0{\pm}6.5$ & $46.4{\pm}3.6$ & $59.4{\pm}3.7$ \\
 & SASA & $\mathbf{5.5{\pm}2.6}$ & $\mathbf{23.9{\pm}4.4}$ & $\mathbf{45.7{\pm}3.1}$ \\
\addlinespace
$49{,}152$ & TopK & $38.7{\pm}2.7$ & $50.2{\pm}3.0$ & $58.6{\pm}3.9$ \\
 & BatchTopK & $54.7{\pm}7.9$ & $65.1{\pm}9.1$ & $65.9{\pm}11.3$ \\
 & Matryoshka & $29.1{\pm}3.6$ & $49.6{\pm}1.8$ & $66.9{\pm}9.9$ \\
 & SASA & $\mathbf{13.7{\pm}0.7}$ & $\mathbf{31.5{\pm}1.4}$ & $\mathbf{58.1{\pm}6.8}$ \\
\bottomrule
\end{tabular}
\end{table}

Table~\ref{tab:matched-sweep} also reports reconstruction and member prediction for the $12{,}288$-column sweep.
\begin{table}[tbp]
\centering
\caption{Matched GPT-2 sparsity sweep at $12{,}288$ decoder columns. Family absorption and member prediction report mean $\pm$ standard deviation over six seeds and three 70/30 train--test splits, respectively. Member prediction uses at most six decoder columns per family and averages ten eligible families (seven for BatchTopK at $\ell_0=240$).}
\label{tab:matched-sweep}
\small
\setlength{\tabcolsep}{4pt}
\begin{tabular}{@{}lcccc@{}}
\toprule
Model & $\ell_0=30$ & $\ell_0=60$ & $\ell_0=120$ & $\ell_0=240$ \\
\midrule
\multicolumn{5}{@{}l}{\emph{CE score (\%)} $\uparrow$} \\
TopK & $98.7$ & $99.2$ & $99.5$ & $\mathbf{99.8}$ \\
BatchTopK & $\mathbf{98.8}$ & $\mathbf{99.3}$ & $\mathbf{99.6}$ & $\mathbf{99.8}$ \\
Matryoshka & $98.7$ & $99.2$ & $99.5$ & $\mathbf{99.8}$ \\
SASA & $95.8$ & $97.3$ & $98.5$ & $99.3$ \\
\midrule
\multicolumn{5}{@{}l}{\emph{Concept-family absorption (\%)} $\downarrow$} \\
TopK & $46.8{\pm}10.0$ & $43.4{\pm}10.1$ & $49.7{\pm}4.1$ & $59.1{\pm}9.4$ \\
BatchTopK & $37.3{\pm}5.7$ & $48.1{\pm}6.0$ & $52.2{\pm}3.6$ & $61.9{\pm}5.6$ \\
Matryoshka & $26.5{\pm}13.0$ & $45.9{\pm}8.3$ & $51.0{\pm}10.7$ & $73.2{\pm}7.3$ \\
SASA & $\mathbf{7.9{\pm}0.9}$ & $\mathbf{5.5{\pm}2.2}$ & $\mathbf{25.5{\pm}5.0}$ & $\mathbf{36.6{\pm}8.3}$ \\
\midrule
\multicolumn{5}{@{}l}{\emph{Member prediction (\%)} $\uparrow$} \\
TopK & $51.6{\pm}1.0$ & $66.1{\pm}0.7$ & $66.1{\pm}0.5$ & $65.8{\pm}0.5$ \\
BatchTopK & $48.5{\pm}1.2$ & $56.0{\pm}2.1$ & $63.0{\pm}0.3$ & $46.1{\pm}2.1$ \\
Matryoshka & $52.7{\pm}0.9$ & $59.9{\pm}1.5$ & $57.7{\pm}1.1$ & $35.4{\pm}1.2$ \\
SASA & $\mathbf{69.2{\pm}0.2}$ & $\mathbf{77.7{\pm}0.9}$ & $\mathbf{78.2{\pm}1.1}$ & $\mathbf{68.0{\pm}1.5}$ \\
\bottomrule
\end{tabular}
\end{table}

\subsection{Predicting family members}
\label{apx:value-readout}
We test whether units selected to identify a concept family also encode which member is present. We collect examples of the same ten concept families from the Pile dataset, retaining members with at least 20 examples and families with at least three retained members. We rank units by their F1 score for identifying the family, treating all collected family examples as positives and $20{,}000$ additional Pile tokens as negatives. Units must activate on at least half of the family's examples and have precision of at least 0.2.

The main comparison uses one width-6 SASA group or up to six SAE latents, and the larger budget allows up to four groups or 24 latents. We train a logistic-regression classifier on the selected code coordinates to predict the member identity and measure its accuracy on held-out examples. Reported accuracy is averaged over families with selected units.

For the main comparison at $24{,}576$ columns and $\ell_0=120$, increasing the prediction budget from six to 24 columns raises SASA accuracy from 86.2\% to 94.8\%. The best vector baseline, BatchTopK, increases from 68.9\% to 86.7\%. Additional results are reported in Table~\ref{tab:value-readout-budgets}.

\begin{table}[tbp]
\centering
\caption{Accuracy of predicting family members (\%) at $12{,}288$ decoder columns. Values are mean $\pm$ standard deviation over three seeds. Accuracy is averaged over ten families, or seven for BatchTopK at $\ell_0=240$.}
\label{tab:value-readout-budgets}
\small
\setlength{\tabcolsep}{4pt}
\begin{tabular}{@{}lcccc@{}}
\toprule
Model & $\ell_0=30$ & $\ell_0=60$ & $\ell_0=120$ & $\ell_0=240$ \\
\midrule
\multicolumn{5}{@{}l}{\emph{Up to 6 decoder columns}} \\
TopK & $51.6{\pm}1.0$ & $66.1{\pm}0.7$ & $66.1{\pm}0.5$ & $65.8{\pm}0.5$ \\
BatchTopK & $48.5{\pm}1.2$ & $56.0{\pm}2.1$ & $63.0{\pm}0.3$ & $46.1{\pm}2.1$ \\
Matryoshka & $52.7{\pm}0.9$ & $59.9{\pm}1.5$ & $57.7{\pm}1.1$ & $35.4{\pm}1.2$ \\
SASA & $\mathbf{69.2{\pm}0.2}$ & $\mathbf{77.7{\pm}0.9}$ & $\mathbf{78.2{\pm}1.1}$ & $\mathbf{68.0{\pm}1.5}$ \\
\midrule
\multicolumn{5}{@{}l}{\emph{Up to 24 decoder columns}} \\
TopK & $51.6{\pm}1.0$ & $67.8{\pm}0.8$ & $84.4{\pm}0.7$ & $81.3{\pm}0.3$ \\
BatchTopK & $48.5{\pm}1.2$ & $60.0{\pm}1.6$ & $84.5{\pm}0.1$ & $54.3{\pm}1.2$ \\
Matryoshka & $52.7{\pm}0.9$ & $76.9{\pm}1.1$ & $66.8{\pm}1.6$ & $36.8{\pm}1.7$ \\
SASA & $\mathbf{69.7{\pm}0.4}$ & $\mathbf{87.0{\pm}0.8}$ & $\mathbf{95.7{\pm}0.6}$ & $\mathbf{87.6{\pm}1.0}$ \\
\bottomrule
\end{tabular}
\end{table}

\subsection{Feature dimension and rank}
\label{apx:feature-dimension-rank}
In this section, we vary feature dimension, group width, and nuclear-norm regularization to study splitting and effective rank. We also measure overlap between learned group subspaces.

\subsubsection{Effect of feature dimension}
\label{apx:synthetic-dimension}
We study how feature dimension and sparsity affect splitting using activations formed by superposing four of 256 features from known subspaces of $\mathbb{R}^{768}$. Feature dimensions range from 1 to 16, with equal activation probability and expected squared norm. We add isotropic noise whose expected squared norm is 1\% of that of the noiseless activations. All dictionaries have $4{,}096$ decoder columns and train on about $12\mathrm{M}$ activations. We compare TopK and SASA at scalar sparsities $\ell_0\in\{64,32\}$, using SASA group widths $r=16$ and $r=8$, respectively. SASA activates four groups per input and uses $\lambda_{\mathrm{dim}}=0$ to isolate the effect of group width.

For each feature, we count the TopK atoms accounting for 90\% of the summed squared norms of decoded contributions projected onto its subspace, using inputs on which that feature is active. For SASA, we count groups whose combined decoder span captures 90\% of the known feature subspace. We also measure explained variance within each feature subspace to assess reconstruction. Table~\ref{tab:synthetic-dimension} reports averages by feature dimension over three seeds.

\begin{table}[tbp]
\centering
\caption{Feature dimension and sparsity on synthetic activations at matched scalar budgets. Values are averages over three seeds.}
\label{tab:synthetic-dimension}
\small
\setlength{\tabcolsep}{4pt}
\begin{tabular}{@{}lcccccc@{}}
\toprule
Model & $\ell_0$ & $d_i=1$ & $d_i=4$ & $d_i=8$ & $d_i=12$ & $d_i=16$ \\
\midrule
\multicolumn{7}{@{}l}{\emph{Atoms per feature (TopK) or groups per feature (SASA)}} \\
TopK & $64$ & $2.0$ & $8.1$ & $17.9$ & $25.9$ & $31.8$ \\
TopK & $32$ & $2.0$ & $8.0$ & $15.0$ & $22.7$ & $646.8$ \\
\addlinespace
SASA, $r=16$ & $64$ & $1.0$ & $1.0$ & $1.0$ & $1.0$ & $1.2$ \\
SASA, $r=8$ & $32$ & $1.0$ & $1.0$ & $1.0$ & $2.0$ & $2.1$ \\
\midrule
\multicolumn{7}{@{}l}{\emph{Explained variance within the feature subspace}} \\
TopK & $64$ & $99.5\%$ & $98.1\%$ & $96.6\%$ & $94.6\%$ & $92.6\%$ \\
TopK & $32$ & $99.6\%$ & $98.5\%$ & $96.7\%$ & $94.3\%$ & $47.1\%$ \\
\addlinespace
SASA, $r=16$ & $64$ & $97.8\%$ & $98.9\%$ & $98.4\%$ & $97.8\%$ & $95.6\%$ \\
SASA, $r=8$ & $32$ & $91.6\%$ & $96.0\%$ & $97.6\%$ & $84.1\%$ & $83.0\%$ \\
\bottomrule
\end{tabular}
\end{table}

For $d_i=16$, reducing $\ell_0$ from 64 to 32 increases the TopK atom count from 31.8 to 646.8, while the count for $d_i=4$ remains near 8. At the same budgets, SASA captures the 16-dimensional subspace with 1.2 groups of width 16 and 2.1 groups of width 8 on average, close to the minimum of one and two groups, respectively.

\subsubsection{Effect of nuclear-norm regularization}
\label{apx:rank-adaptivity}
We test how nuclear-norm regularization changes the effective rank of groups representing features of different dimensions. Synthetic activations superpose eight of 128 features with dimensions 1--8. SASA has 128 groups of width 10 and activates eight groups per input. We match groups to the known feature subspaces and measure the effective rank of $\boldsymbol{D}_k\boldsymbol{E}_k$ as the fewest leading singular values whose squares sum to at least 90\% of $\|\boldsymbol{D}_k\boldsymbol{E}_k\|_{\mathrm{F}}^2$.

Table~\ref{tab:synthetic-rank} compares identical training runs with and without the penalty. Regularization reduces effective rank more strongly for lower-dimensional features, from 7.1 to 4.2 for one-dimensional features, compared with 8.1 to 7.9 for eight-dimensional features. Explained variance slightly decreases from 98.3\% to 97.2\%. This shows that the regularizer affects each group to adapt to the intrinsic dimension of its data

\begin{table}[tbp]
\centering
\caption{Effect of nuclear-norm regularization on effective rank and explained variance (EV). Ranks of $\boldsymbol{D}_k\boldsymbol{E}_k$ are averaged over active groups capturing at least half of their matched feature subspace. Values average three seeds.}
\label{tab:synthetic-rank}
\small
\setlength{\tabcolsep}{4pt}
\begin{tabular}{@{}lccccccccc@{}}
\toprule
Coefficient & $d_i=1$ & 2 & 3 & 4 & 5 & 6 & 7 & 8 & EV (\%) \\
\midrule
0 & 7.1 & 6.8 & 6.1 & 6.1 & 6.4 & 6.4 & 7.2 & 8.1 & 98.3 \\
300 & 4.2 & 4.7 & 4.7 & 5.5 & 5.5 & 6.5 & 7.0 & 7.9 & 97.2 \\
\bottomrule
\end{tabular}
\end{table}

\subsubsection{Group width}
\label{apx:rank-grid}
At fixed scalar sparsity, increasing group width gives each group more coordinates but reduces the number of active groups, since $s=\ell_0/r$. Table~\ref{tab:rank-grid} compares widths $r\in\{2,4,6,8\}$. At $\ell_0=30$, reducing $r$ from 6 to 2 allows 15 active groups instead of 5 and lowers mean first-letter absorption from 14.9\% to 7.7\%, while mean family absorption rises from 7.9\% to 19.3\%.

\begin{table}[tbp]
\centering
\caption{SASA group width on GPT-2 layer~7 at $12{,}288$ decoder columns. Absorption averages six seeds. Member prediction uses up to 24 decoder columns and averages three seeds. Dashes indicate settings where $\ell_0/r$ is not an integer.}
\label{tab:rank-grid}
\small
\setlength{\tabcolsep}{4pt}
\begin{tabular}{@{}lcccc@{}}
\toprule
Width $r$ & $\ell_0=30$ & $\ell_0=60$ & $\ell_0=120$ & $\ell_0=240$ \\
\midrule
\multicolumn{5}{@{}l}{\emph{CE score (\%)} $\uparrow$} \\
2 & 96.8 & 98.2 & 99.0 & 99.6 \\
4 & -- & 97.6 & 98.8 & 99.4 \\
6 & 95.8 & 97.3 & 98.5 & 99.3 \\
8 & -- & -- & 98.4 & 99.3 \\
\midrule
\multicolumn{5}{@{}l}{\emph{Mean first-letter absorption (\%)}} \\
2 & 7.7 & 15.4 & 12.3 & 17.8 \\
4 & -- & 11.6 & 14.7 & 16.9 \\
6 & 14.9 & 15.9 & 15.1 & 14.8 \\
8 & -- & -- & 13.2 & 15.6 \\
\midrule
\multicolumn{5}{@{}l}{\emph{Full first-letter absorption (\%)}} \\
2 & 17.3 & 11.7 & 5.9 & 7.3 \\
4 & -- & 14.9 & 10.3 & 5.2 \\
6 & 22.0 & 14.5 & 9.7 & 7.2 \\
8 & -- & -- & 10.4 & 8.2 \\
\midrule
\multicolumn{5}{@{}l}{\emph{Mean concept-family absorption (\%)} $\downarrow$} \\
2 & 19.3 & 29.2 & 51.4 & 78.9 \\
4 & -- & 28.1 & 36.0 & 57.4 \\
6 & 7.9 & 5.5 & 25.5 & 36.6 \\
8 & -- & -- & 5.1 & 50.5 \\
\midrule
\multicolumn{5}{@{}l}{\emph{Member prediction at 24 columns (\%)} $\uparrow$} \\
2 & 62.7 & 80.9 & 92.4 & 87.7 \\
4 & -- & 85.0 & 95.7 & 90.7 \\
6 & 69.7 & 87.0 & 95.7 & 87.6 \\
8 & -- & -- & 91.4 & 88.2 \\
\bottomrule
\end{tabular}
\end{table}

\subsubsection{Subspace overlap}
\label{apx:subspace-overlap}
We examine whether groups use distinct directions in the $12{,}288$-column GPT-2 SASA ($r=6$, $\ell_0=60$), considering the 500 most active groups. For each group, we use principal directions explaining 90\% of the variance of its decoded outputs when active. Overlap, measured by the largest cosine between directions in a pair of subspaces, has median 0.131, with 0.16\% of pairs above 0.5.

\subsection{Causal interventions}
\label{apx:causal}

We test whether interventions within SASA's learned subspaces selectively change Mistral-7B's probability of generating month or weekday names. We use the layer-8 SASA with $K=4096$ groups of width $r=8$.

For each concept, we select the group whose pre-activation norm has the highest AUROC for detecting its names. We use 150 examples per name and $20{,}000$ additional tokens from Pile-10k, independently of the intervention prompts, retaining ``May'' and ``March'' only in date contexts. The group with the next-highest AUROC provides a control for the same concept. For a control associated with a different concept, we use the weekday group in month interventions and the month group in weekday interventions.

Each prompt contains a sentence naming a month or weekday, followed by a repetition that ends just before the name, such as \emph{``The class meets on Monday. The class meets on''}. We intervene at the earlier name's position in layer~8 and measure $\mathbb{P}(\mathcal{C})$, the total next-token probability assigned to the relevant names at the completion position. For addition, this name is first replaced by an unrelated noun.

\noindent\textbf{Addition.}~We construct the steering vector $\boldsymbol{v}_k$ by decoding group $k$'s coordinates at month or weekday names in the first four templates of each concept below, restoring the original activation scale, and averaging the resulting vectors. We evaluate on the remaining 12 templates with an unrelated noun in place of the name, using eight nouns per template. At this word's position, we apply
\[
\boldsymbol{h}\leftarrow\boldsymbol{h}+\alpha\boldsymbol{v}_k,
\qquad \alpha\in\{0.5,1,2,4\}.
\]
We repeat the intervention with steering vectors from the control groups and with random directions. All vectors are scaled to $\|\boldsymbol{v}_k\|_2$ before applying the same $\alpha$, so the comparisons hold the intervention magnitude fixed.

Without steering, the model assigns 5.6\% probability to month names and 5.0\% to weekday names on these prompts. Adding the selected group's vector raises these probabilities to 78.7\% and 80.1\% at $\alpha=2$, with smaller increases from the controls (Table~\ref{tab:causal-steering}). To examine whether the effect persists across contexts, we also apply the month vector to weekday templates and the weekday vector to month templates. Each vector preferentially increases the probability of the corresponding names in both sets of prompts (Table~\ref{tab:causal-specificity}).

\noindent\textbf{Removal.}~We instantiate 16 templates with every month or weekday, giving 192 month prompts and 112 weekday prompts. For the selected group $k$, we remove the orthogonal projection of the layer-normalized hidden state $\tilde{\boldsymbol{h}}$ onto its decoder span,
\[
\tilde{\boldsymbol{h}}\leftarrow
\bigl(\boldsymbol{I}-\boldsymbol{P}_{\operatorname{col}(\boldsymbol{D}_k)}\bigr)\tilde{\boldsymbol{h}}.
\]
We then undo the normalization using the original mean and scale before continuing the forward pass. We repeat this projection removal for the two control groups and random subspaces, all of dimension eight. To test whether the effect is due to the magnitude of the intervention, we also subtract a random vector whose norm equals that of the removed component for each input.

Removing the selected subspace component lowers $\mathbb{P}(\mathcal{C})$ from 86.1\% to 40.9\% for months and from 90.0\% to 45.2\% for weekdays (Table~\ref{tab:causal-ablation}). The control interventions leave these probabilities nearly unchanged, including subtraction of a random vector with the same norm as the removed component. Thus, the large decrease depends on the selected subspace rather than merely the intervention magnitude.

\begin{table}[tbp]
\centering
\caption{Next-token probability of month or weekday names (\%) as the steering scale $\alpha$ varies. Within each concept, vectors have the same norm before scaling.}
\label{tab:causal-steering}
\small
\setlength{\tabcolsep}{5pt}
\begin{tabular}{@{}lcccc@{}}
\toprule
\multirow{2}{*}{Scale $\alpha$} & \multirow{2}{*}{Selected group} & \multicolumn{3}{c}{Controls} \\
\cmidrule(lr){3-5}
 & & Same concept & Different concept & Random \\
\midrule
\multicolumn{5}{@{}l}{\emph{Months}} \\
0 & 5.6 & 5.6 & 5.6 & 5.6 \\
0.5 & \textbf{8.3} & 5.7 & 6.5 & 5.6 \\
1 & \textbf{34.6} & 6.1 & 8.3 & 5.6 \\
2 & \textbf{78.7} & 7.2 & 7.2 & 5.9 \\
4 & \textbf{83.0} & 15.3 & 7.2 & 6.5 \\
\midrule
\multicolumn{5}{@{}l}{\emph{Weekdays}} \\
0 & 5.0 & 5.0 & 5.0 & 5.0 \\
0.5 & \textbf{9.3} & 5.0 & 6.5 & 5.1 \\
1 & \textbf{52.0} & 5.8 & 12.0 & 5.4 \\
2 & \textbf{80.1} & 11.3 & 8.7 & 5.9 \\
4 & \textbf{81.7} & 15.5 & 8.2 & 8.0 \\
\bottomrule
\end{tabular}
\end{table}

\begin{table}[tbp]
\centering
\caption{Mean next-token probabilities under subspace removal and control interventions.}
\label{tab:causal-ablation}
\small
\setlength{\tabcolsep}{4pt}
\begin{tabular}{@{}lcc@{}}
\toprule
Intervention & Month names (\%) & Weekday names (\%) \\
\midrule
Unsteered & 86.1 & 90.0 \\
\midrule
Selected group & \textbf{40.9} & \textbf{45.2} \\
Control (same concept) & 85.9 & 88.1 \\
Control (different concept) & 84.7 & 88.9 \\
Random subspace & 86.1 & 90.0 \\
Subtract random vector & 86.2 & 89.7 \\
\bottomrule
\end{tabular}
\end{table}

\begin{table}[tbp]
\centering
\caption{Next-token probabilities without steering and with each steering vector at $\alpha=2$.}
\label{tab:causal-specificity}
\small
\begin{tabular}{@{}llcc@{}}
\toprule
Prompt templates & Steering vector & Month names (\%) & Weekday names (\%) \\
\midrule
Months & Unsteered & 5.6 & 0.1 \\
 & Month vector & \textbf{78.7} & 0.4 \\
 & Weekday vector & 7.2 & \textbf{43.4} \\
\midrule
Weekdays & Unsteered & 0.6 & 5.0 \\
 & Month vector & \textbf{56.8} & 8.7 \\
 & Weekday vector & 0.5 & \textbf{80.1} \\
\bottomrule
\end{tabular}
\end{table}

\noindent\textbf{Templates.}~The templates below use \texttt{M} for a month and \texttt{D} for a weekday. For addition, we replace the name with one of the following 25 single-token words (lamp, chair, bottle, hammer, basket, mirror, pillow, ladder, blanket, candle, bucket, camera, clock, drawer, envelope, fork, guitar, helmet, jacket, knife, plate, radio, spoon, towel, wallet).

\begin{quote}\small\ttfamily
The event happened in M. The event happened in\\
The contract ends in M. The contract ends in\\
She will return in M. She will return in\\
The festival is held in M. The festival is held in\\
The report is due in M. The report is due in\\
The season starts in M. The season starts in\\
The meeting was moved to M. The meeting was moved to\\
Our lease ends in M. Our lease ends in\\
The store closed in M. The store closed in\\
The trial began in M. The trial began in\\
The book came out in M. The book came out in\\
The treaty was signed in M. The treaty was signed in\\
The project was approved in M. The project was approved in\\
We adopted a cat in M. We adopted a cat in\\
The bridge was completed in M. The bridge was completed in\\
Her contract expires in M. Her contract expires in
\end{quote}

\begin{quote}\small\ttfamily
The meeting is on D. The meeting is on\\
She called me on D. She called me on\\
The deadline is D. The deadline is\\
They travel every D. They travel every\\
The class meets on D. The class meets on\\
He arrived last D. He arrived last\\
Payday falls on D. Payday falls on\\
The show airs every D. The show airs every\\
The market opens on D. The market opens on\\
We play tennis on D. We play tennis on\\
The package arrived on D. The package arrived on\\
Her shift starts on D. Her shift starts on\\
The museum is closed on D. The museum is closed on\\
Rehearsal is every D. Rehearsal is every\\
The trash is collected on D. The trash is collected on\\
Our next game is on D. Our next game is on
\end{quote}

\subsection{Empirical Line Covering in GPT-2}
\label{apx:real-line-cover}
Figure~\ref{fig:real-line-cover} compares empirical covering error on a GPT-2 temporal feature with the $k=1$ scaling in Theorem~\ref{thm:line-covering}. We estimate $\mathrm{dim90}=10$ from $726$ SAE temporal reconstructions and fit the reference line's intercept by least squares over all $35$ values of $N$. For the threshold, $\hat{t}$ is the 10th percentile of projected norms and $\hat{\rho}$ is the 80th percentile of angles in the slice. A conservative lower bound on $\hat{t}\sin\hat{\rho}$ with $120$ resamples gives $0.691$.

\begin{figure}[htbp]
    \centering
    \includegraphics[width=0.55\linewidth]{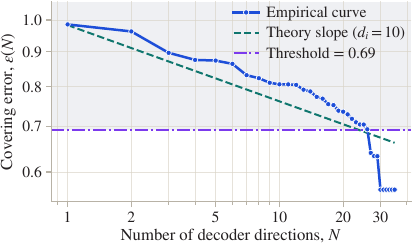}  \captionsetup{font=small,skip=4pt,justification=raggedright,singlelinecheck=false}
    \caption{\textbf{Feature splitting as a covering problem.} For each temporal component reconstructed by an SAE on held-out GPT-2 activations, we choose and scale one of $N$ temporal decoder vectors to minimize reconstruction error. The largest relative error across examples (blue) decreases slowly with $N$. The dashed curve shows the scaling in Theorem~\ref{thm:line-covering}, and the purple line marks the estimated cutoff $t\sin\rho$.}
    \label{fig:real-line-cover}
\end{figure}

\subsection{Concept Analyses}
\label{apx:exp-interpretabiity}
In this section, we analyze the behavior of SASA groups by visualizing their activation profiles and code coordinates.

\subsubsection{Group Selection and Visualization}
\label{apx:group-selection}

For temporal and geographical examples, we screen groups using controlled prompts containing concept terms and check their selectivity on natural OpenWebText contexts. Each concept lexicon is substituted into sentence templates, retaining at most 800 prompts per family. Tokens are labeled by lexicon membership, with four-digit years also counted as temporal tokens. We rank groups by the AUROC of their code norms for separating concept tokens from background tokens. This selects GPT-2 Group 1473 for temporal concepts and Mistral Group 1570 for geographical concepts. Activation profiles show the group-code norm at each token, while geometry plots project the code vectors at the relevant terms.

The templates use \texttt{\{term\}} for an entry in the corresponding lexicon.

\noindent\emph{Geography templates}
\begin{quote}\small\ttfamily
I traveled to \{term\}.\\
She moved to \{term\} for work.\\
He lives in \{term\} now.\\
The city of \{term\} is famous.\\
The region around \{term\} is beautiful.\\
We drove across \{term\}.\\
The capital of \{term\} is well known.\\
The border of \{term\} is disputed.\\
Flights to \{term\} were delayed.\\
A map of \{term\} was published.\\
The river in \{term\} floods in spring.\\
The cuisine of \{term\} is popular.\\
\{term\} appears on the map.\\
The mountains in \{term\} are high.\\
The coastline of \{term\} is long.\\
The embassy in \{term\} reopened.
\end{quote}

\noindent\emph{Temporal templates}
\begin{quote}\small\ttfamily
The event happened in \{term\}.\\
The meeting is on \{term\}.\\
She will return in \{term\}.\\
They met every \{term\}.\\
He works best in the \{term\}.\\
The deadline is \{term\}.\\
We arrived \{term\}.\\
The festival is held in \{term\}.\\
The schedule changed by \{term\}.\\
She called me on \{term\}.\\
The report is due in \{term\}.\\
The appointment is \{term\}.\\
He studies each \{term\}.\\
We meet once a \{term\}.\\
The contract ends in \{term\}.\\
The anniversary is in \{term\}.
\end{quote}

\noindent\emph{Negative templates}
\begin{quote}\small\ttfamily
I bought a \{term\} at the store.\\
She found a \{term\} in the garden.\\
He repaired the \{term\} at home.\\
We cooked a \{term\} for dinner.\\
The \{term\} was on the table.\\
A \{term\} fell from the shelf.\\
He cleaned the \{term\} yesterday.\\
She carried a \{term\} to work.\\
The \{term\} needs a new battery.\\
They painted the \{term\} blue.
\end{quote}

For the circular representation of months in Figure~\ref{fig:real-manifolds}, we select units that activate on at least 20\% of the occurrences of some month name in Pile and on fewer than 1\% of background Pile tokens. A unit is one SAE latent or one SASA group. At each point on the ring, the color identifies the active unit whose decoded contribution, centered within each carrier sentence, has the largest projection along that point's direction in the three-dimensional PCA projection. The bar plots report each selected unit's share of the total activity of its dictionary's selected units on these month names, using scalar activations for the SAE and code norms for SASA.

\subsubsection{Temporal Concepts in GPT-2}
\label{apx:temporal}
In this subsection, we assess whether one SASA group can respond across temporal concepts while retaining distinctions among them. We study GPT-2 Group 1473 through its activation profiles and the geometry of its code coordinates.

For Group 1473 we use three fixed prompts containing days, months, years, and seasons and plot their activation profile in Figure~\ref{fig:temporal_profiles}.
\begin{quote}\small\ttfamily
On Monday, March 3, 1997, the committee met in private session.\\
On Friday, September 21, 2001, the city was marked by heavy rain.\\
In the summer of 2012, the team traveled to Europe for training.
\end{quote}

\begin{figure}[tbp]
\centering
\begin{subfigure}{0.95\linewidth}
\centering
\includegraphics[width=\linewidth]{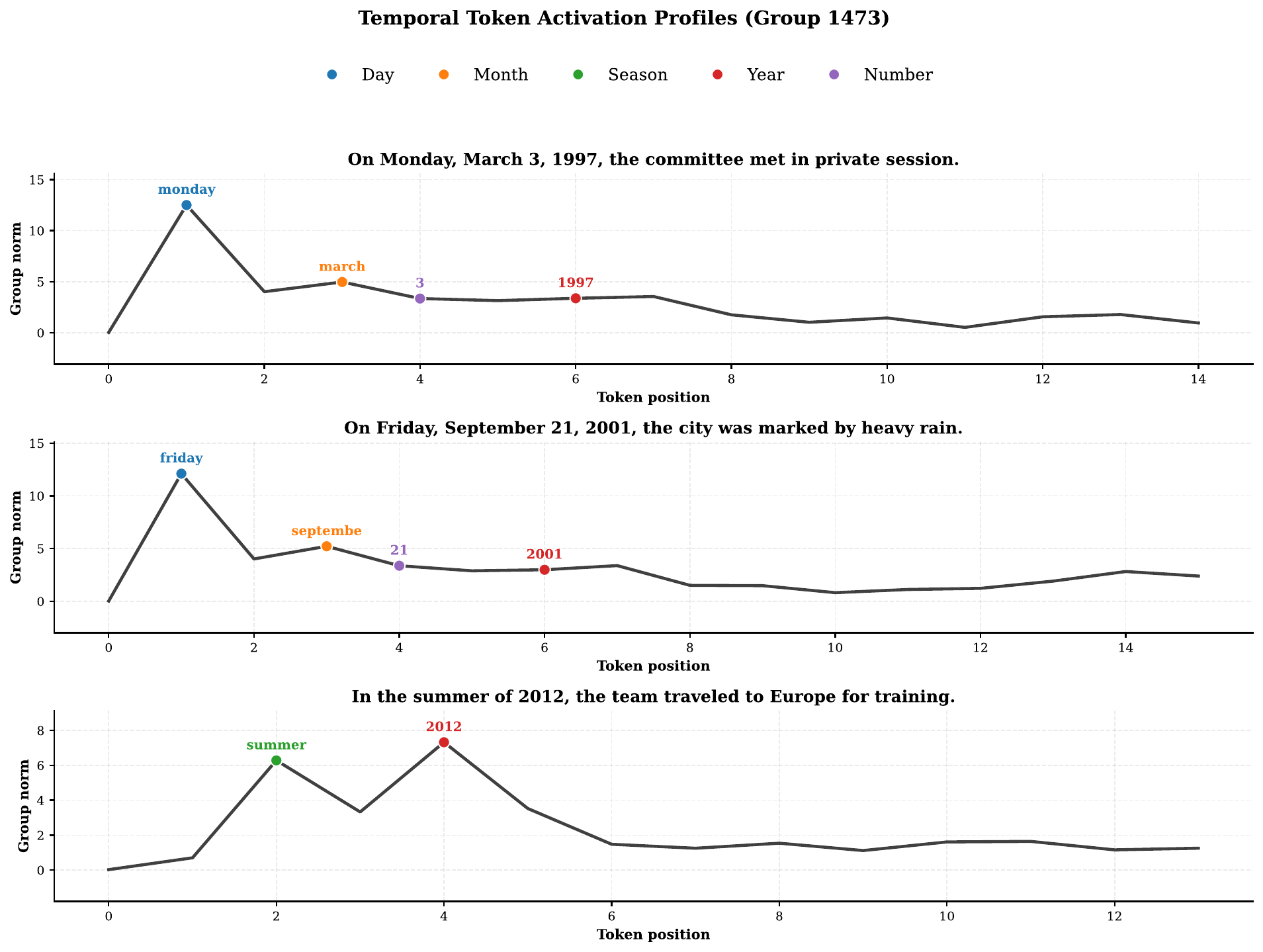}
\caption{Activation profiles across three temporal prompts.}
\label{fig:temporal_profiles}
\end{subfigure}
\par\medskip
\begin{subfigure}{0.62\linewidth}
\centering
\includegraphics[width=\linewidth]{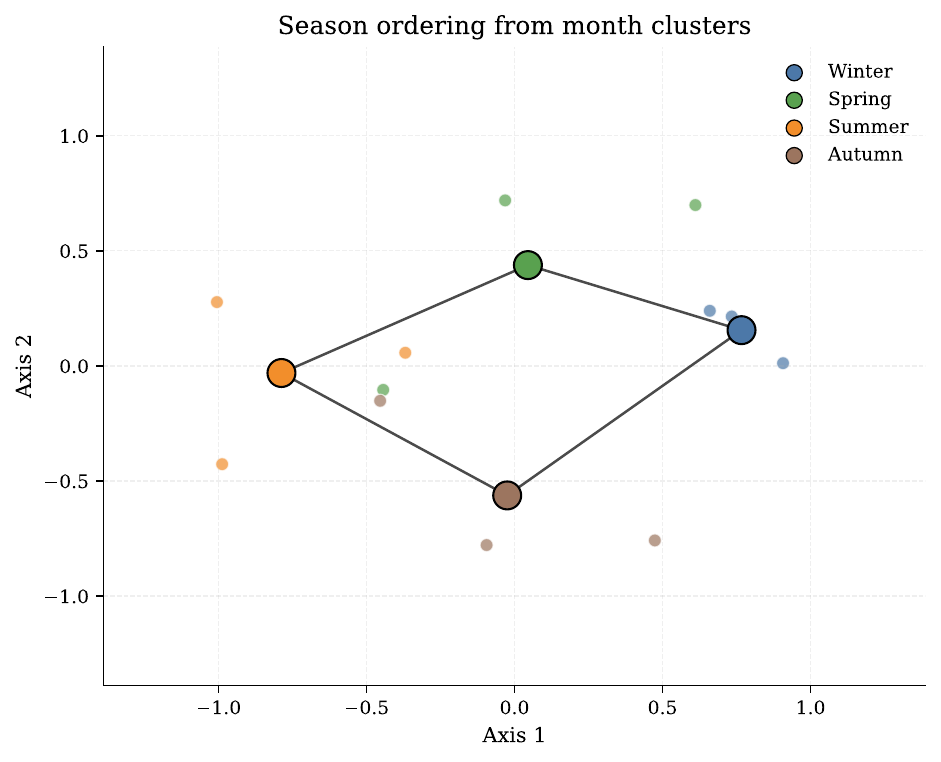}
\caption{Calendar-supervised circular projection.}
\label{fig:temporal_season_order}
\end{subfigure}
\caption{GPT-2 Group 1473 responds across temporal contexts. The circular visualization fits month centroids to calendar targets and then averages months into seasons.}
\label{fig:temporal-examples}
\end{figure}

Table~\ref{tab:temporal_group1473_owt} lists five OpenWebText contexts with the highest Group 1473 activations (group norms), highlighting the most active token.

\begin{table}[!h]
    \centering
    \caption{Activation values and corresponding OpenWebText prompts for GPT-2 SASA Group 1473.}
    \label{tab:temporal_group1473_owt}
    \small
    \begin{tabular}{@{}r p{0.78\linewidth}@{}}
        \toprule
        \textbf{Activation} & \textbf{Prompt} \\
        \midrule
        14.936 & Attorney General Dominic Grieve was presented with legal papers on \colorbox{yellow}{Monday} arguing that because there were no fingerprints on five items found with \\
        14.880 & state of Chihuahua returned to the United States on \colorbox{yellow}{Friday} night. The four-year-old remained in the custody \\
        14.703 & office. The draft resolution, circulated by Egypt on \colorbox{yellow}{Wednesday} night and originally slated for a vote Thursday, demands \\
        14.602 & President Trump made his debut at the United Nations on \colorbox{yellow}{Tuesday}, addressing the U.N. General Assembly at its annual \\
        14.506 & s president Uhuru Kenyatta said in a speech on \colorbox{yellow}{Tuesday} (Feb. 16) that his government is thinking about building \\
        \bottomrule
    \end{tabular}
\end{table}

We collect Group 1473 codes from OpenWebText at temporal terms, including weekday and month names, seasons, years, and abbreviations. For Figure~\ref{fig:temporal_scatter}, we retain codes with norm above $0.1$ and apply PCA to their standardized coordinates, coloring points by temporal category.

For Figure~\ref{fig:temporal_season_order}, we average group codes for each month and fit a supervised linear projection that maps these centroids toward evenly spaced points on the unit circle in calendar order. Season centroids are averages of their constituent months.

\subsubsection{Geographical Concepts in Mistral-7B}
\label{apx:geo}
The activation profiles of Mistral SASA Group 1570 in Figure~\ref{fig:geo_profiles} show responses to cities (``Chicago'', ``Paris''), countries (``US'', ``France''), and continents (``Europe'') within natural contexts.

We use three fixed prompts.
\begin{quote}\small\ttfamily
A storm delayed my flight to Chicago in the US.\\
The Toronto skyline in Canada is impressive.\\
Our team met in Paris in France before heading back from Europe.
\end{quote}

The PCA visualization in Figure~\ref{fig:geo_structure} shows how these categories are arranged within the group's coordinates.
We use the RAVEL city dataset, which comprises $3{,}552$ cities, along with their associated countries and continents.
For each city, we form a prompt of the form \texttt{City: \{city\}. Country: \{country\}. Continent: \{continent\}.}
We average each entity's group codes over its tokens and then across prompts, and apply 2D PCA to the standardized entity vectors.

\begin{figure}[!h]
    \centering
    \includegraphics[width=\linewidth]{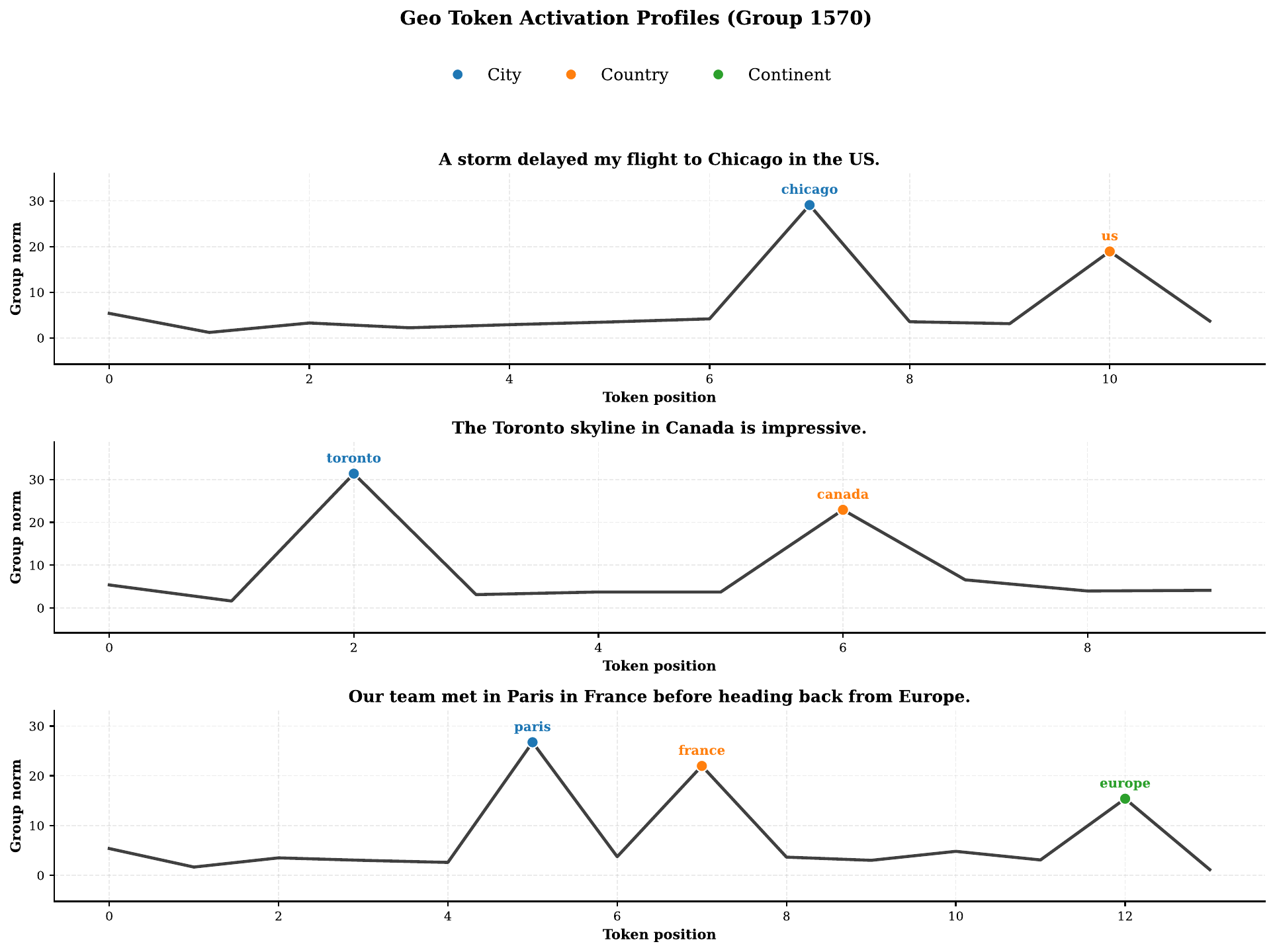}
    \caption{\textbf{Mistral SASA Group 1570 Activation Profiles.} The group consistently activates on geographical tokens.}
    \label{fig:geo_profiles}
\end{figure}

\begin{figure}[!h]
    \centering
    \includegraphics[width=0.65\linewidth]{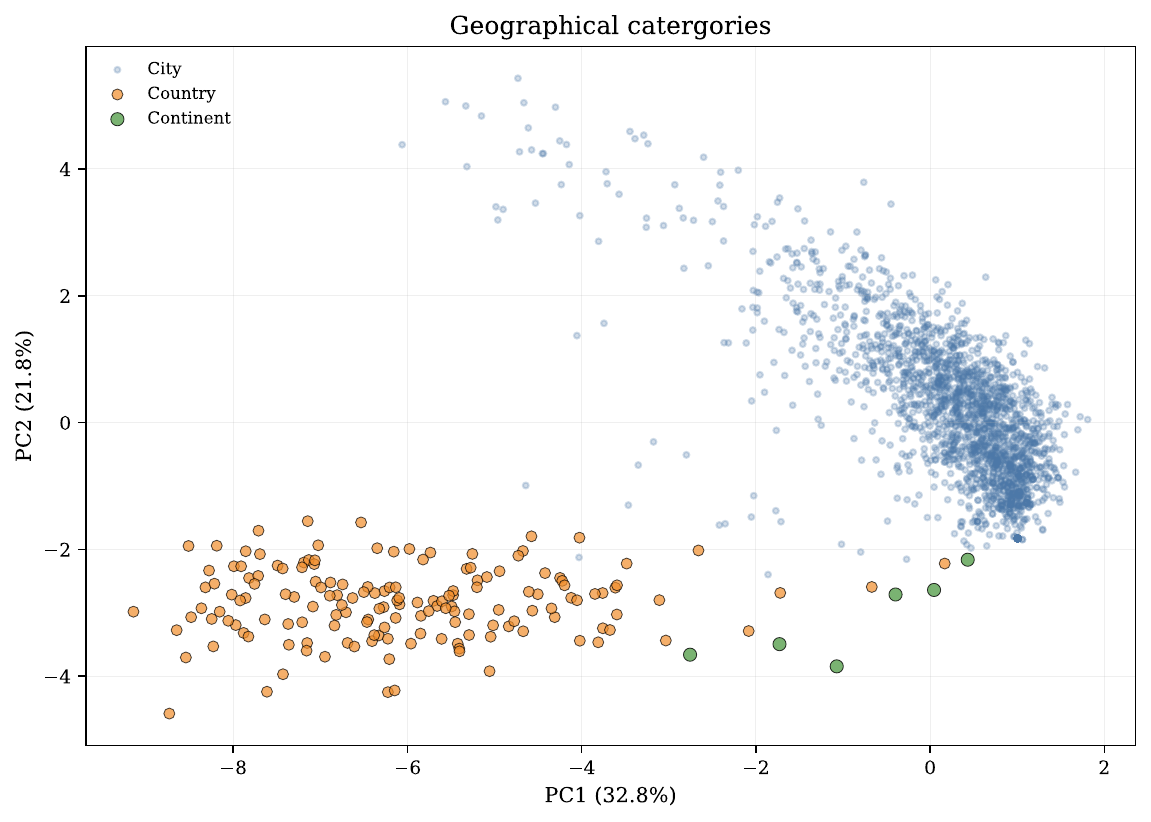}
    \caption{\textbf{Geometry of the Geographical Subspace.} A PCA projection of the latent activations in Mistral SASA Group 1570. Points represent cities (blue), countries (orange), and continents (green). Cities and countries occupy distinct regions of the projection.}
    \label{fig:geo_structure}
\end{figure}

\subsubsection{Sports Concepts}
\label{apx:additional-interp}

Figure~\ref{fig:apx-sports-topic} shows Group 1711 of the GPT-2 SASA, labeled by AutoInterp as \emph{titles and awards related to sports and competitions}. Within this subspace, 3D PCA separates sports (football, basketball, baseball, hockey, soccer), championships (Super Bowl, World Series, Stanley Cup, World Cup), and awards (MVP, Heisman, trophy) into three clusters.

\begin{figure}[!h]
    \centering
    \includegraphics[width=0.65\linewidth]{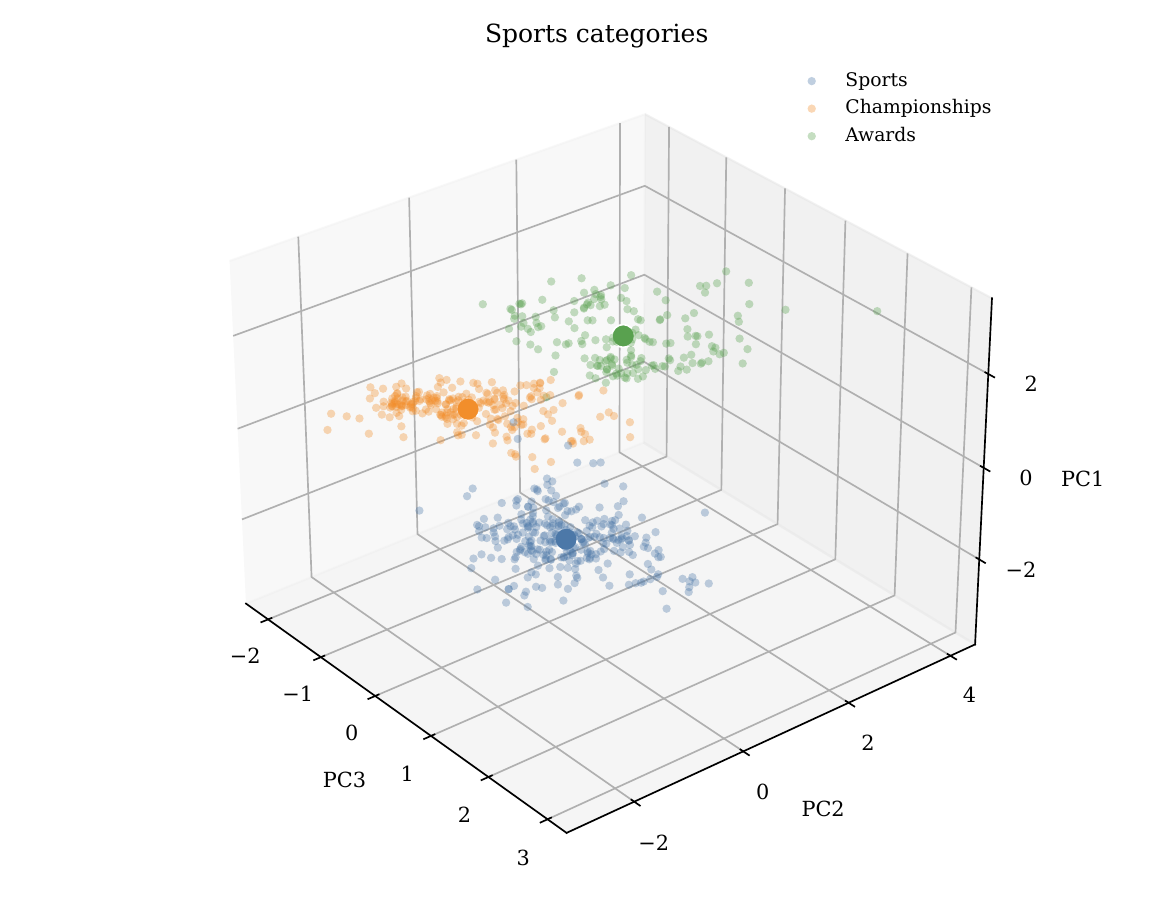}
    \caption{\textbf{GPT-2 SASA Group 1711 for Sports.} AutoInterp labels this group as \emph{titles and awards related to sports and competitions}. A 3D PCA view separates sports (e.g., football, hockey), championships (e.g., Super Bowl, World Cup), and awards (e.g., MVP, trophy). Colors give each term's category, and large markers its centroid.}
    \label{fig:apx-sports-topic}
\end{figure}

\section{Proofs of Theoretical Results}
\label{apx:theory}
This appendix collects the detailed proofs of all theorems, propositions, and lemmas stated or used in the main body. Appendix~\ref{apx:theory-necessity} develops the covering bounds and the $\ell_1$ risk analysis. Appendix~\ref{apx:theory-sasa} then proves single-block capacity and optimality for SASA. Finally, Appendix~\ref{apx:theory-sample-complexity} derives the sample complexity bounds for subspace estimation and for observing every atom in a fixed dictionary.

\subsection{Necessity of Subspace Learning}
\label{apx:theory-necessity}
We first establish the geometric tools needed to compare the feature slice with spans of at most $k$ decoder atoms. A bound on the surface measure near each span yields the covering lower bound in Theorem~\ref{thm:line-covering}, which we then extend to noisy superposed inputs in Proposition~\ref{prop:superposition}. We next study the $\ell_1$-regularized objective, beginning with the residual left by the optimal basis code and using its correlations with additional decoder directions to establish a strict decrease in population risk.

\subsubsection{Additional notation and geometric preliminaries}
\paragraph{Angular distance on sphere.} Before proving the main results, we recall the geometric tools used below. For a metric space $(\mathcal{X},\gamma)$ and a subset $\mathcal{T}\subseteq\mathcal{X}$, the $\varepsilon$-covering number is
\begin{equation}
\mathcal{N}(\mathcal{T},\gamma,\varepsilon)
:=\min\Big\{N:\exists \boldsymbol{x}_1,\dots,\boldsymbol{x}_N\in \mathcal{X}  \;\;\text{s.t.}\;\;\ \mathcal{T}\subseteq \bigcup_{i=1}^N B_{\gamma}(\boldsymbol{x}_i,\varepsilon)\Big\}.
\end{equation}
where $B_{\gamma}(\boldsymbol{x},\varepsilon):=\{\boldsymbol{y}\in\mathcal{X}:\gamma(\boldsymbol{x},\boldsymbol{y})\le \varepsilon\}$ \cite{Vershynin_2018}. For $\boldsymbol{u},\boldsymbol{v} \in \mathbb{S}^{d-1}$, we use the distance
\[
\gamma(\boldsymbol{u},\boldsymbol{v}) := \arccos(\vert\langle \boldsymbol{u},\boldsymbol{v}\rangle\vert) \in [0, \frac{\pi}{2}].
\]

Because $\gamma$ identifies antipodal directions, its balls consist of two opposite spherical caps. For $\rho\in(0,\frac{\pi}{2}]$ and $\boldsymbol{u}_0\in\mathbb{S}^{d-1}$, define
\[
B_{\gamma}(\boldsymbol{u}_0,\rho)
:= \{\boldsymbol{u}\in\mathbb{S}^{d-1}:\gamma(\boldsymbol{u},\boldsymbol{u}_0)\le \rho\}.
\]
For a fixed feature $i\in \mathcal{I}$, when working on the slice with radius $t$ inside a feature subspace $\mathcal{V}_i$, we use the scaled cap as defined 
\[
t \cdot\big(B_{\gamma}(\boldsymbol{u}_0,\rho)\cap \mathcal{V}_i\big)
:=\{t\boldsymbol{u}:\ \boldsymbol{u}\in B_{\gamma}(\boldsymbol{u}_0,\rho)\cap\mathcal{V}_i\}\subseteq \mathcal{M}_i(t).
\]
The inclusion is the condition in Assumption~\ref{ass:local-richness}.

\begin{lemma}[Subspace projection reduces distance]
\label{lem:proj}
For any $\boldsymbol{h}\in\mathcal{V}_i$ and any linear subspace $\boldsymbol{W}\subset\mathbb{R}^d$,
\[
  \operatorname{dist}\bigl(\boldsymbol{h},\boldsymbol{P}_{\mathcal{V}_i}\boldsymbol{W}\bigr)
  \le
  \operatorname{dist}(\boldsymbol{h},\boldsymbol{W}).
\]
\end{lemma}

\begin{proof}
Recall that $\mathcal{V}_i$ is the feature subspace and $\boldsymbol{P}_{\mathcal{V}_i}$ the orthogonal projector onto it. The claim is that projecting a competing subspace $\boldsymbol{W}$ onto $\mathcal{V}_i$ can only bring it closer to a point that already lies in $\mathcal{V}_i$. To see this, let $\boldsymbol{w}^*$ attain the distance from $\boldsymbol{h}$ to $\boldsymbol{W}$. Since $\boldsymbol{h}\in\mathcal{V}_i$, the decomposition $\boldsymbol{h}-\boldsymbol{w}^* = (\boldsymbol{h}-\boldsymbol{P}_{\mathcal{V}_i}\boldsymbol{w}^*) - \boldsymbol{P}_{\mathcal{V}_i^\perp}\boldsymbol{w}^*$ splits into orthogonal components lying in $\mathcal{V}_i$ and $\mathcal{V}_i^\perp$ respectively, so the Pythagorean theorem gives
\[
  \|\boldsymbol{h}-\boldsymbol{w}^*\|_2^2
  =
  \|\boldsymbol{h}-\boldsymbol{P}_{\mathcal{V}_i}\boldsymbol{w}^*\|_2^2
  +
  \|\boldsymbol{P}_{\mathcal{V}_i^\perp}\boldsymbol{w}^*\|_2^2
  \ge
  \operatorname{dist}\bigl(\boldsymbol{h},\boldsymbol{P}_{\mathcal{V}_i}\boldsymbol{W}\bigr)^2.
  \qedhere
\]
\end{proof}

\begin{lemma}[Tube volume on the sphere]
\label{lem:tube}
Let $\boldsymbol{W}\subset\mathbb{R}^{d_i}$ be a linear subspace of dimension at most $k$, where $1\le k < d_i$, and let $\beta\in(0,\pi/2)$. Define the angular tube
\[
  \mathrm{Tube}_\beta(\boldsymbol{W})
  =
  \bigl\{\boldsymbol{u}\in\mathbb{S}^{d_i-1}
  :\operatorname{dist}(\boldsymbol{u},\boldsymbol{W})
  \le\sin\beta\bigr\}.
\]
Let $\sigma_m$ denote the surface measure on $\mathbb{S}^{m}$. Then
\[
  \sigma_{d_i-1}\bigl(\mathrm{Tube}_\beta(\boldsymbol{W})\bigr)
  \le
  C_{\mathrm{T}}(d_i,k)\beta^{d_i-k},
  \qquad
  C_{\mathrm{T}}(d_i,k)
  =
  \frac{\sigma_{k-1}(\mathbb{S}^{k-1})\sigma_{d_i-k-1}(\mathbb{S}^{d_i-k-1})}{d_i-k}.
\]
\end{lemma}

\begin{proof}
Enlarging $\boldsymbol{W}$ to dimension $k$ can only enlarge its tube, so it suffices to consider this case. Without loss of generality, let $\boldsymbol{W}=\operatorname{span}\{\boldsymbol{e}_1,\dots,\boldsymbol{e}_k\}$. Parameterize $\boldsymbol{u}\in\mathbb{S}^{d_i-1}$ as $\boldsymbol{u}=\cos\theta\boldsymbol{w}+\sin\theta\boldsymbol{n}$ with $\boldsymbol{w}\in\mathbb{S}^{k-1}\subset \boldsymbol{W}$, $\boldsymbol{n}\in\mathbb{S}^{d_i-k-1}\subset \boldsymbol{W}^\perp$, and $\theta\in[0,\pi/2]$. The decomposition of the surface measure adapted to the splitting
$\mathbb{R}^{d_i}=\boldsymbol{W}\oplus\boldsymbol{W}^\perp$
\citep[Lemma~A.5.4]{dai2013approximation}, after writing the
$\boldsymbol{W}^\perp$ component in polar coordinates (with $m=k$, $d=d_i-k$), gives
\[
  d\sigma_{d_i-1}(\boldsymbol{u})
  =\cos^{k-1}\theta\,\sin^{d_i-k-1}\theta\,
   d\theta\,d\sigma_{k-1}(\boldsymbol{w})\,d\sigma_{d_i-k-1}(\boldsymbol{n}).
\]
The tube condition $\operatorname{dist}(\boldsymbol{u},\boldsymbol{W})=\sin\theta\le\sin\beta$ is equivalent to $\theta\le\beta$, the restriction $\beta\in(0,\pi/2)$ keeping $\sin$ monotone so the two are interchangeable. Integrating and using $\cos\theta\le 1$, $\sin\theta\le\theta$ on $[0,\pi/2]$,
\[
  \sigma_{d_i-1}\bigl(\mathrm{Tube}_\beta(\boldsymbol{W})\bigr)
  \le
  \sigma_{k-1}(\mathbb{S}^{k-1})\sigma_{d_i-k-1}(\mathbb{S}^{d_i-k-1})
  \int_0^\beta\theta^{d_i-k-1}d\theta
  =
  C_{\mathrm{T}}(d_i,k)\beta^{d_i-k}.
  \qedhere
\]
\end{proof}

\begin{lemma}[Cap volume positivity]
\label{lem:cap}
For $d_i\ge 2$, $\boldsymbol{u}_0\in\mathbb{S}^{d_i-1}$, and $\rho\in(0,\pi/2)$, there exists a constant $c_\rho=c_\rho(d_i)>0$ such that
\[
  \sigma_{d_i-1}\bigl(B_\gamma(\boldsymbol{u}_0,\rho)\bigr)
  \ge
  c_\rho.
\]
\end{lemma}

\begin{proof}
The projective ball $B_\gamma(\boldsymbol{u}_0,\rho)$ consists of two antipodal spherical caps. By rotational invariance, $\sigma_{d_i-1}(B_\gamma(\boldsymbol{u}_0,\rho)) = 2\sigma_{d_i-2}(\mathbb{S}^{d_i-2})\int_0^\rho\sin^{d_i-2}\theta\,d\theta$, which is positive since the integrand is continuous and strictly positive on $(0,\rho)$.
\end{proof}

\subsubsection{Proof of Theorem~\ref{thm:line-covering}}
\label{apx:proof-covering}

\begin{proof}
Recall that $\mathcal{M}_i(t)$ is the noiseless single-feature slice, contained in the radius-$t$ sphere $\{t\boldsymbol{u}:\boldsymbol{u}\in\mathbb{S}^{d-1}\cap\mathcal{V}_i\}$ inside the feature subspace $\mathcal{V}_i$, and that the goal is to lower bound the number $N$ of unit decoder directions needed to $\varepsilon$-cover it with $k$-dimensional spans. Fix $1\le k<d_i$, set $\beta=\arcsin(\varepsilon/t)\in(0,\rho)$, where the hypothesis $\varepsilon\in(0,t\sin\rho)$ keeps the argument in range so that $\beta<\rho<\pi/2$, and let $\boldsymbol{d}_1,\dots,\boldsymbol{d}_N\in\mathbb{S}^{d-1}$ account for $L_i^{(k)}(\varepsilon)=N$. By Definition~\ref{def:k-cover}, every $\boldsymbol{h}=t\boldsymbol{u}\in\mathcal{M}_i(t)$ admits a subset $J\subseteq[N]$ with $|J|\le k$ satisfying
\[
  \operatorname{dist}\bigl(\boldsymbol{h},\operatorname{span}\{\boldsymbol{d}_j\}_{j\in J}\bigr)
  \le
  \varepsilon.
\]

Lemma~\ref{lem:proj} projects this guarantee into $\mathcal{V}_i$, so that with $\widetilde{\boldsymbol{W}}_J = \boldsymbol{P}_{\mathcal{V}_i}\operatorname{span}\{\boldsymbol{d}_j\}_{j\in J}$ and $\dim(\widetilde{\boldsymbol{W}}_J)\le k$, every $\boldsymbol{u}\in B_\gamma(\boldsymbol{u}_0,\rho)\cap\mathcal{V}_i$ lies in one of the tubes $\mathrm{Tube}_\beta(\widetilde{\boldsymbol{W}}_J)$. The number of nonempty subsets of $[N]$ of size at most $k$ is at most
\[
  \sum_{j=1}^{k}\binom{N}{j}
  \le
  kN^k.
\]
Each tube has surface measure at most $C_{\mathrm{T}}(d_i,k)\beta^{d_i-k}$ by Lemma~\ref{lem:tube}. The union of all tubes must cover $B_\gamma(\boldsymbol{u}_0,\rho)\cap\mathcal{V}_i$, so Lemma~\ref{lem:cap} gives
\[
  0
  <
  c_\rho
  \le
  kN^k\cdot C_{\mathrm{T}}(d_i,k)\cdot\beta^{d_i-k}.
\]

Now, rearranging and using $\sin\beta\ge\frac{2}{\pi}\beta$ on $(0,\frac{\pi}{2})$, which gives $\beta\le\frac{\pi\varepsilon}{2t}$,
\[
  N^k
  \ge
  \frac{c_\rho}{kC_{\mathrm{T}}} \cdot \beta^{-(d_i-k)}
  \ge
  \frac{c_\rho}{kC_{\mathrm{T}}}
  \left(\frac{2}{\pi}\right)^{d_i-k}
  \left(\frac{t}{\varepsilon}\right)^{d_i-k}.
\]
Taking $k$-th roots yields $N\ge C(d_i,k,\rho)(t/\varepsilon)^{(d_i-k)/k}$ with $C(d_i,k,\rho) = (c_\rho/kC_{\mathrm{T}})^{1/k}(2/\pi)^{(d_i-k)/k}$.

When $k\ge d_i$, any orthonormal basis of $\mathcal{V}_i$ spans the full subspace, achieving $\operatorname{dist}(\boldsymbol{h},\mathcal{V}_i)=0$ for all $\boldsymbol{h}\in\mathcal{M}_i(t)$ with $N=d_i$. Fewer than $d_i$ columns have a projected span of dimension below $d_i$ and cannot contain the positive-measure cap, so $L_i^{(d_i)}(0)=d_i$.
\end{proof}

\subsubsection{Proof of Proposition~\ref{prop:superposition}}
\label{apx:proof-super}

\begin{proof}
For each slice point, choose an admissible superposed input containing it, as assumed in the proposition. Under Hypothesis~\ref{ass:mdsuperposition}, this gives $\boldsymbol{h}=\boldsymbol{V}_i\boldsymbol{z}_i+\boldsymbol{\delta}$ with at most $s$ active features and $\boldsymbol{\delta}=\sum_{j\neq i}\boldsymbol{V}_j\boldsymbol{z}_j+\boldsymbol{\xi}$. By Lemma~\ref{lem:proj}, the count in Theorem~\ref{thm:line-covering} depends only on the projected spans $\widetilde{\boldsymbol{W}}_J=\boldsymbol{P}_{\mathcal{V}_i}\operatorname{span}\{\boldsymbol{d}_j\}_{j\in J}$, so the slice point $\boldsymbol{V}_i\boldsymbol{z}_i$ is perturbed only by the in-subspace component $\boldsymbol{P}_{\mathcal{V}_i}\boldsymbol{\delta}$, which we bound first.

The $\mu$-coherence of Hypothesis~\ref{ass:mdsuperposition} is the operator-norm bound $\|\boldsymbol{V}_i^\top\boldsymbol{V}_j\|_2\le\mu$ for $j\neq i$. Because $\boldsymbol{V}_i$ has orthonormal columns, each off-feature satisfies
\[
  \|\boldsymbol{P}_{\mathcal{V}_i}\boldsymbol{V}_j\boldsymbol{z}_j\|_2
  =\|\boldsymbol{V}_i\boldsymbol{V}_i^\top\boldsymbol{V}_j\boldsymbol{z}_j\|_2
  =\|\boldsymbol{V}_i^\top\boldsymbol{V}_j\boldsymbol{z}_j\|_2
  \le\|\boldsymbol{V}_i^\top\boldsymbol{V}_j\|_2\,\|\boldsymbol{z}_j\|_2
  \le\mu t,
\]
where the last step uses $\|\boldsymbol{z}_j\|_2\le t$. Summing over the at most $s-1$ active off-features and using $\|\boldsymbol{P}_{\mathcal{V}_i}\boldsymbol{\xi}\|_2\le\|\boldsymbol{\xi}\|_2\le\eta$ gives
\[
  \|\boldsymbol{P}_{\mathcal{V}_i}\boldsymbol{\delta}\|_2
  \le
  (s-1)\mu t+\eta
  \le
  \mu s t+\eta.
\]

Now fix a subset $J\subseteq[N]$ with $|J|\le k$ and a point $\boldsymbol{p}\in\operatorname{span}\{\boldsymbol{d}_j\}_{j\in J}$ with $\|\boldsymbol{h}-\boldsymbol{p}\|_2\le\varepsilon$. Then $\boldsymbol{P}_{\mathcal{V}_i}\boldsymbol{p}\in\widetilde{\boldsymbol{W}}_J$ and $\dim\widetilde{\boldsymbol{W}}_J\le k$. Since $\boldsymbol{V}_i\boldsymbol{z}_i=\boldsymbol{P}_{\mathcal{V}_i}\boldsymbol{h}-\boldsymbol{P}_{\mathcal{V}_i}\boldsymbol{\delta}$,
\[
  \operatorname{dist}\bigl(\boldsymbol{V}_i\boldsymbol{z}_i,\widetilde{\boldsymbol{W}}_J\bigr)
  \le
  \|\boldsymbol{P}_{\mathcal{V}_i}(\boldsymbol{h}-\boldsymbol{p})-\boldsymbol{P}_{\mathcal{V}_i}\boldsymbol{\delta}\|_2
  \le
  \|\boldsymbol{h}-\boldsymbol{p}\|_2+\|\boldsymbol{P}_{\mathcal{V}_i}\boldsymbol{\delta}\|_2
  \le
  \varepsilon+\mu s t+\eta.
\]
The projected spans thus $k$-cover every slice point $\boldsymbol{V}_i\boldsymbol{z}_i\in\mathcal{M}_i(t)$ to error $\varepsilon'=\varepsilon+\mu st+\eta$. This is the in-$\mathcal{V}_i$ covering counted in the proof of Theorem~\ref{thm:line-covering}, so that argument applies with effective error $\varepsilon'=\varepsilon+\mu st+\eta$ whenever $\varepsilon'<t\sin\rho$.
\end{proof}

\subsubsection{Basis instability}
\label{apx:instability}

Throughout this part we work with the decoder risk $R(\boldsymbol{D})$ induced by the standard SAE objective. For a fixed decoder $\boldsymbol{D}$ the code is chosen to minimize the inner $\ell_1$-regularized reconstruction objective $\tfrac12\|\boldsymbol{h}-\boldsymbol{D}\boldsymbol{a}(\boldsymbol{h})\|_2^2+\lambda\|\boldsymbol{a}(\boldsymbol{h})\|_1$, and $R(\boldsymbol{D})$ averages this best-code value over $\boldsymbol{h}$ drawn uniformly from the full radius-$t$ sphere in $\mathcal{V}_i$, as in Section~\ref{sec:instability}. The goal is to show that the true basis decoder, padded with duplicate columns, is not a local minimizer of this risk. We begin by computing the optimal code, and the residual it leaves, when the decoder is exactly the feature basis.

\begin{lemma}[Optimal basis code and residual]
\label{lem:residual}
At $\boldsymbol{D}_{\mathrm{basis}}=[\boldsymbol{v}_1,\dots,\boldsymbol{v}_{d_i}]$, the $\ell_1$-regularized reconstruction objective
\[
  \min_{\boldsymbol{a}\in\mathbb{R}^{d_i}}
  \tfrac{1}{2}\|\boldsymbol{h}-\boldsymbol{D}_{\mathrm{basis}}\boldsymbol{a}\|_2^2
  +\lambda\|\boldsymbol{a}\|_1
\]
for $\boldsymbol{h}=t\boldsymbol{u}$ with $\boldsymbol{u}=\sum_{k=1}^{d_i}u_k\boldsymbol{v}_k\in\mathbb{S}^{d-1}\cap\mathcal{V}_i$ has the unique solution $\boldsymbol{a}^\star(\boldsymbol{h})$ with coordinates
\[
  a_k^\star
  =
  \operatorname{sign}(u_k)\bigl(t|u_k|-\lambda\bigr)_+,
  \qquad k\in[d_i].
\]
Let $\boldsymbol{r}(\boldsymbol{h})=\boldsymbol{h}-\boldsymbol{D}_{\mathrm{basis}}\boldsymbol{a}^\star(\boldsymbol{h})$. On the set $\mathcal{A}=\{\boldsymbol{u}\in\mathbb{S}^{d-1}\cap\mathcal{V}_i:t|u_k|>\lambda\;\forall k\}$, all coefficients of the optimal basis code are nonzero and the residual satisfies
\[
   \boldsymbol{r}(\boldsymbol{h})
  =
  \lambda\sum_{k=1}^{d_i}\operatorname{sign}(u_k)\boldsymbol{v}_k,
  \qquad
  \| \boldsymbol{r}(\boldsymbol{h})\|_2 = \lambda\sqrt{d_i}.
\]
\end{lemma}

\begin{proof}
Since $\boldsymbol{h}=t\sum_{k=1}^{d_i}u_k\boldsymbol{v}_k$ and the basis is orthonormal, for any code $\boldsymbol{a}\in\mathbb{R}^{d_i}$ we have
\[
\begin{aligned}
\|\boldsymbol{h}-\boldsymbol{D}_{\mathrm{basis}}\boldsymbol{a}\|_2^2
&=\left\|\sum_{k=1}^{d_i}(tu_k-a_k)\boldsymbol{v}_k\right\|_2^2\\
&=\sum_{k=1}^{d_i}(tu_k-a_k)^2.
\end{aligned}
\]
The objective therefore separates into scalar functions
\[
\begin{aligned}
\tfrac12\|\boldsymbol{h}-\boldsymbol{D}_{\mathrm{basis}}\boldsymbol{a}\|_2^2+\lambda\|\boldsymbol{a}\|_1
&=\sum_{k=1}^{d_i}\phi_k(a_k),\\
\phi_k(a_k)&=\tfrac12(tu_k-a_k)^2+\lambda|a_k|.
\end{aligned}
\]
Each $\phi_k$ is strictly convex, so its unique minimizer is characterized by
\[
0\in\partial\phi_k(a_k)
=a_k-tu_k+\lambda\partial|a_k|,
\qquad
\partial|a_k|=
\begin{cases}
\{1\}, & a_k>0,\\
[-1,1], & a_k=0,\\
\{-1\}, & a_k<0.
\end{cases}
\]
A positive solution must satisfy $a_k=tu_k-\lambda>0$, which is possible exactly when $tu_k>\lambda$. A negative solution must satisfy $a_k=tu_k+\lambda<0$, which is possible exactly when $tu_k<-\lambda$. At $a_k=0$, the optimality condition becomes $tu_k\in\lambda[-1,1]$, or $t|u_k|\le\lambda$. Consequently, the unique minimizer is the soft-thresholding rule \citep{Tibshirani_1996},
\[
\begin{aligned}
a_k^\star
&=\begin{cases}
tu_k-\lambda, & tu_k>\lambda,\\
0, & t|u_k|\le\lambda,\\
tu_k+\lambda, & tu_k<-\lambda
\end{cases}\\
&=\operatorname{sign}(u_k)(t|u_k|-\lambda)_+.
\end{aligned}
\]
The unique minimizers of the scalar terms together give the unique optimal basis code.

The same cases yield the residual coordinate
\[
tu_k-a_k^\star
=\begin{cases}
\lambda, & tu_k>\lambda,\\
tu_k, & t|u_k|\le\lambda,\\
-\lambda, & tu_k<-\lambda.
\end{cases}
\]
For $\boldsymbol{u}\in\mathcal{A}$, every coordinate satisfies $t|u_k|>\lambda$, so $tu_k-a_k^\star=\lambda\operatorname{sign}(u_k)$. Therefore
\[
\begin{aligned}
\boldsymbol{r}(\boldsymbol{h})
&=\boldsymbol{h}-\boldsymbol{D}_{\mathrm{basis}}\boldsymbol{a}^\star(\boldsymbol{h})
=\sum_{k=1}^{d_i}(tu_k-a_k^\star)\boldsymbol{v}_k\\
&=\lambda\sum_{k=1}^{d_i}\operatorname{sign}(u_k)\boldsymbol{v}_k.
\end{aligned}
\]
Finally, orthonormality and $\operatorname{sign}(u_k)^2=1$ on $\mathcal{A}$ give
\[
\begin{aligned}
\|\boldsymbol{r}(\boldsymbol{h})\|_2^2
&=\lambda^2\sum_{k,\ell=1}^{d_i}
  \operatorname{sign}(u_k)\operatorname{sign}(u_\ell)
  \boldsymbol{v}_k^\top\boldsymbol{v}_\ell\\
&=\lambda^2\sum_{k=1}^{d_i}\operatorname{sign}(u_k)^2
=d_i\lambda^2.
\end{aligned}
\]
Taking square roots yields $\|\boldsymbol{r}(\boldsymbol{h})\|_2=\lambda\sqrt{d_i}$.
\end{proof}

\paragraph{Proof of Theorem~\ref{thm:violation}.}

\begin{proof}
Use the optimal basis code $\boldsymbol{a}^\star(\boldsymbol{h})$ and residual $\boldsymbol{r}(\boldsymbol{h})$ from Lemma~\ref{lem:residual}. After adjoining $\widetilde{\boldsymbol{d}}$ to the basis, keep the basis coefficients fixed and denote the coefficient of the new column by $a_{d_i+1}$. The loss of this code is
\[
\phi_{\boldsymbol{h}}(a_{d_i+1})
=\tfrac12\|\boldsymbol{r}(\boldsymbol{h})-a_{d_i+1}\widetilde{\boldsymbol{d}}\|_2^2
+\lambda\|\boldsymbol{a}^\star(\boldsymbol{h})\|_1+\lambda|a_{d_i+1}|.
\]
Its subdifferential at $a_{d_i+1}=0$ is
\[
\partial\phi_{\boldsymbol{h}}(0)
=-\widetilde{\boldsymbol{d}}^\top \boldsymbol{r}(\boldsymbol{h})+\lambda[-1,1].
\]
Thus the zero coefficient satisfies the stationarity condition precisely when
\[
0\in\partial\phi_{\boldsymbol{h}}(0)
\quad\Longleftrightarrow\quad
\big|\widetilde{\boldsymbol{d}}^\top \boldsymbol{r}(\boldsymbol{h})\big|\le\lambda,
\]
which is the inactivity certificate in the statement.

Write $\widetilde{\boldsymbol{d}}=\sum_{k=1}^{d_i}c_k\boldsymbol{v}_k$ and let $S=\{k:c_k\ne0\}$. Since $\|\boldsymbol{c}\|_2=1$ and $\widetilde{\boldsymbol{d}}$ is not a signed basis vector, $|S|\ge2$. Consequently,
\[
\|\boldsymbol{c}\|_1^2
=\sum_{k=1}^{d_i}c_k^2+2\sum_{k<\ell}|c_kc_\ell|
>\sum_{k=1}^{d_i}c_k^2=1.
\]
For $\boldsymbol{u}\in\mathcal{U}_{\boldsymbol{c}}\subseteq\mathcal{A}$, Lemma~\ref{lem:residual} and orthonormality of the basis give
\[
\begin{aligned}
\widetilde{\boldsymbol{d}}^\top \boldsymbol{r}(t\boldsymbol{u})
&=\Bigl(\sum_{k=1}^{d_i}c_k\boldsymbol{v}_k\Bigr)^\top
  \Bigl(\lambda\sum_{\ell=1}^{d_i}\operatorname{sign}(u_\ell)\boldsymbol{v}_\ell\Bigr)\\
&=\lambda\sum_{k\in S}c_k\operatorname{sign}(u_k)
=\lambda\sum_{k\in S}|c_k|
=\lambda\|\boldsymbol{c}\|_1
>\lambda.
\end{aligned}
\]
The third equality uses the sign constraints defining $\mathcal{U}_{\boldsymbol{c}}$, so the certificate is violated throughout this set.

It remains to establish that this set has positive measure under uniform sampling. Set
\[
\boldsymbol{u}_0=\frac1{\sqrt{d_i}}\sum_{k=1}^{d_i}\boldsymbol{v}_k,
\qquad
\delta=\frac12\left(\frac1{\sqrt{d_i}}-\frac\lambda t\right)>0.
\]
Since $\boldsymbol{v}_k^\top\boldsymbol{u}_0=1/\sqrt{d_i}$ and $\|\boldsymbol{v}_k\|_2=1$, Cauchy--Schwarz gives
\[
\begin{aligned}
\left|u_k-\frac1{\sqrt{d_i}}\right|
&=\left|\boldsymbol{v}_k^\top(\boldsymbol{u}-\boldsymbol{u}_0)\right|\\
&\le\|\boldsymbol{v}_k\|_2\|\boldsymbol{u}-\boldsymbol{u}_0\|_2
=\|\boldsymbol{u}-\boldsymbol{u}_0\|_2.
\end{aligned}
\]
Thus, for any $\boldsymbol{u}\in\mathbb{S}^{d-1}\cap\mathcal{V}_i$ with $\|\boldsymbol{u}-\boldsymbol{u}_0\|_2<\delta$, every coordinate satisfies
\[
u_k
\ge\frac1{\sqrt{d_i}}-\|\boldsymbol{u}-\boldsymbol{u}_0\|_2
>\frac1{\sqrt{d_i}}-\delta
 =\frac12\left(\frac1{\sqrt{d_i}}+\frac\lambda t\right)
>\frac\lambda t,
\]
where the last inequality uses $\lambda/t<1/\sqrt{d_i}$. Hence every direction in this spherical neighborhood belongs to $\mathcal{A}$. The neighborhood is a nonempty open patch of the sphere in $\mathcal{V}_i$, so it has positive surface measure and therefore positive probability under the uniform distribution.

Both $\mathcal{A}$ and the uniform measure are invariant under independent sign flips of the basis coordinates. The $2^{d_i}$ possible patterns of coordinate signs therefore partition $\mathcal{A}$ into regions of equal measure. The set $\mathcal{U}_{\boldsymbol{c}}$ fixes the signs of the $|S|$ nonzero coordinates of $\boldsymbol{c}$ and allows all $2^{d_i-|S|}$ choices of the remaining signs, so it is the union of that many pieces of equal measure. It follows that
\[
\Pr(\boldsymbol{u}\in\mathcal{U}_{\boldsymbol{c}})
=2^{d_i-|S|}\,2^{-d_i}\Pr(\boldsymbol{u}\in\mathcal{A})
=2^{-|S|}\Pr(\boldsymbol{u}\in\mathcal{A})
>0.
\]
Thus $\mathcal{U}_{\boldsymbol{c}}$ has positive probability, and every point in this set satisfies $\widetilde{\boldsymbol{d}}^\top \boldsymbol{r}(t\boldsymbol{u})=\lambda\|\boldsymbol{c}\|_1>\lambda$. The inactivity certificate is therefore violated on a set of positive measure.
\end{proof}

\paragraph{Proof of Proposition~\ref{prop:reduction}}

\begin{proof}
We turn the certificate violation of Theorem~\ref{thm:violation} into a strictly lower risk by exhibiting a feasible code that activates the extra column. First recall that the population risk in Eq.~\ref{eq:risk} is the expectation of the minimum loss for each input. Denoting this inner minimum by $L^\star(\boldsymbol{h};\boldsymbol{D})$, we have
\[
\begin{aligned}
L^\star(\boldsymbol{h};\boldsymbol{D})
&:=\min_{\boldsymbol{a}}\left\{\tfrac12\|\boldsymbol{h}-\boldsymbol{D}\boldsymbol{a}\|_2^2+\lambda\|\boldsymbol{a}\|_1\right\},\\
R(\boldsymbol{D})
&=\mathbb{E}_{\boldsymbol{h}}\big[L^\star(\boldsymbol{h};\boldsymbol{D})\big].
\end{aligned}
\]
For the basis decoder, let $L_{\mathrm{basis}}(\boldsymbol{h})=L^\star(\boldsymbol{h};\boldsymbol{D}_{\mathrm{basis}})$.
For $\boldsymbol{D}=[\boldsymbol{D}_{\mathrm{basis}},\widetilde{\boldsymbol{d}}]$, use the optimal basis code $\boldsymbol{a}^\star(\boldsymbol{h})$ from Lemma~\ref{lem:residual} and set the added coefficient to zero. For every input, this feasible code gives
\[
\begin{aligned}
L^\star(\boldsymbol{h};\boldsymbol{D})
&\le\tfrac12\left\|\boldsymbol{h}-\boldsymbol{D}
\begin{bmatrix}\boldsymbol{a}^\star(\boldsymbol{h})\\0\end{bmatrix}\right\|_2^2
+\lambda\left\|\begin{bmatrix}\boldsymbol{a}^\star(\boldsymbol{h})\\0\end{bmatrix}\right\|_1\\
&=\tfrac12\|\boldsymbol{r}(\boldsymbol{h})\|_2^2
+\lambda\|\boldsymbol{a}^\star(\boldsymbol{h})\|_1
=L_{\mathrm{basis}}(\boldsymbol{h}).
\end{aligned}
\]

For $\boldsymbol{u}\in\mathcal{U}_{\boldsymbol{c}}$, Theorem~\ref{thm:violation} gives $\widetilde{\boldsymbol{d}}^\top\boldsymbol{r}(\boldsymbol{h})=\lambda\|\boldsymbol{c}\|_1>\lambda$. Keep the basis coefficients fixed and give the added column a coefficient $a_{d_i+1}\ge0$. Since $\|\widetilde{\boldsymbol{d}}\|_2=1$, the minimized loss satisfies
\[
\begin{aligned}
L^\star(\boldsymbol{h};\boldsymbol{D})
&\le\tfrac12\|\boldsymbol{r}(\boldsymbol{h})-a_{d_i+1}\widetilde{\boldsymbol{d}}\|_2^2
+\lambda\|\boldsymbol{a}^\star(\boldsymbol{h})\|_1+\lambda a_{d_i+1}\\
&=L_{\mathrm{basis}}(\boldsymbol{h})
-a_{d_i+1}\widetilde{\boldsymbol{d}}^\top\boldsymbol{r}(\boldsymbol{h})
+\tfrac12a_{d_i+1}^2+\lambda a_{d_i+1}\\
&=L_{\mathrm{basis}}(\boldsymbol{h})
+\tfrac12a_{d_i+1}^2-\lambda\bigl(\|\boldsymbol{c}\|_1-1\bigr)a_{d_i+1}\\
&=L_{\mathrm{basis}}(\boldsymbol{h})
+\tfrac12\bigl(a_{d_i+1}-\lambda(\|\boldsymbol{c}\|_1-1)\bigr)^2
-\tfrac{\lambda^2}{2}\bigl(\|\boldsymbol{c}\|_1-1\bigr)^2.
\end{aligned}
\]
Taking $a_{d_i+1}=\lambda(\|\boldsymbol{c}\|_1-1)>0$ gives
\[
L_{\mathrm{basis}}(\boldsymbol{h})-L^\star(\boldsymbol{h};\boldsymbol{D})
\ge\tfrac{\lambda^2}{2}\bigl(\|\boldsymbol{c}\|_1-1\bigr)^2>0
\qquad(\boldsymbol{u}\in\mathcal{U}_{\boldsymbol{c}}).
\]
The loss difference is nonnegative for every other input by the first comparison. Averaging therefore yields
\[
\begin{aligned}
R(\boldsymbol{D}_{\mathrm{basis}})-R(\boldsymbol{D})
&=\mathbb{E}_{\boldsymbol{h}}\big[
L_{\mathrm{basis}}(\boldsymbol{h})-L^\star(\boldsymbol{h};\boldsymbol{D})\big]\\
&\ge\tfrac{\lambda^2}{2}\bigl(\|\boldsymbol{c}\|_1-1\bigr)^2
\Pr(\boldsymbol{u}\in\mathcal{U}_{\boldsymbol{c}})>0.
\end{aligned}
\]
\end{proof}

Proposition~\ref{prop:reduction} shows that adding a non-basis direction to the basis decoder strictly lowers the population risk. We now ask whether the risk can also be lowered by an arbitrarily small perturbation at a fixed dictionary width. For a dictionary consisting of the feature basis and additional copies of its vectors, we keep the original basis fixed and rotate one copy toward the normalized sum of the basis vectors. Theorem~\ref{thm:instability} shows that arbitrarily small rotations along this path strictly lower the risk, so such a dictionary, although contains basis, is not a local minimizer.
\begin{theorem}[Basis instability]
\label{thm:instability}
Let $2 \le d_i < m$ and $0<\lambda<t/\sqrt{d_i}$. For any $\boldsymbol{D}_0 \in \mathcal{D}_m$ whose first $d_i$ columns form an orthonormal basis of $\mathcal{V}_i$ and whose remaining columns are copies of basis vectors, including $\boldsymbol{d}_j=\boldsymbol{v}_1$ for some $j>d_i$, set $\widetilde{\boldsymbol{d}} = \frac{1}{\sqrt{d_i}}\sum_{k=1}^{d_i}\boldsymbol{v}_k$. For $0\le\alpha\le\pi/2$, define $\boldsymbol{d}_j(\alpha)$ by normalizing $\cos\alpha\,\boldsymbol{v}_1+\sin\alpha\,\widetilde{\boldsymbol{d}}$ to unit norm and let $\boldsymbol{D}(\alpha)$ be $\boldsymbol{D}_0$ with column $j$ replaced by $\boldsymbol{d}_j(\alpha)$. Write $f(\alpha)=\boldsymbol{d}_j(\alpha)^\top \boldsymbol{r}(\boldsymbol{h})/\lambda$ for $\boldsymbol{h}=t\boldsymbol{u}$ on the region with positive basis coordinates $\mathcal{U}^+:=\{\boldsymbol{u}\in\mathcal{A}:u_k>0\;\forall k\}$, with $f$ increasing from $f(0)=1$ to $\sqrt{d_i}>1$ as $\alpha$ increases from $0$ to $\pi/2$. Then for every $0<\alpha\le\pi/2$,
\begin{equation}\label{eq:descent}
R(\boldsymbol{D}_0) - R(\boldsymbol{D}(\alpha)) \;\ge\; 2^{-(d_i+1)}\lambda^2\Pr(\boldsymbol{u}\in\mathcal{A})\bigl(f(\alpha)-1\bigr)^2
  \;>\;0.
\end{equation}
In particular, there is a continuous path from $\boldsymbol{D}_0$ to strictly lower risk, so $\boldsymbol{D}_0$ is not a local minimizer of $R$ on $\mathcal{D}_m$.
\end{theorem}

\begin{proof}[Proof of Theorem~\ref{thm:instability}]
Recall that $\boldsymbol{D}_0$ consists of $\boldsymbol{D}_{\mathrm{basis}}$ followed by copies of its columns. For $0<\alpha\le\pi/2$, we will show that rotating the copy $\boldsymbol{d}_j=\boldsymbol{v}_1$ cannot increase the minimized loss for any input and strictly lowers it on $\mathcal{U}^+$. Averaging these pointwise comparisons will give the stated risk gap.

Let $\boldsymbol{a}^\star(\boldsymbol{h})$ be the optimal code for the basis decoder $\boldsymbol{D}_{\mathrm{basis}}$ from Lemma~\ref{lem:residual}. We first show that the minimum loss under $\boldsymbol{D}_0$ equals this basis loss. Fix $\boldsymbol{h}=t\boldsymbol{u}$, with $u_k=\boldsymbol{v}_k^\top\boldsymbol{u}$, and any code $\boldsymbol{a}\in\mathbb{R}^m$ for $\boldsymbol{D}_0$. Since every column $\boldsymbol{d}_\ell$ of $\boldsymbol{D}_0$ is a basis vector, orthonormality and the triangle inequality give
\[
\begin{aligned}
&\tfrac12\|\boldsymbol{h}-\boldsymbol{D}_0\boldsymbol{a}\|_2^2+\lambda\|\boldsymbol{a}\|_1\\
&=\tfrac12\sum_{k=1}^{d_i}\left(tu_k-\sum_{\ell:\boldsymbol{d}_\ell=\boldsymbol{v}_k}a_\ell\right)^2
+\lambda\sum_{k=1}^{d_i}\sum_{\ell:\boldsymbol{d}_\ell=\boldsymbol{v}_k}|a_\ell|\\
&\ge\tfrac12\sum_{k=1}^{d_i}\left(tu_k-\sum_{\ell:\boldsymbol{d}_\ell=\boldsymbol{v}_k}a_\ell\right)^2
+\lambda\sum_{k=1}^{d_i}\left|\sum_{\ell:\boldsymbol{d}_\ell=\boldsymbol{v}_k}a_\ell\right|\\
&\ge\tfrac12\sum_{k=1}^{d_i}(tu_k-a_k^\star)^2+\lambda\sum_{k=1}^{d_i}|a_k^\star|\\
&=L_{\mathrm{basis}}(\boldsymbol{h}).
\end{aligned}
\]
The second inequality uses the scalar optimality of each $a_k^\star$ from Lemma~\ref{lem:residual}. Taking the minimum over $\boldsymbol{a}$ yields $L^\star(\boldsymbol{h};\boldsymbol{D}_0)\ge L_{\mathrm{basis}}(\boldsymbol{h})$.

For every $\alpha\in[0,\pi/2]$, the first $d_i$ columns of $\boldsymbol{D}(\alpha)$ remain $\boldsymbol{D}_{\mathrm{basis}}$. Using $\boldsymbol{a}^\star(\boldsymbol{h})$ on these columns and zero on all additional columns gives
\[
\begin{aligned}
L^\star(\boldsymbol{h};\boldsymbol{D}(\alpha))
&\le\tfrac12\left\|\boldsymbol{h}-\boldsymbol{D}(\alpha)
\begin{bmatrix}\boldsymbol{a}^\star(\boldsymbol{h})\\\boldsymbol{0}\end{bmatrix}\right\|_2^2
+\lambda\left\|\begin{bmatrix}\boldsymbol{a}^\star(\boldsymbol{h})\\\boldsymbol{0}\end{bmatrix}\right\|_1\\
&=\tfrac12\|\boldsymbol{h}-\boldsymbol{D}_{\mathrm{basis}}\boldsymbol{a}^\star(\boldsymbol{h})\|_2^2
+\lambda\|\boldsymbol{a}^\star(\boldsymbol{h})\|_1\\
&=L_{\mathrm{basis}}(\boldsymbol{h}).
\end{aligned}
\]
At $\alpha=0$, this upper bound and the preceding lower bound force equality for $\boldsymbol{D}_0$. Consequently, for every $\boldsymbol{h}\in\mathcal{M}_i(t)$,
\[
L^\star(\boldsymbol{h};\boldsymbol{D}(\alpha))
\le L_{\mathrm{basis}}(\boldsymbol{h})
=L^\star(\boldsymbol{h};\boldsymbol{D}_0)
\qquad(0\le\alpha\le\pi/2).
\]
Thus the loss cannot increase outside $\mathcal{U}^+$ either.

It remains to prove a strict loss decrease on $\mathcal{U}^+$ by giving the rotated column a nonzero coefficient. Recall that $\widetilde{\boldsymbol{d}}=d_i^{-1/2}\sum_{k=1}^{d_i}\boldsymbol{v}_k$. For $\boldsymbol{u}\in\mathcal{U}^+$, Lemma~\ref{lem:residual} gives
\[
\boldsymbol{r}(t\boldsymbol{u})
=\lambda\sum_{k=1}^{d_i}\boldsymbol{v}_k
=\lambda\sqrt{d_i}\,\widetilde{\boldsymbol{d}}.
\]
Since $\|\widetilde{\boldsymbol{d}}\|_2=1$ and $\boldsymbol{v}_1^\top\widetilde{\boldsymbol{d}}=1/\sqrt{d_i}$, the normalization in $\boldsymbol{d}_j(\alpha)$ satisfies
\[
\begin{aligned}
\|\cos\alpha\,\boldsymbol{v}_1+\sin\alpha\,\widetilde{\boldsymbol{d}}\|_2^2
&=\cos^2\alpha+\sin^2\alpha
+\frac{2\sin\alpha\cos\alpha}{\sqrt{d_i}}\\
&=1+\frac{\sin(2\alpha)}{\sqrt{d_i}}\ge1.
\end{aligned}
\]
Substituting this normalization and the residual into the definition of $f$ yields
\[
\begin{aligned}
f(\alpha)
&=\frac{\boldsymbol{d}_j(\alpha)^\top \boldsymbol{r}(t\boldsymbol{u})}{\lambda}\\
&=\sqrt{d_i}\,\boldsymbol{d}_j(\alpha)^\top\widetilde{\boldsymbol{d}}\\
&=\frac{\sqrt{d_i}\bigl(\cos\alpha\,\boldsymbol{v}_1^\top\widetilde{\boldsymbol{d}}
+\sin\alpha\,\|\widetilde{\boldsymbol{d}}\|_2^2\bigr)}
{\sqrt{1+\sin(2\alpha)/\sqrt{d_i}}}\\
&=\frac{\cos\alpha+\sqrt{d_i}\sin\alpha}
{\sqrt{1+\sin(2\alpha)/\sqrt{d_i}}}.
\end{aligned}
\]
Differentiating this expression yields
\[
\begin{aligned}
f'(\alpha)
&=\frac{-\sin\alpha+\sqrt{d_i}\cos\alpha}
{\sqrt{1+\sin(2\alpha)/\sqrt{d_i}}}
-\frac{(\cos\alpha+\sqrt{d_i}\sin\alpha)\cos(2\alpha)}
{\sqrt{d_i}\bigl(1+\sin(2\alpha)/\sqrt{d_i}\bigr)^{3/2}}\\
&=\frac{(d_i-1)\cos\alpha}
{\sqrt{d_i}\bigl(1+\sin(2\alpha)/\sqrt{d_i}\bigr)^{3/2}}>0
\qquad(0\le\alpha<\pi/2).
\end{aligned}
\]
Thus $f$ is strictly increasing from $f(0)=1$ to $f(\pi/2)=\sqrt{d_i}$, and $f(\alpha)>1$ whenever $0<\alpha\le\pi/2$.

Fix such an $\alpha$ and an input with $\boldsymbol{u}\in\mathcal{U}^+$. Keep the first $d_i$ coefficients equal to $\boldsymbol{a}^\star(\boldsymbol{h})$, assign a nonnegative coefficient $a_j$ to the rotated column, and set the other coefficients to zero. This feasible code gives
\[
\begin{aligned}
L^\star(\boldsymbol{h};\boldsymbol{D}(\alpha))
&\le\tfrac12\|\boldsymbol{r}(\boldsymbol{h})-a_j\boldsymbol{d}_j(\alpha)\|_2^2
+\lambda\|\boldsymbol{a}^\star(\boldsymbol{h})\|_1+\lambda a_j\\
&=L_{\mathrm{basis}}(\boldsymbol{h})
-a_j\boldsymbol{d}_j(\alpha)^\top \boldsymbol{r}(\boldsymbol{h})
+\tfrac12a_j^2+\lambda a_j\\
&=L_{\mathrm{basis}}(\boldsymbol{h})
+\tfrac12a_j^2-\lambda\bigl(f(\alpha)-1\bigr)a_j\\
&=L_{\mathrm{basis}}(\boldsymbol{h})
+\tfrac12\bigl(a_j-\lambda(f(\alpha)-1)\bigr)^2
-\tfrac{\lambda^2}{2}\bigl(f(\alpha)-1\bigr)^2.
\end{aligned}
\]
Taking $a_j=\lambda(f(\alpha)-1)>0$ and using $L^\star(\boldsymbol{h};\boldsymbol{D}_0)=L_{\mathrm{basis}}(\boldsymbol{h})$ gives
\[
L^\star(\boldsymbol{h};\boldsymbol{D}_0)
-L^\star(\boldsymbol{h};\boldsymbol{D}(\alpha))
\ge\tfrac{\lambda^2}{2}\bigl(f(\alpha)-1\bigr)^2
\qquad(\boldsymbol{u}\in\mathcal{U}^+).
\]

By the sign symmetry established in the proof of Theorem~\ref{thm:violation}, $\Pr(\boldsymbol{u}\in\mathcal{U}^+)=2^{-d_i}\Pr(\boldsymbol{u}\in\mathcal{A})>0$. Combining the strict improvement on $\mathcal{U}^+$ with the nonnegative improvement elsewhere and taking expectations yields
\[
\begin{aligned}
R(\boldsymbol{D}_0)-R(\boldsymbol{D}(\alpha))
&=\mathbb{E}_{\boldsymbol{h}}\big[
L^\star(\boldsymbol{h};\boldsymbol{D}_0)
-L^\star(\boldsymbol{h};\boldsymbol{D}(\alpha))\big]\\
&\ge\tfrac{\lambda^2}{2}\bigl(f(\alpha)-1\bigr)^2
\Pr(\boldsymbol{u}\in\mathcal{U}^+)\\
&=2^{-(d_i+1)}\lambda^2\Pr(\boldsymbol{u}\in\mathcal{A})
\bigl(f(\alpha)-1\bigr)^2>0.
\end{aligned}
\]
Finally, the positive normalizing denominator makes $\boldsymbol{D}(\alpha)$ a continuous path in $\mathcal{D}_m$ with $\boldsymbol{D}(0)=\boldsymbol{D}_0$. Since $\boldsymbol{D}(\alpha)\to\boldsymbol{D}_0$ as $\alpha\to0$ and every positive $\alpha$ gives strictly lower risk, every neighborhood of $\boldsymbol{D}_0$ contains a dictionary of lower risk. Thus $\boldsymbol{D}_0$ is not a local minimizer.
\end{proof}
\subsection{Subspace-Aware Sparse Autoencoders}
\label{apx:theory-sasa}
We first construct an encoder and decoder that represent the entire feature slice with one active group. We then compare the population objective across routing partitions and prove that every global minimizer uses a single group.

\subsubsection{Proof of Proposition~\ref{prop:sasa-no-splitting}}
\label{apx:proof-sasa-no-splitting}

\begin{proof}
Let $\boldsymbol{V}_i=[\boldsymbol{v}_1,\dots,\boldsymbol{v}_{d_i}]\in\mathbb{R}^{d\times d_i}$ have orthonormal columns spanning $\mathcal{V}_i$, and fix an index $k^\star\in[K]$. Since $r\ge d_i$, we can choose
\[
\boldsymbol{D}_{k^\star}=[\boldsymbol{V}_i\ \boldsymbol{0}]\in\mathbb{R}^{d\times r},
\qquad
\boldsymbol{E}_{k^\star}=\begin{bmatrix}\boldsymbol{V}_i^\top\\\boldsymbol{0}\end{bmatrix}\in\mathbb{R}^{r\times d},
\]
where the zero columns and rows pad the decoder and encoder to width $r$. Set $\boldsymbol{D}_k=\boldsymbol{0}$ and $\boldsymbol{E}_k=\boldsymbol{0}$ for $k\ne k^\star$.

For every $\boldsymbol{h}\in\mathcal{M}_i(t)$, we have $\boldsymbol{h}\in\mathcal{V}_i$ and $\|\boldsymbol{h}\|_2=t>0$. Thus
\[
\begin{aligned}
\|\boldsymbol{p}_{k^\star}(\boldsymbol{h})\|_2^2
&=\|\boldsymbol{V}_i^\top\boldsymbol{h}\|_2^2
=\boldsymbol{h}^\top\boldsymbol{V}_i\boldsymbol{V}_i^\top\boldsymbol{h}\\
&=\boldsymbol{h}^\top\boldsymbol{P}_{\mathcal{V}_i}\boldsymbol{h}
=\|\boldsymbol{h}\|_2^2=t^2,
\end{aligned}
\]
whereas $\boldsymbol{p}_k(\boldsymbol{h})=\boldsymbol{0}$ for every $k\ne k^\star$. The Top-$1$ gate therefore selects $k^\star$ uniquely on the entire slice. Its active code is $\boldsymbol{a}_{k^\star}(\boldsymbol{h})=\boldsymbol{E}_{k^\star}\boldsymbol{h}$, so
\[
\begin{aligned}
\boldsymbol{D}\boldsymbol{a}(\boldsymbol{h})
&=\boldsymbol{D}_{k^\star}\boldsymbol{E}_{k^\star}\boldsymbol{h}
=\boldsymbol{V}_i\boldsymbol{V}_i^\top\boldsymbol{h}\\
&=\boldsymbol{P}_{\mathcal{V}_i}\boldsymbol{h}
=\boldsymbol{h}.
\end{aligned}
\]
In particular, $\operatorname{col}(\boldsymbol{D}_{k^\star})=\mathcal{V}_i$ and the claimed supremum of distances is zero.

The equality $L_i^{(d_i)}(0)=d_i$ follows from the basis argument in Theorem~\ref{thm:line-covering}.
\end{proof}
\subsubsection{Proof of Theorem~\ref{thm:block-recovery}}
\label{apx:proof-block-recovery}

\begin{proof}
We compare the population objective $R(\boldsymbol{E},\boldsymbol{D})$ of an arbitrary representation with the value attained by a single block. We will show that this single-block value is minimal and that attaining it determines the routing and reconstruction maps.

First, use the encoder and decoder constructed in Proposition~\ref{prop:sasa-no-splitting} and multiply the active decoder block by $1-\lambda_{\mathrm{dim}}d_i/(2t^2)>0$. The encoder still routes the entire slice to $k^\star$, and the reconstruction map is
\[
\boldsymbol{W}_{k^\star}
=\left(1-\frac{\lambda_{\mathrm{dim}}d_i}{2t^2}\right)\boldsymbol{P}_{\mathcal{V}_i},
\qquad
\boldsymbol{W}_k=\boldsymbol{0}\quad(k\ne k^\star).
\]
Since $\|\boldsymbol{h}\|_2=t$ and $\|\boldsymbol{P}_{\mathcal{V}_i}\|_*=d_i$, this construction has objective value
\[
\begin{aligned}
R(\boldsymbol{E},\boldsymbol{D})
&=t^2\left(\frac{\lambda_{\mathrm{dim}}d_i}{2t^2}\right)^2
+\lambda_{\mathrm{dim}}d_i\left(1-\frac{\lambda_{\mathrm{dim}}d_i}{2t^2}\right)\\
&=d_i\lambda_{\mathrm{dim}}-\frac{d_i^2\lambda_{\mathrm{dim}}^2}{4t^2}.
\end{aligned}
\]

We now prove that no encoder and decoder can achieve a smaller value. Uniform sampling from the full sphere in $\mathcal{V}_i$ gives
\[
\boldsymbol{\Sigma}_i=\mathbb{E}[\boldsymbol{h}\boldsymbol{h}^\top]
=\frac{t^2}{d_i}\boldsymbol{P}_{\mathcal{V}_i}.
\]
For arbitrary encoder and decoder matrices, let
\[
\begin{aligned}
\mathcal{C}_k&=\{\boldsymbol{h}\in\mathcal{M}_i(t):\mathcal{T}_1(\boldsymbol{h})=\{k\}\},\\
\boldsymbol{M}_k&=\mathbb{E}\big[\boldsymbol{h}\boldsymbol{h}^\top\mathbbm{1}\{\boldsymbol{h}\in\mathcal{C}_k\}\big],
\qquad \sum_{k=1}^K\boldsymbol{M}_k=\boldsymbol{\Sigma}_i.
\end{aligned}
\]
These blocks partition the slice. On $\mathcal{C}_k$, the reconstruction is $\boldsymbol{W}_k\boldsymbol{h}$ with $\boldsymbol{W}_k=\boldsymbol{D}_k\boldsymbol{E}_k$. Using $\|\boldsymbol{x}\|_2^2=\operatorname{tr}(\boldsymbol{x}\boldsymbol{x}^\top)$ and linearity of expectation, we obtain
\begin{equation}
\label{eq:sasa-routing-cost}
\begin{aligned}
R(\boldsymbol{E},\boldsymbol{D})
&=\mathbb{E}\big[\|\boldsymbol{h}-\boldsymbol{D}\boldsymbol{a}(\boldsymbol{h})\|_2^2\big]
+\lambda_{\mathrm{dim}}\sum_{k=1}^K\|\boldsymbol{W}_k\|_*\\
&=\sum_{k=1}^K\left[
\mathbb{E}\big[\|(\boldsymbol{I}-\boldsymbol{W}_k)\boldsymbol{h}\|_2^2\,\mathbbm{1}\{\boldsymbol{h}\in\mathcal{C}_k\}\big]
+\lambda_{\mathrm{dim}}\|\boldsymbol{W}_k\|_*\right]\\
&=\sum_{k=1}^K\left[
\operatorname{tr}\big((\boldsymbol{I}-\boldsymbol{W}_k)\boldsymbol{M}_k(\boldsymbol{I}-\boldsymbol{W}_k)^\top\big)
+\lambda_{\mathrm{dim}}\|\boldsymbol{W}_k\|_*\right].
\end{aligned}
\end{equation}

The key step is to bound the nuclear-norm penalty in terms of the reconstruction. Because $\boldsymbol{0}\preceq\boldsymbol{M}_k\preceq\boldsymbol{\Sigma}_i$, we have $\|\boldsymbol{M}_k\|_2\le t^2/d_i$. The trace inequality for the nuclear and spectral norms therefore gives
\[
\operatorname{tr}(\boldsymbol{W}_k\boldsymbol{M}_k)
\le\|\boldsymbol{W}_k\|_*\|\boldsymbol{M}_k\|_2
\le\frac{t^2}{d_i}\|\boldsymbol{W}_k\|_*.
\]
Summing over blocks and using the definition of $\boldsymbol{M}_k$ yields
\begin{equation}
\label{eq:sasa-penalty-bound}
\begin{aligned}
\lambda_{\mathrm{dim}}\sum_{k=1}^K\|\boldsymbol{W}_k\|_*
&\ge\frac{\lambda_{\mathrm{dim}}d_i}{t^2}
\sum_{k=1}^K\operatorname{tr}(\boldsymbol{W}_k\boldsymbol{M}_k)\\
&=\frac{\lambda_{\mathrm{dim}}d_i}{t^2}
\sum_{k=1}^K\mathbb{E}\big[\boldsymbol{h}^\top\boldsymbol{W}_k\boldsymbol{h}\,\mathbbm{1}\{\boldsymbol{h}\in\mathcal{C}_k\}\big]\\
&=\frac{\lambda_{\mathrm{dim}}d_i}{t^2}
\mathbb{E}\big[\boldsymbol{h}^\top\boldsymbol{D}\boldsymbol{a}(\boldsymbol{h})\big].
\end{aligned}
\end{equation}
Substituting this bound into $R(\boldsymbol{E},\boldsymbol{D})$ and completing the square in the reconstruction $\boldsymbol{D}\boldsymbol{a}(\boldsymbol{h})$ gives
\begin{equation}
\label{eq:sasa-consolidation-bound}
\begin{aligned}
R(\boldsymbol{E},\boldsymbol{D})
&\ge\mathbb{E}\left[
\|\boldsymbol{h}-\boldsymbol{D}\boldsymbol{a}(\boldsymbol{h})\|_2^2
+\frac{\lambda_{\mathrm{dim}}d_i}{t^2}\boldsymbol{h}^\top\boldsymbol{D}\boldsymbol{a}(\boldsymbol{h})\right]\\
&=\mathbb{E}\left[
\|\boldsymbol{D}\boldsymbol{a}(\boldsymbol{h})\|_2^2
-\left(2-\frac{\lambda_{\mathrm{dim}}d_i}{t^2}\right)\boldsymbol{h}^\top\boldsymbol{D}\boldsymbol{a}(\boldsymbol{h})
+t^2\right]\\
&=\mathbb{E}\Bigl[
\|\boldsymbol{D}\boldsymbol{a}(\boldsymbol{h})\|_2^2
-2\left(1-\frac{\lambda_{\mathrm{dim}}d_i}{2t^2}\right)\boldsymbol{h}^\top\boldsymbol{D}\boldsymbol{a}(\boldsymbol{h})\\
&\qquad+\left(1-\frac{\lambda_{\mathrm{dim}}d_i}{2t^2}\right)^2t^2\Bigr]\\
&\quad+t^2-\left(1-\frac{\lambda_{\mathrm{dim}}d_i}{2t^2}\right)^2t^2\\
&=\underbrace{\mathbb{E}\left\|
\boldsymbol{D}\boldsymbol{a}(\boldsymbol{h})
-\left(1-\frac{\lambda_{\mathrm{dim}}d_i}{2t^2}\right)\boldsymbol{h}
\right\|_2^2}_{\ge0}\\
&\quad+t^2\left[1-\left(1-\frac{\lambda_{\mathrm{dim}}d_i}{2t^2}\right)^2\right]\\
&\ge t^2\left[1-\left(1-\frac{\lambda_{\mathrm{dim}}d_i}{2t^2}\right)^2\right]\\
&=t^2\left[1-\left(1-\frac{\lambda_{\mathrm{dim}}d_i}{t^2}
+\frac{\lambda_{\mathrm{dim}}^2d_i^2}{4t^4}\right)\right]\\
&=t^2\left[\frac{\lambda_{\mathrm{dim}}d_i}{t^2}
-\frac{\lambda_{\mathrm{dim}}^2d_i^2}{4t^4}\right]\\
&=d_i\lambda_{\mathrm{dim}}-\frac{d_i^2\lambda_{\mathrm{dim}}^2}{4t^2}.
\end{aligned}
\end{equation}
Keeping the remainder in Eq.~\ref{eq:sasa-penalty-bound} gives the exact identity
\[
\begin{aligned}
R(\boldsymbol{E},\boldsymbol{D})
&=\mathbb{E}\|\boldsymbol{h}-\boldsymbol{D}\boldsymbol{a}(\boldsymbol{h})\|_2^2
+\lambda_{\mathrm{dim}}\sum_{k=1}^K\|\boldsymbol{W}_k\|_*\\
&=\mathbb{E}\|\boldsymbol{h}-\boldsymbol{D}\boldsymbol{a}(\boldsymbol{h})\|_2^2
+\lambda_{\mathrm{dim}}\sum_{k=1}^K\|\boldsymbol{W}_k\|_*\\
&\quad+\frac{\lambda_{\mathrm{dim}}d_i}{t^2}
\mathbb{E}\big[\boldsymbol{h}^\top\boldsymbol{D}\boldsymbol{a}(\boldsymbol{h})\big]
-\frac{\lambda_{\mathrm{dim}}d_i}{t^2}
\mathbb{E}\big[\boldsymbol{h}^\top\boldsymbol{D}\boldsymbol{a}(\boldsymbol{h})\big]\\
&=\mathbb{E}\|\boldsymbol{h}-\boldsymbol{D}\boldsymbol{a}(\boldsymbol{h})\|_2^2
+\frac{\lambda_{\mathrm{dim}}d_i}{t^2}
\mathbb{E}\big[\boldsymbol{h}^\top\boldsymbol{D}\boldsymbol{a}(\boldsymbol{h})\big]\\
&\quad+\lambda_{\mathrm{dim}}\sum_{k=1}^K\|\boldsymbol{W}_k\|_*
-\frac{\lambda_{\mathrm{dim}}d_i}{t^2}
\sum_{k=1}^K\operatorname{tr}(\boldsymbol{W}_k\boldsymbol{M}_k)\\
&=d_i\lambda_{\mathrm{dim}}-\frac{d_i^2\lambda_{\mathrm{dim}}^2}{4t^2}\\
&\quad+\mathbb{E}\left\|
\boldsymbol{D}\boldsymbol{a}(\boldsymbol{h})
-\left(1-\frac{\lambda_{\mathrm{dim}}d_i}{2t^2}\right)\boldsymbol{h}
\right\|_2^2\\
&\quad+\lambda_{\mathrm{dim}}\sum_{k=1}^K\left[
\|\boldsymbol{W}_k\|_*
-\frac{d_i}{t^2}\operatorname{tr}(\boldsymbol{W}_k\boldsymbol{M}_k)\right].
\end{aligned}
\]
The single-block construction attains the bound in Eq.~\ref{eq:sasa-consolidation-bound}. At any global minimizer, $R(\boldsymbol{E},\boldsymbol{D})=d_i\lambda_{\mathrm{dim}}-d_i^2\lambda_{\mathrm{dim}}^2/(4t^2)$, so the nonnegative squared term and every nonnegative summand must vanish. The squared term gives
\[
\begin{aligned}
\mathbb{E}\left\|\boldsymbol{D}\boldsymbol{a}(\boldsymbol{h})
-\left(1-\frac{\lambda_{\mathrm{dim}}d_i}{2t^2}\right)\boldsymbol{h}\right\|_2^2=0\\
\Longrightarrow\quad
\boldsymbol{D}\boldsymbol{a}(\boldsymbol{h})
=\left(1-\frac{\lambda_{\mathrm{dim}}d_i}{2t^2}\right)\boldsymbol{h},
\end{aligned}
\]
where the reconstruction identity holds almost surely.
The penalty summands also vanish, so for every block $k$,
\[
\begin{aligned}
\|\boldsymbol{W}_k\|_*
-\frac{d_i}{t^2}\operatorname{tr}(\boldsymbol{W}_k\boldsymbol{M}_k)=0\\
\Longrightarrow\quad
\|\boldsymbol{W}_k\|_*
=\frac{d_i}{t^2}\operatorname{tr}(\boldsymbol{W}_k\boldsymbol{M}_k).
\end{aligned}
\]

We use these conditions to prove that a global minimizer selects the same block almost surely. Suppose, for a contradiction, that two or more blocks have positive probability. Fix any such block $\mathcal{C}_k$ and any nonzero $\boldsymbol{v}\in\mathcal{V}_i$. Under uniform sampling, the sphere points orthogonal to $\boldsymbol{v}$ form a lower-dimensional subset of probability zero. Thus
\[
\mathbb{P}(\boldsymbol{h}\in\mathcal{C}_k,\ \boldsymbol{v}^\top\boldsymbol{h}\ne0)
=\mathbb{P}(\boldsymbol{h}\in\mathcal{C}_k)>0.
\]
The integrand below is strictly positive on this event, so
\[
\boldsymbol{v}^\top\boldsymbol{M}_k\boldsymbol{v}
=\mathbb{E}\big[(\boldsymbol{v}^\top\boldsymbol{h})^2\mathbbm{1}\{\boldsymbol{h}\in\mathcal{C}_k\}\big]>0.
\]
Consequently, the moment of every positive-probability block is positive definite on $\mathcal{V}_i$. At least one other block has positive probability, so applying the same argument to the other blocks gives
\[
\boldsymbol{v}^\top(\boldsymbol{\Sigma}_i-\boldsymbol{M}_k)\boldsymbol{v}
=\sum_{\ell\ne k}\boldsymbol{v}^\top\boldsymbol{M}_\ell\boldsymbol{v}>0.
\]
The strict positive definiteness of $\boldsymbol{\Sigma}_i-\boldsymbol{M}_k$ on $\mathcal{V}_i$ therefore gives
\[
\|\boldsymbol{M}_k\|_2
<\|\boldsymbol{\Sigma}_i\|_2
=\frac{t^2}{d_i}.
\]
For almost every $\boldsymbol{h}\in\mathcal{C}_k$, the reconstruction identity gives
\[
\|\boldsymbol{W}_k\boldsymbol{h}\|_2
=\left(1-\frac{\lambda_{\mathrm{dim}}d_i}{2t^2}\right)t>0,
\]
so $\boldsymbol{W}_k\ne\boldsymbol{0}$. The penalty gap for this block consequently satisfies
\[
\|\boldsymbol{W}_k\|_*
-\frac{d_i}{t^2}\operatorname{tr}(\boldsymbol{W}_k\boldsymbol{M}_k)
\ge\left(1-\frac{d_i}{t^2}\|\boldsymbol{M}_k\|_2\right)\|\boldsymbol{W}_k\|_*>0.
\]
This contradicts the zero penalty gap required at every global minimizer. Thus one block $\mathcal{C}_{k^\star}$ has probability one, so
\[
\boldsymbol{M}_{k^\star}=\boldsymbol{\Sigma}_i,
\qquad
\boldsymbol{M}_k=\boldsymbol{0}\quad(k\ne k^\star).
\]
We now determine the block reconstruction maps. For every unselected block, the equality condition gives
\[
\|\boldsymbol{W}_k\|_*
=\frac{d_i}{t^2}\operatorname{tr}(\boldsymbol{W}_k\boldsymbol{M}_k)
=0,
\qquad
\boldsymbol{W}_k=\boldsymbol{0}.
\]

For the active block, the reconstruction identity determines its action on $\mathcal{V}_i$. Using $\boldsymbol{\Sigma}_i=(t^2/d_i)\boldsymbol{P}_{\mathcal{V}_i}$,
\[
\left\|\boldsymbol{W}_{k^\star}\boldsymbol{P}_{\mathcal{V}_i}
-\left(1-\frac{\lambda_{\mathrm{dim}}d_i}{2t^2}\right)\boldsymbol{P}_{\mathcal{V}_i}\right\|_{\mathrm{F}}^2
=\frac{d_i}{t^2}\mathbb{E}\left\|\boldsymbol{W}_{k^\star}\boldsymbol{h}
-\left(1-\frac{\lambda_{\mathrm{dim}}d_i}{2t^2}\right)\boldsymbol{h}\right\|_2^2
=0.
\]
Hence
\[
\boldsymbol{W}_{k^\star}\boldsymbol{P}_{\mathcal{V}_i}
=\left(1-\frac{\lambda_{\mathrm{dim}}d_i}{2t^2}\right)\boldsymbol{P}_{\mathcal{V}_i}.
\]
Equality in the penalty bound also fixes the nuclear norm,
\[
\begin{aligned}
\|\boldsymbol{W}_{k^\star}\|_*
&=\frac{d_i}{t^2}\operatorname{tr}(\boldsymbol{W}_{k^\star}\boldsymbol{\Sigma}_i)
=\operatorname{tr}(\boldsymbol{W}_{k^\star}\boldsymbol{P}_{\mathcal{V}_i})\\
&=d_i\left(1-\frac{\lambda_{\mathrm{dim}}d_i}{2t^2}\right)
=\|\boldsymbol{W}_{k^\star}\boldsymbol{P}_{\mathcal{V}_i}\|_*.
\end{aligned}
\]
To identify the full map, it remains to show that it vanishes on $\mathcal{V}_i^\perp$. The orthogonal decomposition gives
\[
\begin{aligned}
\boldsymbol{W}_{k^\star}\boldsymbol{W}_{k^\star}^\top
&=\boldsymbol{W}_{k^\star}\boldsymbol{P}_{\mathcal{V}_i}\boldsymbol{W}_{k^\star}^\top\\
&\quad+\boldsymbol{W}_{k^\star}(\boldsymbol{I}-\boldsymbol{P}_{\mathcal{V}_i})\boldsymbol{W}_{k^\star}^\top,
\end{aligned}
\]
where the second summand is positive semidefinite. Consequently,
\[
\begin{aligned}
\|\boldsymbol{W}_{k^\star}\|_*
&=\operatorname{tr}\big((\boldsymbol{W}_{k^\star}\boldsymbol{W}_{k^\star}^\top)^{1/2}\big)\\
&\ge\operatorname{tr}\big((\boldsymbol{W}_{k^\star}\boldsymbol{P}_{\mathcal{V}_i}\boldsymbol{W}_{k^\star}^\top)^{1/2}\big)
=\|\boldsymbol{W}_{k^\star}\boldsymbol{P}_{\mathcal{V}_i}\|_*.
\end{aligned}
\]
Adding a nonzero positive semidefinite matrix strictly increases the trace of the square root. Equality of the two nuclear norms therefore implies
\[
\begin{aligned}
\boldsymbol{W}_{k^\star}(\boldsymbol{I}-\boldsymbol{P}_{\mathcal{V}_i})\boldsymbol{W}_{k^\star}^\top
&=\boldsymbol{0},\\
\|\boldsymbol{W}_{k^\star}(\boldsymbol{I}-\boldsymbol{P}_{\mathcal{V}_i})\|_{\mathrm{F}}^2
&=\operatorname{tr}\big(\boldsymbol{W}_{k^\star}(\boldsymbol{I}-\boldsymbol{P}_{\mathcal{V}_i})\boldsymbol{W}_{k^\star}^\top\big)
=0.
\end{aligned}
\]
Thus
\[
\boldsymbol{W}_{k^\star}
=\boldsymbol{W}_{k^\star}\boldsymbol{P}_{\mathcal{V}_i}
=\left(1-\frac{\lambda_{\mathrm{dim}}d_i}{2t^2}\right)\boldsymbol{P}_{\mathcal{V}_i}.
\]
This map has column space $\mathcal{V}_i$ and is unique up to the active block index. If $r=d_i$, then
\[
d_i=\operatorname{rank}(\boldsymbol{W}_{k^\star})
\le\operatorname{rank}(\boldsymbol{D}_{k^\star})\le r=d_i,
\]
and $\operatorname{col}(\boldsymbol{W}_{k^\star})\subseteq\operatorname{col}(\boldsymbol{D}_{k^\star})$ gives $\operatorname{col}(\boldsymbol{D}_{k^\star})=\mathcal{V}_i$.
\end{proof}

\subsection{Sample Complexity}
\label{apx:theory-sample-complexity}
We first use matrix Bernstein to bound the error in the empirical second-moment matrix. A Davis--Kahan bound then translates this error into a bound on the estimated feature subspace, yielding Theorem~\ref{thm:sasa-sample-complexity}. Finally, we prove Proposition~\ref{prop:split-sample-complexity} by bounding the probability that at least one atom in a fixed dictionary receives no routed sample.

\subsubsection{Matrix concentration and subspace perturbation}
This section introduces some useful facts, which are key in the proof of Theorem~\ref{thm:sasa-sample-complexity} and Proposition~\ref{prop:split-sample-complexity} below. We first recall the Matrix Bernstein inequality.
\begin{theorem}[Theorem 1.4, \cite{Tropp_2012}]\label{thm:matrix:Bernstein}
    Consider a finite sequence $\{\boldsymbol{X}_k\}$ of independent, random, self-adjoint matrices with dimension $d$. Assume that each random matrix satisfies
\[
\mathbb{E}[\boldsymbol{X}_k] = \boldsymbol{0} \quad \text{and} \quad \lambda_{\max}(\boldsymbol{X}_k) \leq R \quad \text{almost surely.}
\]
Then, for all $t \geq 0$,
\[
\mathbb{P}\left\{ \lambda_{\max}\left( \sum_k \boldsymbol{X}_k \right) \geq t \right\} \leq d  \exp\left( \frac{-t^2}{2\sigma^2 + \frac{2}{3}Rt} \right) \qquad \text{where } \sigma^2 := \left\| \sum_k \mathbb{E}[ (\boldsymbol{X}_k^2)] \right\|.
\]
\end{theorem}
We now provide a matrix concentration inequality based on that as follows.
\begin{lemma}
    \label{lem:cov-concentration}
    Let $\boldsymbol{h}_1,\dots,\boldsymbol{h}_n\in\mathbb{R}^d$ be i.i.d.\ with $\|\boldsymbol{h}\|_2=t$ and define
    $\boldsymbol{\Sigma}:=\mathbb{E}_{\boldsymbol{h} \sim \mathcal{D}}[\boldsymbol{h}\boldsymbol{h}^{\top}],\,
    \widehat{\boldsymbol{\Sigma}}:=\frac{1}{n}\sum_{i=1}^{n}\boldsymbol{h}_i\boldsymbol{h}_i^{\top}$.
    Then for any $\delta\in(0,1)$, with probability at least $1-\delta$,
    \[
    \bigl\|\widehat{\boldsymbol{\Sigma}}-\boldsymbol{\Sigma}\bigr\|_2
    \le
    2t^2\left(\frac{\log\!\left(\frac{2d}{\delta}\right)}{3n} +  \frac{1}{n}\sqrt{\frac{\log^2\!\left(\frac{2d}{\delta}\right)}{9} + 2n\log\!\left(\frac{2d}{\delta}\right)}\right).
    \]
\end{lemma}
\begin{proof}
    We start by applying  Theorem~\ref{thm:matrix:Bernstein} to the centered rank-one summands $\boldsymbol{h}_i\boldsymbol{h}_i^\top-\boldsymbol{\Sigma}$, for which the norm bound $R$ and the variance proxy $\sigma^2$ follow directly from $\|\boldsymbol{h}\|_2=t$. Define the centered self-adjoint random matrices
    \[
    \boldsymbol{X}_i:=\boldsymbol{h}_i\boldsymbol{h}_i^\top-\boldsymbol{\Sigma},
    \qquad
    \boldsymbol{S}:=\sum_{i=1}^n \boldsymbol{X}_i,
    \]
    so that $\mathbb{E}[\boldsymbol{X}_i]=\boldsymbol{0}$, $\widehat{\boldsymbol{\Sigma}}-\boldsymbol{\Sigma}=\frac{1}{n}\boldsymbol{S}$, and therefore
    $\|\widehat{\boldsymbol{\Sigma}}-\boldsymbol{\Sigma}\|_2 = \frac{1}{n}\|\boldsymbol{S}\|_2$. Let us first bound $R$ and $\sigma^2$. For $R$, note that based on Jensen's inequality, we have
    \[
    \|\boldsymbol{\Sigma}\|_2
    =\lambda_{\max}(\boldsymbol{\Sigma})
    = \lambda_{\max}(\mathbb{E}[\boldsymbol{h}\boldsymbol{h}^\top])
    \le \mathbb{E}\bigl[\lambda_{\max}(\boldsymbol{h}\boldsymbol{h}^\top)\bigr]
    = \mathbb{E}\bigl[\|\boldsymbol{h}\boldsymbol{h}^\top\|_2\bigr]
    =t^2.
    \]
    Hence, we have
    \[
    \|\boldsymbol{X}_i\|_2
    \le \|\boldsymbol{h}_i\boldsymbol{h}_i^\top\|_2+\|\boldsymbol{\Sigma}\|_2
    \le 2t^2
    \]
    and we set $R= 2t^2$. Now for $\sigma^2$, note that we have
    \begin{equation*}
        \mathbb{E}[\boldsymbol{X}_i^2]\preceq 2\mathbb{E}\bigl[(\boldsymbol{h}\boldsymbol{h}^\top)^2\bigr] + 2\boldsymbol{\Sigma}^2
        = 2\mathbb{E}\bigl[\|\boldsymbol{h}\|_2^2 \boldsymbol{h}\boldsymbol{h}^\top\bigr] + 2\boldsymbol{\Sigma}^2
        \preceq 2t^2 \boldsymbol{\Sigma} + 2t^2 \boldsymbol{\Sigma}
        = 4t^2 \boldsymbol{\Sigma}.
    \end{equation*}
    So we have
    \[
    \sigma^2= \left\| \sum_{i=1}^n \mathbb{E}[\boldsymbol{X}_i^2] \right\|_2 \le \sum_{i=1}^n 4t^2 \|\boldsymbol{\Sigma}\|_2 \le 4nt^4.
    \]
    Now note that we can write the spectral norm as the maximum over the largest positive eigenvalue and the negative of the smallest negative eigenvalue
    \[
    \|\boldsymbol{S}\|_2 = \max \{\lambda_{\max}(\boldsymbol{S}), -\lambda_{\min}(\boldsymbol{S})\}
    \]
    
   As we seek to bound $\|\boldsymbol{S}\|_2$, we can  simply apply the union bound to get
    \[
    \mathbb{P}\left\{ \|\boldsymbol{S}\|_2 \geq u \right\} \leq \mathbb{P}\left\{\lambda_{\max}(\boldsymbol{S}) \geq u \right\} + \mathbb{P}\left\{ -\lambda_{\min}(\boldsymbol{S}) \geq u \right\} 
    \]
    which leads to
    \[
    \mathbb{P}\bigl(\|\boldsymbol{S}\|_2\ge u\bigr) \le 2d\exp\!\left( -\frac{u^2}{2\sigma^2+\frac{2}{3}Ru} \right).
    \]
    To obtain a high-probability bound, it suffices to choose $u$ so that
    \[
    2d\exp\!\left(
    -\frac{u^2}{2\sigma^2+\frac{2}{3}Ru}
    \right)\le \delta.
    \]
    Now, we have
    \begin{align*}
        &\frac{u^2}{2\sigma^2+\frac{2}{3}Ru} \ge \log\!\left(\frac{2d}{\delta}\right)\\
        \implies &u^2 - \frac{2R}{3}\log\!\left(\frac{2d}{\delta}\right)u -2\log\!\left(\frac{2d}{\delta}\right)\sigma^2 \ge 0
    \end{align*}
    which is quadratic in $u$ and is satisfied when
    \[
    u \ge \frac{R}{3}\log\!\left(\frac{2d}{\delta}\right) + \sqrt{\frac{R^2}{9}\log^2\!\left(\frac{2d}{\delta}\right) + 2\sigma^2\log\!\left(\frac{2d}{\delta}\right)}
    \]
    Therefore, with probability at least $1-\delta$,
    \[
    \|\boldsymbol{S}\|_2\le \frac{R}{3}\log\!\left(\frac{2d}{\delta}\right) + \sqrt{\frac{R^2}{9}\log^2\!\left(\frac{2d}{\delta}\right) + 2\sigma^2\log\!\left(\frac{2d}{\delta}\right)}.
    \]
    Finally, by substituting  $\widehat{\boldsymbol{\Sigma}}-\boldsymbol{\Sigma}=\frac{1}{n}\boldsymbol{S}$, $R= 2t^2$, and $\sigma^2 \le 4nt^4$, we conclude that with probability $1-\delta$ it holds that
    \begin{align*}
        \|\boldsymbol{S}\|_2 &\le \frac{R}{3}\log\!\left(\frac{2d}{\delta}\right) + \sqrt{\frac{R^2}{9}\log^2\!\left(\frac{2d}{\delta}\right) + 2\sigma^2\log\!\left(\frac{2d}{\delta}\right)}\\
        \implies  \|\widehat{\boldsymbol{\Sigma}}-\boldsymbol{\Sigma}\|_2
        &\le \frac{2t^2}{3n}\log\!\left(\frac{2d}{\delta}\right) + \frac{1}{n}\sqrt{\frac{4t^4}{9}\log^2\!\left(\frac{2d}{\delta}\right) + 8nt^4\log\!\left(\frac{2d}{\delta}\right)}\\
        \implies  \|\widehat{\boldsymbol{\Sigma}}-\boldsymbol{\Sigma}\|_2 &\le \frac{2t^2}{3n}\log\!\left(\frac{2d}{\delta}\right) + \frac{2t^2}{n} \sqrt{\frac{\log^2\!\left(\frac{2d}{\delta}\right)}{9} + 2n\log\!\left(\frac{2d}{\delta}\right)}\\
        &=2t^2\left(\frac{\log\!\left(\frac{2d}{\delta}\right)}{3n} +  \frac{1}{n}\sqrt{\frac{\log^2\!\left(\frac{2d}{\delta}\right)}{9} + 2n\log\!\left(\frac{2d}{\delta}\right)}\right).
    \end{align*}
\end{proof}

Lemma~\ref{lem:cov-concentration} provides a high-probability norm control of the perturbation $\widehat{\boldsymbol{\Sigma}}-\boldsymbol{\Sigma}$.
Next, we restate one variant of Davis--Kahan Theorems.
\begin{theorem}[Theorem 1, \cite{davis-kahan}]
\label{thm:dk}
Let $\boldsymbol{\Sigma}, \hat{\boldsymbol{\Sigma}} \in \mathbb{R}^{p \times p}$ be symmetric, with eigenvalues $\lambda_1 \ge \dots \ge \lambda_p$ and $\hat{\lambda}_1 \ge \dots \ge \hat{\lambda}_p$ respectively. Fix $1 \le r \le s \le p$, let $d = s - r + 1$, and let $\boldsymbol{V} = (\boldsymbol{v}_r, \boldsymbol{v}_{r+1}, \dots, \boldsymbol{v}_s) \in \mathbb{R}^{p \times d}$ and $\hat{\boldsymbol{V}} = (\hat{\boldsymbol{v}}_r, \hat{\boldsymbol{v}}_{r+1}, \dots, \hat{\boldsymbol{v}}_s) \in \mathbb{R}^{p \times d}$ have orthonormal columns satisfying $\boldsymbol{\Sigma} \boldsymbol{v}_j = \lambda_j \boldsymbol{v}_j$ and $\hat{\boldsymbol{\Sigma}} \hat{\boldsymbol{v}}_j = \hat{\lambda}_j \hat{\boldsymbol{v}}_j$ for $j = r, r+1, \dots, s$. Write $\delta = \inf\{|\hat{\lambda} - \lambda| : \lambda \in [\lambda_s, \lambda_r], \hat{\lambda} \in (-\infty, \hat{\lambda}_{s+1}] \cup [\hat{\lambda}_{r-1}, \infty)\}$, where we define $\hat{\lambda}_0 = -\infty$ and $\hat{\lambda}_{p+1} = \infty$, and assume that $\delta > 0$. Then
\[
\|\sin \Theta(\hat{\boldsymbol{V}}, \boldsymbol{V})\|_2 \le \frac{\|\hat{\boldsymbol{\Sigma}} - \boldsymbol{\Sigma}\|_2}{\delta}.
\]
\end{theorem}
The following lemma now turns the matrix perturbation bound in Lemma~\ref{lem:cov-concentration} into
a bound on the operator norm of the difference of orthogonal projectors.
\begin{lemma}[Davis--Kahan projector bound]
\label{lem:dk-projector}
Let $\boldsymbol{\Sigma},\widehat{\boldsymbol{\Sigma}}\in\mathbb{R}^{d\times d}$ be symmetric. Fix an integer $1\le d_i<d$, and let
$\mathcal{V}$ and $\widehat{\mathcal{V}}$ denote the top-$d_i$ eigenspaces of $\boldsymbol{\Sigma}$ and $\widehat{\boldsymbol{\Sigma}}$, respectively.  Write the eigenvalues of $\boldsymbol{\Sigma}$ as
$\lambda_1\ge \lambda_2\ge \cdots \ge \lambda_d$, and define the population eigengap $\Delta := \lambda_{d_i}-\lambda_{d_i+1} > 0$. 
If
\[
\bigl\|\widehat{\boldsymbol{\Sigma}}-\boldsymbol{\Sigma}\bigr\|_2 \le \frac{\Delta}{2},
\]
then
\[
\bigl\|\boldsymbol{P}_{\widehat{\mathcal{V}}}-\boldsymbol{P}_{\mathcal{V}}\bigr\|_2
\le
\frac{2}{\Delta}\,\bigl\|\widehat{\boldsymbol{\Sigma}}-\boldsymbol{\Sigma}\bigr\|_2.
\]
\end{lemma}

\begin{proof}
The bound converts a perturbation of the second-moment matrix into a perturbation of its top-$d_i$ eigenspace, measured by the operator norm of the gap between the two projectors. Choose matrices $\boldsymbol{V},\widehat{\boldsymbol{V}}\in\mathbb{R}^{d\times d_i}$ with orthonormal columns spanning $\mathcal{V}$ and $\widehat{\mathcal{V}}$, respectively, so that $\boldsymbol{P}_{\mathcal{V}}=\boldsymbol{V}\boldsymbol{V}^\top$ and
$\boldsymbol{P}_{\widehat{\mathcal{V}}}=\widehat{\boldsymbol{V}}\widehat{\boldsymbol{V}}^\top$.
It is shown in~\cite{Knyazev_2010}[Corollary~2.13] that 
\begin{equation}
\label{eq:proj-sin}
\bigl\|\boldsymbol{P}_{\widehat{\mathcal{V}}}-\boldsymbol{P}_{\mathcal{V}}\bigr\|_2
=
\max_{j\in[d_i]}\sin\theta_j
=
\bigl\|\sin\Theta(\widehat{\boldsymbol{V}},\boldsymbol{V})\bigr\|_2,
\end{equation}
where $\sin\Theta(\widehat{\boldsymbol{V}},\boldsymbol{V})$ is the diagonal matrix with entries $\sin\theta_j$ corresponding to the sine of the $j$-th principal angle between $\boldsymbol{V}$ and $\widehat{\boldsymbol{V}}$. Using Theorem~\ref{thm:dk} and the fact that $\delta = \lambda_{d_i}-\widehat{\lambda}_{d_i+1}$, we have 
\begin{equation}
\label{eq:dk}
\bigl\|\sin\Theta(\widehat{\boldsymbol{V}},\boldsymbol{V})\bigr\|_2
\le
\frac{\|\widehat{\boldsymbol{\Sigma}}-\boldsymbol{\Sigma}\|_2}{\lambda_{d_i}-\widehat{\lambda}_{d_i+1}},
\end{equation}
Using Corollary 6.3.4 of~\cite{Horn_Johnson_1985}, we have $\widehat{\lambda}_{d_i+1} \le \lambda_{d_i+1} + \|\widehat{\boldsymbol{\Sigma}}-\boldsymbol{\Sigma}\|_2$, so
\[
\lambda_{d_i}-\widehat{\lambda}_{d_i+1}
\ge
\lambda_{d_i}-\lambda_{d_i+1}-\|\widehat{\boldsymbol{\Sigma}}-\boldsymbol{\Sigma}\|_2
=
\Delta-\|\widehat{\boldsymbol{\Sigma}}-\boldsymbol{\Sigma}\|_2
\ge
\frac{\Delta}{2},
\]
where the last step uses the assumption $\|\widehat{\boldsymbol{\Sigma}}-\boldsymbol{\Sigma}\|_2\le \frac{\Delta}{2}$.
Substituting this into Eq.~\ref{eq:dk} yields
\[
\bigl\|\sin\Theta(\widehat{\boldsymbol{V}},\boldsymbol{V})\bigr\|_2
\le
\frac{2}{\Delta}\,\|\widehat{\boldsymbol{\Sigma}}-\boldsymbol{\Sigma}\|_2.
\]
Finally, combining with Eq.~\ref{eq:proj-sin} gives
\[
\bigl\|\boldsymbol{P}_{\widehat{\mathcal{V}}}-\boldsymbol{P}_{\mathcal{V}}\bigr\|_2
\le
\frac{2}{\Delta}\,\bigl\|\widehat{\boldsymbol{\Sigma}}-\boldsymbol{\Sigma}\bigr\|_2,
\]
which concludes the proof.
\end{proof}

Now using these, we are ready to prove Theorem~\ref{thm:sasa-sample-complexity}.
\subsubsection{Proof of Theorem~\ref{thm:sasa-sample-complexity}}

\begin{proof}
If $d_i=d$, the estimated subspace is $\mathbb{R}^d$ and the conclusion is immediate. Assume $d_i<d$. We bound the sample complexity of the principal subspace estimator defined in the theorem by combining the concentration of Lemma~\ref{lem:cov-concentration} with the Davis--Kahan bound of Lemma~\ref{lem:dk-projector}. Here $\mathcal{M}_i(t)$ is the full radius-$t$ sphere in $\mathcal{V}_i$, sampled uniformly. Define the population covariance
\[
\boldsymbol{\Sigma}_i
:=
\underset{\boldsymbol{h}\sim\mathrm{Uniform}(\mathcal{M}_i(t))}{\mathbb{E}}\big[\boldsymbol{h}\boldsymbol{h}^{\top}\big],
\]
with $\|\boldsymbol{h}\|_2=t$.
Note that $\mathcal{M}_i(t)\subseteq \mathcal{V}_i$, so $\boldsymbol{\Sigma}_i$ has range contained in $\mathcal{V}_i$ and therefore
$\lambda_{d_i+1}(\boldsymbol{\Sigma}_i)=0$.
Moreover, because sampling is uniform on the sphere of radius $t$ inside $\mathcal{V}_i$, the population covariance takes the isotropic form $\boldsymbol{\Sigma}_i=\frac{t^2}{d_i}\boldsymbol{P}_{\mathcal{V}_i}$.
Hence the top-$d_i$ eigenspace of $\boldsymbol{\Sigma}_i$ is exactly $\mathcal{V}_i$, and the population eigengap equals
\[
\Delta:=\lambda_{d_i}(\boldsymbol{\Sigma}_i)-\lambda_{d_i+1}(\boldsymbol{\Sigma}_i)=\frac{t^2}{d_i}.
\]

Let $\widehat{\boldsymbol{\Sigma}}_i=\frac{1}{n}\sum_{j=1}^{n}\boldsymbol{h}_j\boldsymbol{h}_j^{\top}$ and let
$\widehat{\mathcal{V}}_i$ be the top-$d_i$ eigenspace of $\widehat{\boldsymbol{\Sigma}}_i$.
By Lemma~\ref{lem:cov-concentration}, with probability at least $1-\delta$,
\[
\bigl\|\widehat{\boldsymbol{\Sigma}}_i-\boldsymbol{\Sigma}_i\bigr\|_2
\le
2t^2\left(
\frac{\log\!\left(\frac{2d}{\delta}\right)}{3n}
+
\frac{1}{n}\sqrt{\frac{\log^2\!\left(\frac{2d}{\delta}\right)}{9}+2n\log\!\left(\frac{2d}{\delta}\right)}
\right).
\]
Using $\sqrt{a+b}\le \sqrt{a}+\sqrt{b}$ with
$a=\frac{\log^2\!\left(\frac{2d}{\delta}\right)}{9}$ and $b=2n\log\!\left(\frac{2d}{\delta}\right)$ gives
\[
\begin{aligned}
\bigl\|\widehat{\boldsymbol{\Sigma}}_i-\boldsymbol{\Sigma}_i\bigr\|_2
&\le
2t^2\left(\frac{\log\!\left(\frac{2d}{\delta}\right)}{3n}
+\frac{1}{n}\left(\frac{\log\!\left(\frac{2d}{\delta}\right)}{3}+\sqrt{2n\log\!\left(\frac{2d}{\delta}\right)}\right)\right)\\
&=
\frac{4t^2\log\!\left(\frac{2d}{\delta}\right)}{3n}
+
2t^2\sqrt{\frac{2\log\!\left(\frac{2d}{\delta}\right)}{n}}.
\end{aligned}
\]
We need to make sure the RHS is $\le \frac{t\varepsilon}{2d_i}$. It suffices to require each term on the right-hand side to be at most $\frac{t\varepsilon}{4d_i}$.
\[
\frac{4t^2\log\!\left(\frac{2d}{\delta}\right)}{3n}
\le
\frac{t\varepsilon}{4d_i},
\qquad
2t^2\sqrt{\frac{2\log\!\left(\frac{2d}{\delta}\right)}{n}}
\le
\frac{t\varepsilon}{4d_i}.
\]
These conditions are equivalent to
\[
n
\ge
\frac{16}{3}\,d_i\,t\,\frac{\log\!\left(\frac{2d}{\delta}\right)}{\varepsilon},
\qquad
n
\ge
128\,d_i^2\,\frac{t^2\log\!\left(\frac{2d}{\delta}\right)}{\varepsilon^2}.
\]
Now if
\[
n \ge 128\,d_i^2\,\frac{t^2\log\!\left(\frac{2d}{\delta}\right)}{\varepsilon^2},
\]
it satisfies both inequalities.
Then
\[
\bigl\|\widehat{\boldsymbol{\Sigma}}_i-\boldsymbol{\Sigma}_i\bigr\|_2
\le
\frac{t\varepsilon}{2d_i}
=
\frac{\Delta}{2} \cdot \frac{\varepsilon}{t}\le \frac{\Delta}{2}.
\]

Therefore Lemma~\ref{lem:dk-projector} applies and yields
\[
\bigl\|\boldsymbol{P}_{\widehat{\mathcal{V}}_i}-\boldsymbol{P}_{\mathcal{V}_i}\bigr\|_2
\le
\frac{2}{\Delta}\bigl\|\widehat{\boldsymbol{\Sigma}}_i-\boldsymbol{\Sigma}_i\bigr\|_2
\le
\frac{2}{\frac{t^2}{d_i}}\cdot \frac{t\varepsilon}{2d_i}
=
\frac{\varepsilon}{t}.
\]
Finally, for any $\boldsymbol{h}\in\mathcal{M}_i(t)$, we have $\boldsymbol{h}\in\mathcal{V}_i$ and thus
\[
\operatorname{dist}\big(\boldsymbol{h},\widehat{\mathcal{V}}_i\big)
=
\bigl\|(\boldsymbol{I}-\boldsymbol{P}_{\widehat{\mathcal{V}}_i})\boldsymbol{h}\bigr\|_2
=
\bigl\|(\boldsymbol{P}_{\mathcal{V}_i}-\boldsymbol{P}_{\widehat{\mathcal{V}}_i})\boldsymbol{h}\bigr\|_2
\le
\bigl\|\boldsymbol{P}_{\widehat{\mathcal{V}}_i}-\boldsymbol{P}_{\mathcal{V}_i}\bigr\|_2\,\|\boldsymbol{h}\|_2
\le
\frac{\varepsilon}{t}\cdot t
=
\varepsilon.
\]
Taking the supremum over $\boldsymbol{h}\in\mathcal{M}_i(t)$ with $\|\boldsymbol{h}\|_2=t$ completes the proof.
\end{proof}
\subsubsection{Proof of Proposition~\ref{prop:split-sample-complexity}}
\begin{proof}
The fixed routing rule maps i.i.d. inputs to a coupon-collector process \citep{Doumas_Papanicolaou_2015}. Write $p_j=\mathbb{P}(J(\boldsymbol{h})=j)$, so $\sum_jp_j=1$. Let $I_j=\mathbbm{1}\{C_j=0\}$ and $Z=\sum_jI_j$ be the number of unseen atoms. Independence of the inputs and Jensen's inequality give
\[
\mathbb{E}[Z]
=\sum_{j=1}^N(1-p_j)^n
\ge N\left(1-\frac{1}{N}\right)^n
\ge N\exp\left(-\frac{n}{N-1}\right),
\]
where the last inequality uses $\log(1-x)\ge-x/(1-x)$ for $0\le x<1$.

For distinct $j$ and $\ell$,
\[
\mathbb{E}[I_jI_\ell]
=(1-p_j-p_\ell)^n
\le(1-p_j)^n(1-p_\ell)^n
=\mathbb{E}[I_j]\mathbb{E}[I_\ell].
\]
Consequently, $\mathbb{E}[Z^2]\le\mathbb{E}[Z]+(\mathbb{E}[Z])^2$. By Cauchy--Schwarz,
\[
\mathbb{P}(Z>0)
\ge\frac{(\mathbb{E}[Z])^2}{\mathbb{E}[Z^2]}
\ge\frac{\mathbb{E}[Z]}{1+\mathbb{E}[Z]}.
\]
If every atom receives a sample with probability at least $1-\delta$, then $\mathbb{P}(Z>0)\le\delta$, so $\mathbb{E}[Z]\le\delta/(1-\delta)$. Combining this with the lower bound on $\mathbb{E}[Z]$ and rearranging gives
\[
n\ge(N-1)\log\left(\frac{N(1-\delta)}{\delta}\right).
\]
as desired.
\end{proof}


\clearpage  
\section*{NeurIPS Paper Checklist}

\begin{enumerate}

\item {\bf Claims}
    \item[] Question: Do the main claims made in the abstract and introduction accurately reflect the paper's contributions and scope?
    \item[] Answer: \answerYes{} 
    \item[] Justification: The paper contribution reflects the claims made. 

\item {\bf Limitations}
    \item[] Question: Does the paper discuss the limitations of the work performed by the authors?
    \item[] Answer: \answerYes{} 
    \item[] Justification: In Section~\ref{sec:discussion} and throughout each, we discussed the limitation and scope of the work.

\item {\bf Theory assumptions and proofs}
    \item[] Question: For each theoretical result, does the paper provide the full set of assumptions and a complete (and correct) proof?
    \item[] Answer: \answerYes{} 
    \item[] Justification: Each theoretical results has the assumptions stated

    \item {\bf Experimental result reproducibility}
    \item[] Question: Does the paper fully disclose all the information needed to reproduce the main experimental results of the paper to the extent that it affects the main claims and/or conclusions of the paper (regardless of whether the code and data are provided or not)?
    \item[] Answer: \answerYes{} 
    \item[] Justification: We provide all the necessary information to reproduce the result, including data, model, and code.

\item {\bf Open access to data and code}
    \item[] Question: Does the paper provide open access to the data and code, with sufficient instructions to faithfully reproduce the main experimental results, as described in supplemental material?
    \item[] Answer: \answerYes{} 
    \item[] Justification: All the data/models that we used are available publicly.

\item {\bf Experimental setting/details}
    \item[] Question: Does the paper specify all the training and test details (e.g., data splits, hyperparameters, how they were chosen, type of optimizer) necessary to understand the results?
    \item[] Answer: \answerYes{}{} 
    \item[] Justification: The details are in the Appendix.

\item {\bf Experiment statistical significance}
    \item[] Question: Does the paper report error bars suitably and correctly defined or other appropriate information about the statistical significance of the experiments?
    \item[] Answer: \answerYes{} 
    \item[] Justification: We ran experiments for three seeds and reported the mean and standard deviation for them.

\item {\bf Experiments compute resources}
    \item[] Question: For each experiment, does the paper provide sufficient information on the computer resources (type of compute workers, memory, time of execution) needed to reproduce the experiments?
    \item[] Answer: \answerYes{} 
    \item[] Justification: We stated it in both the main body and the Appendix.
    
\item {\bf Code of ethics}
    \item[] Question: Does the research conducted in the paper conform, in every respect, with the NeurIPS Code of Ethics \url{https://neurips.cc/public/EthicsGuidelines}?
    \item[] Answer: \answerYes{} 
    \item[] Justification: Yes it conforms the Code of ethics.

\item {\bf Broader impacts}
    \item[] Question: Does the paper discuss both potential positive societal impacts and negative societal impacts of the work performed?
    \item[] Answer: \answerNA{} 
    \item[] Justification: We don't think the work has a broader societal impact.
    
\item {\bf Safeguards}
    \item[] Question: Does the paper describe safeguards that have been put in place for responsible release of data or models that have a high risk for misuse (e.g., pre-trained language models, image generators, or scraped datasets)?
    \item[] Answer: \answerNA 
    \item[] Justification: Our model is not high risk of misuse.

\item {\bf Licenses for existing assets}
    \item[] Question: Are the creators or original owners of assets (e.g., code, data, models), used in the paper, properly credited and are the license and terms of use explicitly mentioned and properly respected?
    \item[] Answer: \answerYes{} 
    \item[] Justification: We cited all of the original owners.

\item {\bf New assets}
    \item[] Question: Are new assets introduced in the paper well documented and is the documentation provided alongside the assets?
    \item[] Answer: \answerYes{} 
    \item[] Justification: We released the code and documentation on how to use the trained models, as well as the recipe for training and evaluation.

\item {\bf Crowdsourcing and research with human subjects}
    \item[] Question: For crowdsourcing experiments and research with human subjects, does the paper include the full text of instructions given to participants and screenshots, if applicable, as well as details about compensation (if any)? 
    \item[] Answer: \answerNA 
    \item[] Justification: We don't have such experiments.

\item {\bf Institutional review board (IRB) approvals or equivalent for research with human subjects}
    \item[] Question: Does the paper describe potential risks incurred by study participants, whether such risks were disclosed to the subjects, and whether Institutional Review Board (IRB) approvals (or an equivalent approval/review based on the requirements of your country or institution) were obtained?
    \item[] Answer: \answerNA 
    \item[] Justification: N/A. We don't have such experiments.

\item {\bf Declaration of LLM usage}
    \item[] Question: Does the paper describe the usage of LLMs if it is an important, original, or non-standard component of the core methods in this research? Note that if the LLM is used only for writing, editing, or formatting purposes and does \emph{not} impact the core methodology, scientific rigor, or originality of the research, declaration is not required.
    \item[] Answer: \answerNA{} 
    \item[] Justification: LLM was used for searching the literature, assisting with code, verifying  and writing and polishing.

\end{enumerate}

\end{document}